\documentclass[10pt,journal,compsoc]{IEEEtran}
\usepackage{cite}
\usepackage{amsmath}
\usepackage{amsthm}
\usepackage{amsfonts}
\usepackage{url}
\usepackage{graphicx}
\usepackage{subfigure}
\usepackage{caption}
\usepackage{epstopdf}
\usepackage{multirow}
\usepackage{color,xcolor}
\usepackage{cite}
\usepackage[switch]{lineno}
\usepackage{CJK}
\usepackage{mathrsfs}
\usepackage{algorithm}
\usepackage{algorithmic}
\usepackage{multicol}
\usepackage{mathtools}
\usepackage{booktabs}
\usepackage{amssymb}
\usepackage{hyperref}
\usepackage{comment}
\usepackage{makecell}

\def\red{\textcolor{red}}

\newtheorem{theorem}{Theorem}

\def\Xset{\mathbf{X} }
\def\Yset{\mathbf{Y}}

\def\Pset{\mathbf{P} }
\def\Qset{\mathbf{Q} }
\def\Cset{\mathbf{C} }
\def\p{\mathbf{p}}
\def\q{\mathbf{q}}

\def\CD{\rm CD}
\def\PED{\rm PED}
\def\x{\mathbf{x}}
\def\y{\mathbf{y}}

\def\c{\mathbf{c}}

\def\N{\mathcal{N}}

\def\yhat{\hat{\mathbf{y}}}
\newcommand{\R}{\ensuremath{\mathbb{R}}}

\newif\ifannotated

\annotatedtrue

\ifannotated

\newcommand{\delete}[1]{{\color{red}{\sout{#1}}}}

\newcommand{\margincomment}[1]{\marginpar{$\Rightarrow$\color{red}\fbox{\parbox{\linewidth}{\color{black}\scriptsize#1}}}}
\else

\newcommand{\delete}[1]{{\ignorespaces}}

\newcommand{\margincomment}[1]{}
\fi

\begin{document}

\title{MPED: Quantifying Point Cloud Distortion Based on Multiscale Potential Energy Discrepancy}
\author{Qi Yang,
        Yujie Zhang,
        Siheng Chen,
        Yiling Xu,~\IEEEmembership{Member,~IEEE},
        Jun Sun, and
        Zhan Ma,~\IEEEmembership{Senior Member,~IEEE}
        \IEEEcompsocitemizethanks{\IEEEcompsocthanksitem Q. Yang, Y. Zhang, S. Chen, Y. Xu, J. Sun are from Cooperative Medianet Innovation Center, Shanghai Jiao Tong University, Shanghai, 200240, China, (e-mail: \{yang\underline{\hbox to 0.1cm{}}littleqi, yujie19981026, sihengc, yl.xu, junsun\}@sjtu.edu.cn) \\
        \IEEEcompsocthanksitem Z. Ma is with the Nanjing University, Nanjing, Jiangsu, 210093, China, (email: mazhan@nju.edu.cn)}
        \thanks{(Q. Yang and Y. Zhang contributed equally to this work.)}
        \thanks{(Corresponding author: Y. Xu)}
}

\IEEEtitleabstractindextext{
\begin{abstract}
In this paper, we propose a new distortion quantification method for point clouds, the multiscale potential energy discrepancy (MPED). Currently, there is a lack of effective distortion quantification for a variety of point cloud perception tasks. Specifically, in human vision tasks, a distortion quantification method is used to predict human subjective scores and optimize the selection of human perception task parameters, such as dense point cloud compression and enhancement.  In machine vision tasks, a distortion quantification method usually serves as loss function to guide the training of deep neural networks for unsupervised learning tasks (e.g., sparse point cloud reconstruction, completion, and upsampling). Therefore, an effective distortion quantification should be differentiable, distortion discriminable, and have low computational complexity. However, current distortion quantification cannot satisfy all three conditions. To fill this gap, we propose a new point cloud feature description method, the point potential energy (PPE), inspired by classical physics. We regard the point clouds are systems that have potential energy and the distortion can change the total potential energy.  By evaluating various neighborhood sizes, the proposed MPED achieves global-local tradeoffs, capturing distortion in a multiscale fashion. We further theoretically show that classical Chamfer distance is a special case of our MPED. Extensive experiments show that the proposed MPED is superior to current methods on both human and machine perception tasks. Our code is available at https://github.com/Qi-Yangsjtu/MPED.

\end{abstract}

}

\maketitle

\IEEEdisplaynontitleabstractindextext

%
\IEEEpeerreviewmaketitle

\section{Introduction}\label{sec:intro}
Three-dimensional (3D) point cloud has become an important geometric data
structure, which precisely presents the external surfaces of objects or scenes in the 3D space. They can be widely used in augmented reality~\cite{lim2020Augmented}, autonomous driving~\cite{ChenLFGW:20}, industrial robotics~\cite{Rusu2011PCL}, art documentation~\cite{Dore2012HBIM} and many others. To provide an immersive user experience and reliable performance, a high-quality point cloud is in high demand. For example, point clouds used in augmented reality provide rich color and detailed geometric structures to better serve human perception; and point clouds used in autonomous driving present informative depth information for accurate 3D object detection  (e.g., pedestrians, trees and buildings). 

To evaluate the distortion of a point cloud based on its reference version, we need a reliable, accurate and objective distortion quantification that can be generally applied to a wide range of scenarios. To reflect human perception, a distortion quantification should be sensitive to both geometric and color distortion. For example, the point cloud ``LongDress''\cite{yang2020inferring} provided by the moving picture experts group (MPEG)  involves the RGB color as additional attributions. After being compressed by geometry-based point cloud compression (G-PCC), the point number might decrease or increase, and points' colors might be deteriorated due to the quantization \cite{javaheri2019point}.

\begin{figure}[pt]
	\centering
		\includegraphics[width=1\linewidth]{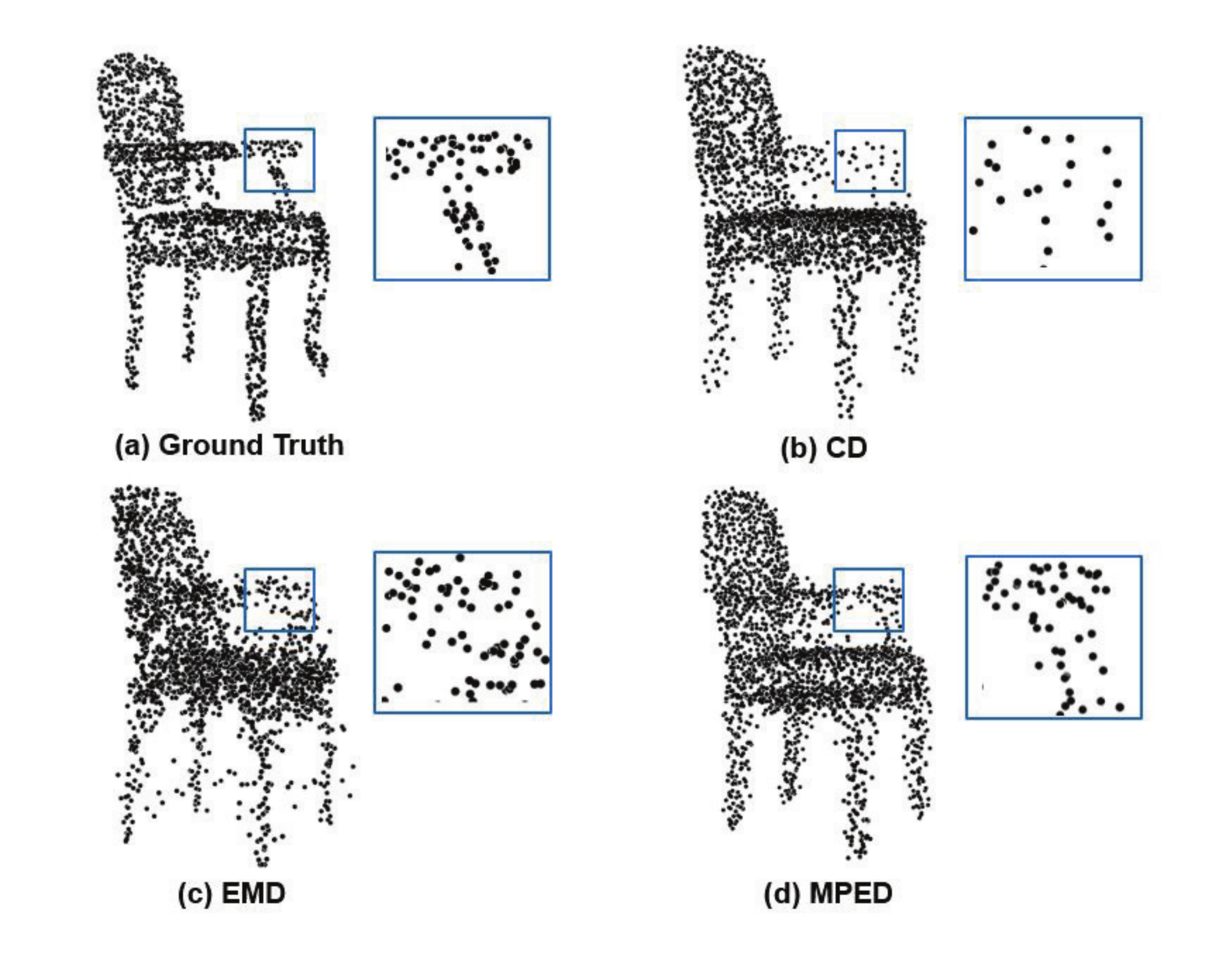}
	\caption{Point cloud reconstruction with LatentNet \cite{achlioptas2018learning}. (a) baseline; (b) CD as loss function; (c) EMD as loss function; (d) MPED as loss function. MPED provides the best reconstruction.}
	\label{fig:point_example}
\end{figure} 

Besides being used for quality assessment in human perception tasks, an objective distortion quantification could also serve for machine perception, which is a loss function to guide the training of deep neural networks. Specifically, unsupervised tasks of machine perception generally aim to extract informative features from point clouds without using human supervision~\cite{zhang2019unsupervised}. Some typical tasks include reconstruction~\cite{achlioptas2018learning,chen2019deep}, completion~\cite{liu2020morphing,huang2020pf}, and upsampling~\cite{qian2019pu}. To train those networks, a distortion quantification usually works as the loss function to supervise the output of the network towards the ground truth. Therefore, a distortion quantification should be differentiable to enable backward propagation. Since machine perception tasks need to train networks with a huge amount of data, a distortion quantification needs to have a low resource overhead to facilitate training.

Based on the above analysis, there are several conditions for the desirable design of the point cloud distortion quantifications. For human perception tasks, a distortion quantification should reflect human visual characteristics and be computationally effective to enable real-time applications; for machine perceptions tasks, a distortion quantification is desired for distortion discrimination, differentiability and low complexity.
Currently, however, there is a lack of a generic distortion quantification for point clouds that can satisfy the above conditions at the same time. For example, Chamfer distance (CD)\cite{achlioptas2018learning}, Earth mover's distance (EMD)\cite{kuhn1955hungarian},  and $\rm PSNR_{YUV}$ \cite{torlig2018novel} use point-wise distance to quantify distortion, which presents unstable performance when faced with multiple distortions. Besides, as the most widely used loss function in machine perception tasks, CD has a shortcoming named ``Chamfer's blindness'' \cite{achlioptas2018learning} or ``point-collapse''\cite{pang2020tearingnet} which can lead to unexpected results (Fig. \ref{fig:point_example} (b)). Point cloud quality metric (PCQM)\cite{meynet2020pcqm} and GraphSIM\cite{yang2020inferring} present obvious performance improvement in human perception tasks compared with CD, but the feature extraction methods (e.g., surface fitting and graph construction) used in these two methods are time-consuming. This work aims to fill this gap and designs a universal distortion quantification that can present robust and reliable performance on point cloud distortion measurement.

{\bf Proposed Distortion Quantification.}
In this paper, we present a new point cloud distortion quantification named multiscale potential energy discrepancy (MPED). 

First, we introduce a new feature extraction method for 3D point cloud distortion quantification; called point potential energy (PPE). The proposed PPE is inspired by the potential energy used in classical physics.  In classic mechanics, the potential energy is the energy that is stored in an object due to its position relative to some zero position, which is a quantitative measure of the object's physical state. For example, the gravitational potential energy relative to the Earth's surface is defined as $E = mgh,$
where $m$ is the mass, $h$ is the Euclidean distance to the origin $\x_0$ and $g$ reflects the gravitational field. The potential energy of a point $\x_i \in \R^6$ is
{\setlength\abovedisplayskip{1pt}
\setlength\belowdisplayskip{1pt}
\begin{equation}
\label{eq:classical_gravitation}
  E_{\x_i} = m_{\x_i} g_{\x_i} h_{\x_i},
\end{equation}}
 where the zero potential point is at $\x_0$.

Considering lossless 3D point cloud is a stable spatial system, the distortion can be regarded as system disturbance caused by the external force~\cite{Valsera1997reconstruction}, and the influence of disturbance can be quantified by the difference of PPE, noted as potential energy discrepancy (PED).

Moreover, the achievement in 2D images (e.g., MSSIM~\cite{wang2003multiscale} and MSEA\cite{yang2019modeling}) and previous research (e.g., MS-GraphSIM \cite{zhang2021ms} and IT-PCQA~\cite{new11-Yang_2022_CVPR}) also inspire us that multiscale features provide a more comprehensive measurement of the overall impact of distortion. Therefore, we fuse the thoughts about distortion-energy inspiration and multiscale representation to design the distortion quantification metric, i.e., MPED.  Fig. \ref{fig:energy} illustrates the scheme of the proposed MPED: i) we first divide a point cloud into multiple spatial neighborhoods and set the center of each neighborhood as zero potential energy plane; ii) then, we compare the differences of PPE of each corresponding neighborhood in the source and target point clouds; iii) we pool the energy discrepancy of multiple neighborhoods into the final PED; iv) in the end we fuse the PED under the different settings of neighborhood-scale to generate the final results.

\begin{figure}[pt]
	\centering
    \includegraphics[width=1\linewidth]{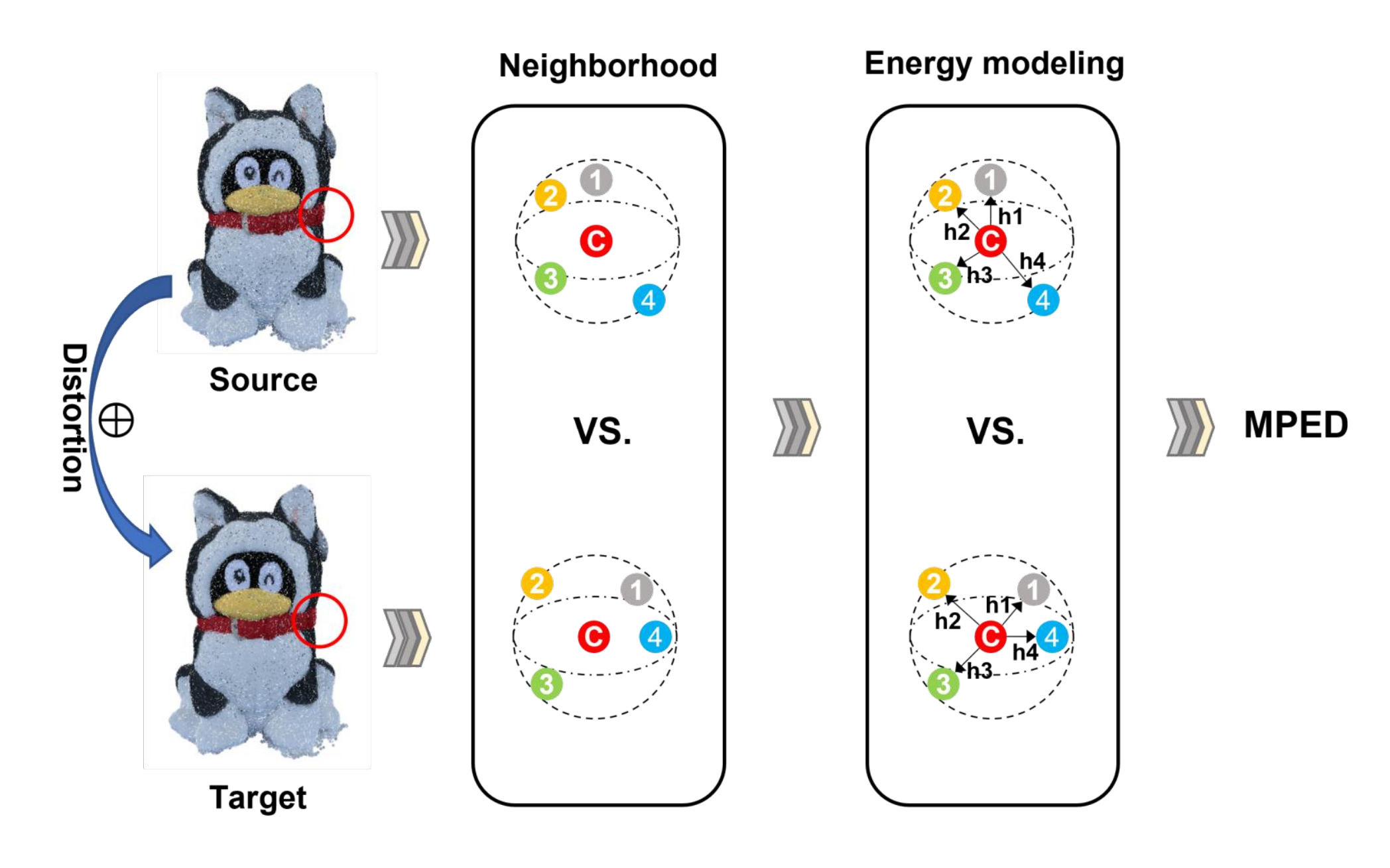}
    \caption{MPED compares the total point potential energies of multiple neighboring point clouds in the source and the target. The region bounded by the red circle represents a neighborhood, the red points labeled with "C" represent the neighborhood center, other color points labeled with numbers represent adjacent points, $h_i$ is the distance between center and its neighbor.}
	\label{fig:energy}
\end{figure}

We present the properties of MPED meticulously. We first prove that CD is a special case of MPED, and MPED satisfies three desired conditions for a universal
objective distortion quantification, i.e., distortion discrimination, differentiability, and low complexity. Next, we respectively analyse the effectiveness of MPED on human and machine perception tasks. For human perception tasks, we use several toy examples to verify the distortion sensitivity of MPED. For machine perception tasks, we analyse the derivative of MPED to show its robustness during the backward propagation.

We validate the performance of MPED based on human and machine perception tasks. In terms of human perception tasks, the proposed MPED presents better performance on perception prediction. Specifically, MPED provides better Pearson linear correlation (PLCC), Spearman rank-order correlation (SROCC) and root mean squared error (RMSE) 
on publicly accessible databases (e.g., SJTU-PCQA~\cite{yang2020predicting}, LS-PCQA~\cite{liu2022point} and WPC \cite{su2019perceptual}). In terms of machine perception tasks, the proposed MPED achieves a better tradeoff between effectiveness and efficiency. Compared with CD, MPED solves the ``point-collapse'' issue; compared with EMD, MPED reduces the computational cost from $O(N^3)$ to $O(N^2+KN)$. Specifically, MPED provides better performance compared to CD and EMD on three unsupervised machine perception tasks (i.e., point cloud reconstruction, shape completion, and upsampling).

The main contributions of this paper are summarized as:

$\bullet$ We present a new universal distortion quantification, MPED, for point clouds. Inspired by classic mechanics, we use potential energy variation to quantify point cloud distortion. 

$\bullet$ MPED satisfies the three desired conditions for a universal objective distortion quantification; that is, MPED is differentiable, low-complexity, and distortion discriminative. 

$\bullet$ MPED shows reliable performance for both human and machine perception tasks.

The rest of the paper is organized as follows.  Section ~\ref{sec:related_work} presents the related work of point cloud distortion quantifications. Section~\ref{sec:problem formulation} formulates the point cloud representation and distortion quantification in different tasks. Section ~\ref{sec:ppe} demonstrates the proposed potential-energy-based feature used for point cloud distortion quantification.
Section~\ref{sec:model} first presents the generic formulation of the proposed MPED, and then presents the implementation details of MPED on both human and machine perception tasks. Section~\ref{sec:property} presents the properties of MPED. Section~\ref{sec:experiment} gives the experiment results. Section~\ref{sec:conclusion} concludes the paper. 

\section{Related Work}\label{sec:related_work}

Point cloud distortion quantifications are first used in human perception tasks. To evaluate the distortion introduced by lossy compression, MPEG first uses p2point as the distortion quantification. Specifically, p2point shares the same calculation method as the CD. However, according to a serial experiments~\cite{alexiou2017subjective,alexiou2017performance,da2019point,alexious2018point,alexiou2018pointt,alexiou2019exploiting,torlig2018novel}, p2point presents unstable results in terms of multiple types of distortion. Studies~\cite{wang2004ssim,wang2003multiscale} in image distortion quantification have proven that the human visual system (HVS) is more sensitive to structural features, rather than point-wise features. Therefore, recently proposed distortion quantifications for human perception tasks, such as p2plane~\cite{tian2017geometric}, Angular Similarity~\cite{new2-alexiou2018point}, Point-to-Distribution~\cite{new3-javaheri2020mahalanobis}, Local Luminance Patterns (LLP) ~\cite{new4-diniz2020local}, BitDance~\cite{9450013}, PointSSIM~\cite{alexiou2020towards},  PCQM~\cite{meynet2020pcqm}, EPES\cite{new1-xu2021epes} and GraphSIM~\cite{yang2020inferring}\cite{zhang2021ms} use the distortion of structural features to replace the distortion of point-wise features. Also, there are some research using statistical features to quantify dense point cloud distortion, such as ~\cite{new6-viola2020color} and \cite{new7-viola2020reduced}. Experimental results demonstrate that these methods provide better performance than the original p2point. For example, experiments in~\cite{yang2020inferring} show that the utilization of structure features (e.g., gradients) presents better performances when predicting human perception in terms of compression-related distortion (e.g., G-PCC and video-based point cloud compression (V-PCC)). Meanwhile, some quantifications have been applied in compression-related works and achieved good performance~\cite{liu2021reduced}. However, most of these previous works (e.g., PCQM and GraphSIM) need to convert scattered points to inerratic representations, such as surface and spatial graph, which are time-consuming.

Besides, point cloud distortion quantifications are commonly used as loss functions to assist in unsupervised machine perception tasks (such as point cloud reconstruction, shape completion, and upsampling).
There are two mainly used quality quantifications in current machine perception tasks: CD and EMD. 
EMD is the solution to a transportation problem that attempts to transform one set into the other.

The advantage of CD is low resource overhead. Since one point can be used multiple times during point matching, this causes a significant drawback for CD, known as ``point-collapse''\cite{pang2020tearingnet} or ``Chamfer's blindness''\cite{achlioptas2018learning}. Specifically, the generated samples that use CD as a loss function suffer from obvious structure deformations. Fig.~\ref{fig:point_example} (b) shows an example in which the points of a chair arm in the reconstruction are sparser than those in the ground truth, while the points of the seat in the reconstruction are denser than those in the ground truth. Points are over-populated around collapse centers\cite{pang2020tearingnet}. The reason is that a reference point may be used multiple times in the point-nearest-neighbor search, and the final results fall into local optima.

EMD is a method that computes the distance between two distributions as the minimum mass that needs to transfer to match two distributions. Compared to CD, the point matching in EMD is stricter as any point that has been used in the point matching cannot be used to match with other points anymore. It means that the pairing of two points considers their spatial distance as well as the global transfer costs. Therefore, EMD is more reliable than CD in machine perception tasks.  A comparison of reconstruction results with CD and EMD as loss functions in~\cite{achlioptas2018learning} reveals that CD can only provide partially good reconstructed results while EMD can provide better global reconstruction. Despite the better performance, many methods avoid using EMD because of its computational complexity~\cite{groueix2018papier,pang2020tearingnet,chen2019deep}. The computational complexity of EMD is $O(N^3)$~\cite{kuhn1955hungarian} and that of CD is $O(N^2)$ for $N$ points. 

There are some other loss functions used for unsupervised point cloud learning and most of them are variants of CD or EMD. Wu et.al. \cite{ wu2021density} proposed the density-aware CD to better detect the disparity of density distributions for point cloud completion and Liu et.al. \cite{liu2020morphing} provides a lightweight version of EMD with complexity $O(N^2k)$, where $k$ is the number of iterations to search for the best-matched point pair. Another related work is by Urbach et. al. \cite{urbach2020dpdist},  which proposed DPDist that compares point clouds by measuring the distance between the surfaces that they were sampled on. However, these works mostly focus on the optimization of one specific task (e.g., point cloud completion or registration) and lack validation in various vision tasks. In comparison, we start from the view of classic mechanics and validate the effectiveness of our method via theoretical proof and experiments on multiple unsupervised tasks.

\section{Problem Formulation}\label{sec:problem formulation}
In this section, we first introduce some basic properties of 3D point clouds. Then we present the details on the application of distortion quantification in human and machine tasks. Finally, we introduce the commonly used distortion quantification in unsupervised learning, i.e., Chamfer Distance.
\subsection{3D Point Cloud}\label{sec:problem_formulation_1}

Let $\Pset$ be a 3D point cloud with $N$ points: $\Pset = \{\x_1, \ldots, \x_N\} \in \mathbb{R}^{N\times6}$, where each $\x_i \in \mathbb{R}^6$ is a vector with 3D coordinates and three-channel color attributes. We only consider color attributes of point cloud in this paper; therefore, $\x_i=[x, y, z, R, G, B]\equiv[\x_i^O, \x_i^I]$, where $\x_i^O=[x, y, z]$ and $\x_i^I=[R, G, B]$. The superscript ``O'' stands for geometric {\it occupancy}, and ``I'' stands for color {\it intensity}.

\subsection{Distortion Quantification of 3D Point Clouds}\label{sec:problem_formulation_2}
An objective distortion quantification aims to measure the difference between two point cloud samples. For two point clouds $\Pset$ and $\Qset$, the difference between them can be formulated as 
$d(\Pset, \Qset)$, where $d(\cdot)$ represents an objective distortion quantification. 

For human or machine perception tasks, point cloud distortion quantification can have different implementations and requirements, which is mostly due to the differences in processing objects and task formulations. Therefore, we respectively reveal the details of distortion quantification in human and machine perception tasks as follows.  

\begin{table}
\caption{Task properties of human and machine perception.}\label{TABLE:task comparison}
\begin{center}
\begin{tabular}{|c|c|c|c|}
  \hline
   \multicolumn{2}{|c|}{} & Human Perception & Machine Perception \\  \hline
  \multirow{2}{0.8cm}{\centering {Point}\\ {Clouds}} & Density & dense &  sparse \\ \cline{2-4}
  & Color & colored &  uncolored \\ \hline
  \multicolumn{2}{|c|}{\thead{Task\\ Formulation}} & $d(\Pset,M(\Pset))=s$ & $\min_{\Psi} d(\Pset, \Psi(M(\Pset)))$ \\ \hline

  \multicolumn{2}{|c|}{\thead{Specific\\ Requirements }} & \thead{ 1. human visual \\oriented distortion \\sensitivity \\ 2. low complexity \\(in real \\time applications)} & \thead{1. any distortion\\ sensitivity \\2. differentiable \\3. low complexity}\\

  \hline
\end{tabular}
\end{center}

\end{table}

{\bf Distortion Quantification in Human Perception Tasks.}
The point clouds used for human perception tasks target for visualizations, which usually require dense point distributions and contain rich color information. Therefore, distortion quantification aims to evaluate the quality degree of a distorted point cloud $M(\Pset)$ corresponding to its reference $\Pset$ refers to the mechanism of HVS, where $M(\cdot)$ represents an operator that influences the quality of point clouds, such as compression and transmission\cite{new9-liuqiPCC}. 

Mathematically, through subjective experiment, a quality score can be derived as
{\setlength\abovedisplayskip{1pt}
\setlength\belowdisplayskip{1pt}
\begin{equation}
     s=\Omega(\Pset,M(\Pset)),\nonumber
\end{equation}}
$\Omega(\cdot)$ is the evaluation criterion of experimenters, which can be influenced by acquired learning\cite{BT500}. The purpose of distortion quantification for human perception is to find a model $d(\cdot)$ that satisfies
{\setlength\abovedisplayskip{1pt}
\setlength\belowdisplayskip{1pt}
\begin{equation}\label{eq:hv_formulation}
  d(\Pset,M(\Pset)))=s=\Omega(\Pset,M(\Pset)).
\end{equation}}
It means the designing of objective distortion quantification needs to consider the characteristics of the human visual system and subjective criteria simultaneously.

{\bf Distortion Quantification in Machine Perception Tasks.} 
The point clouds used for 3D machine perception are usually sparser than those for human perception and do not have color information because they are not consumed by human beings. These samples are often used in unsupervised learning and distortion quantification can serve as a loss function to guide the training of deep neural networks. 

Specifically, in unsupervised learning tasks, we usually aim to design an architecture that can produce a 3D point cloud to approximate an original 3D point cloud. Mathematically, let $\Pset = \{ \p_i \in \R^3 \}_{i=1}^{N}$ be an original 3D point cloud and $M(\cdot)$ be an generic operator that produces a partial, noisy or subsampled 3D point cloud, $M(\Pset)$. A network $\Psi$ takes $M(\Pset)$ as input and produces a reconstruction, $\Psi(M(\Pset))$, which is desired to approximate the original 3D point cloud $\Pset$. To train such a network, a typical paradigm is to solve the following optimization problem,
{\setlength\abovedisplayskip{1pt}
\setlength\belowdisplayskip{1pt}
\begin{eqnarray}\label{eq:cv_formulation}
 \min_{ \Psi} & d(\Pset, \Psi(M(\Pset))), \nonumber
\end{eqnarray}}
where the distortion quantification $d(\cdot)$ is a loss function to guide the optimization of network. It is clear that the design of a loss function would significantly influence the network training and the corresponding reconstruction performances.

{\bf Conclusion.} For unsupervised learning of 3D point clouds, a desirable loss function should satisfy the following conditions:

$\bullet$ Differentiability;

$\bullet$ High computation efficiency and low complexity;

$\bullet$ Sensitive to any geometrical distortion.

For human perception prediction, we do not require the quantification is differentiable, but it needs to further consider the characteristics of HVS when evaluating the influence of geometry and color distortion, and low complexity is necessary for some real-time applications (e.g., VR game).

We compare the task properties of human and machine perception in Table \ref{TABLE:task comparison}.

\subsection{Loss Function in Unsupervised Learning}\label{sec:problem_formulation_3}
In machine perception tasks of point clouds, one commonly used loss function is CD. Therefore, in this part, we use CD as an example to illustrate the importance of the loss function. 

We first present the calculation method of CD. For each point in one 3D point cloud, we find its correspondence in the other 3D point cloud via nearest neighbor search and measure their Euclidean distance. CD considers the aggregated Euclidean distances; that is,
{\setlength\abovedisplayskip{1pt}
\setlength\belowdisplayskip{1pt}
\begin{equation}
\label{eq:CD}
d_{\mathrm{CD}}(\Pset,\Qset) =\sum\limits_{\p \in \Pset} \min\limits_{\q \in \Qset} \parallel \p-\q \parallel_{2}^{2} + \sum\limits_{\q\in \Qset} \min\limits_{\p \in \Pset} \parallel \p-\q\parallel_{2}^{2}.
\end{equation}}
The reason that CD calculate the two-way distance between $\Pset$ and $\Qset$ is to ensure all the point will be used at least one time, which means each point will be optimized at least once during neural network training. 

However, CD purely relies on the geometric information measured by the distances between 3D point pairs and could trap the network at a local minimum during the training. Suppose that we have two distinct reconstruction results, i.e., case 1 and case 2 in Fig. \ref{fig:toy_example}. Red circles represent the ground truth, while  blue and green circles represent reconstructed points. The numbers near edges represent the square of Euclidean distance between reconstructed and reference points matched by nearest neighbors. Referring to Eq.~\eqref{eq:CD}, $\mathrm{d_{CDcase1}}=(2+2+2+2)\times2=16$, and $\mathrm{d_{CDcase2}}=(1+1+1+5)\times2=16$. This means when we use CD as the loss function, the network can converge to different geometric structures with the same loss value. Therefore, this network easily falls into a local minimum and is hard to improve further under the supervision of CD. Unfortunately, EMD could encounter a similar problem as it only considers the point-wise features.

Therefore, an effective loss function is extreme important for these unsupervised learning tasks, which drives us to design the MPED to solve above problem. 
\begin{figure}[tt]
\setlength{\abovecaptionskip}{0.cm}
\setlength{\belowcaptionskip}{-0.cm}
	\centering
    \includegraphics[width=1\linewidth]{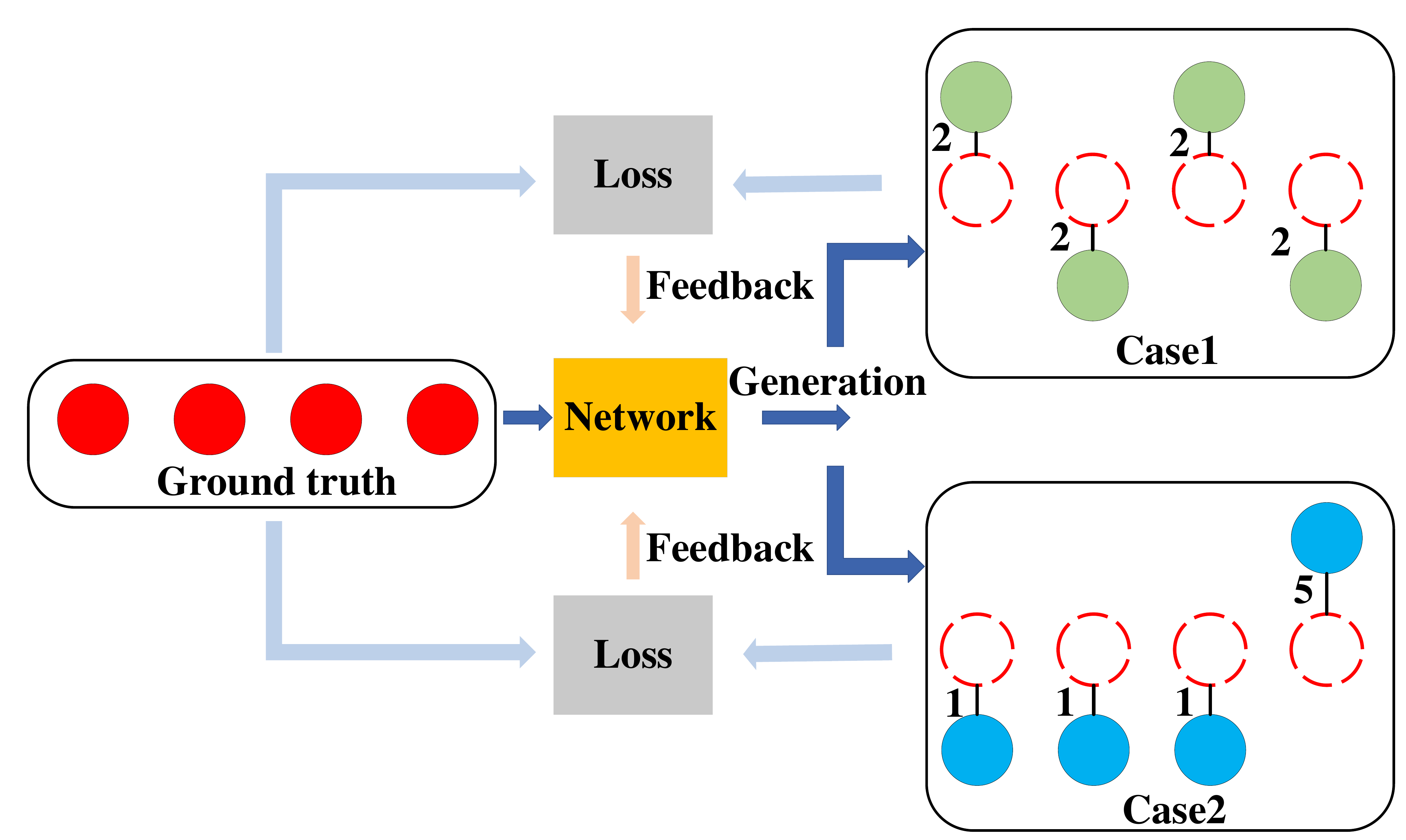}%
	\caption{A toy example of CD limitation. Different geometries may have the same CD value. The numbers near edges in case 1 and case 2 represent
the square of Euclidean distance between reconstructed
and reference points matched by nearest neighbors.}
	\label{fig:toy_example}
\end{figure}

\section{PPE: Point Potential Energy}\label{sec:ppe}
Before presenting our objective distortion quantification between two 3D point clouds, we first introduce a novel feature, i.e., PPE, to quantify the spatial distribution of 3D points. This feature is inspired by the gravitational potential energy in physics and lays a foundation for the proposed distortion quantification. Let  $\Pset_{\x_0} = \{\x_i \in \R^6\}_{i=1}^K$ be a set of $K$ 3D points with $\x_0$ being the origin. We aim to propose a function $\phi(\cdot)$ so that $E_{\Pset_{\x_0}} = \phi(\Pset_{\x_0})$ reflects the spatial distribution of $\Pset_{\x_0}$.

Point cloud distortion might be related to both geometric and attribute information.
Therefore, for a generalized potential energy field, we consider that the mass $m_i$ is related to the point attributes, i.e., $m_{\x_i}=f(\x_0^I, \x_i^I)$; and the gravitational field $g_{\x_i}$ and distance $h_{\x_i}$ is related to the point geometric coordinates, reflecting the spatial location relative to the zero potential plane, e.g., $g_{\x_i}=g(\x_0^O, \x_i^O)$, $h_{\x_i}=h(\x_0^O, \x_i^O)$. Note for a generalized field, $h_{\x_i}$ might not be limited to the Euclidean distance.  We will give specific formulations of $m_{\x_i}$, $g_{\x_i}$ and $h_{\x_i}$ for human and machine perception tasks in Section \ref{sec:imple_hv} and Section \ref{sec:imple_cv}.

For the entire point cloud $\Pset_{\x_0}$, refer to the Eq.\eqref{eq:classical_gravitation} in Section \ref{sec:intro}, we define the total point potential energy as
{\setlength\abovedisplayskip{1pt}
\setlength\belowdisplayskip{1pt}
\begin{eqnarray}
\label{eq:total_point_potential_energy}
 E_{\Pset_{\x_0}} \ = \   \sum_{ \x_i \in \Pset_{\x_0}} E_{\x_i},
\end{eqnarray}}
which aggregates the potential energies of all the points in $\Pset_{\x_0}$. The proposed total point potential energy is a quantitative measure of the state of a 3D point cloud, reflecting the geometric, attributive, and contextual information in this point cloud.

An advantage of the potential energy is its sensitivity to the isometrical distortion, which plays an import role in machine perception tasks to generate high-fidelity samples.

\noindent{\bf Definition 1: isometrical distortion distinguishability.}  
Let $\Pset$ be a reference point cloud and $M(\Pset) = \Pset \oplus \Delta$ be a distorted point cloud, $\oplus$  represents the element-wise addition and $\Delta$ represents some perturbation that distorts the 3D coordinates or attributes of one or multiple points. A measure $d(\cdot)$ is isometrical distortion indistinguishable when we can find a mapping function $Q(\cdot)$ that results in $d(\Pset, M(\Pset)) = Q(\Delta)$ for an arbitrary perturbation $\Delta$.

Intuitively, the isometrical distortion indistinguishability reflects that the measure $d(\cdot)$ is only aware of the distortion and is invariant to the ego property of a point cloud. This is unfavorable because in this case multiple perturbations could lead to the same measure value; that is, $d(\cdot)$ might not distinguish a specific perturbation applied on a point cloud. Thus a good distortion quantification between two point clouds should be isometrical distortion distinguishable.

\noindent{\bf Toy example.}
We illustrate a toy example of the isometrical distortion in Fig. \ref{fig:isometrical}. Assuming $\x_0$ is the origin, the yellow line is a coordinate axis, which represents the spatial distance to the origin or color intensity difference compared with the origin. When the axis means the Euclidean distance to the origin, for a point $\x_1$, $d_{1}(\cdot)$ is a measure that calculates the Euclidean distance differences after applying perturbation $\Delta$. Mathematically, 
{\setlength\abovedisplayskip{1pt}
\setlength\belowdisplayskip{1pt}
\begin{align}
d_{1}(\x_1, \x_1 \oplus \Delta) = &|{\|\x_1-\x_0\|}_2 - {\|\x_1^{'}-\x_0\|}_{2}|\nonumber \\ 
=& |{\|\x_1-\x_0\|}_2 - {\|\x_1^{''}-\x_0\|}_{2}|=\Delta.\nonumber
\end{align}}
We see that $d_{1}(\cdot)$ is not sensitive to the direction of perturbation. The distorted locations $\x_{1}^{'}$ and $\x_{1}^{''}$ play the same role for $d_{1}(\cdot)$.

Moreover, for a point $\x_2\neq \x_1$, 
{\setlength\abovedisplayskip{1pt}
\setlength\belowdisplayskip{1pt}
\begin{align}
d_{1}(\x_2,\x_2 \oplus \Delta) = &|{\|\x_2-\x_0\|}_2 - {\|\x_2^{'}-\x_0\|}_{2}|\nonumber \\ 
=& |{\|\x_2-\x_0\|}_2 - {\|\x_2^{''}-\x_0\|}_{2}|\nonumber \\
=&\Delta  = d_{1}(\x_1,\x_1 \oplus \Delta). \nonumber
\end{align}}
We see that $d_{1}(\cdot)$ is not sensitive to the initial location of perturbation. The initial locations $\x_1$ and $\x_2$ play the same role for  $d_{1}(\cdot)$.

When the axis means the color intensity difference compared with the origin, and $d_{1}(\cdot)$ is a measure that captures the color differences after perturbation, the above derivations are still true. 

We consider six point clouds, $\Pset_{1} = [\x_0, \x_1]$, $\Pset_{2} = [\x_0, \x_1^{'}]$, $\Pset_{3} = [\x_0, \x_1^{''}]$, $\Pset_{4} = [\x_0, \x_2]$, $\Pset_{5} = [\x_0, \x_2^{'}]$, and $\Pset_{6} = [\x_0, \x_2^{''}]$. Each point cloud has two points, and CD as $d_{1}$.

Refer to \eqref{eq:CD}, we have
{\setlength\abovedisplayskip{1pt}
\setlength\belowdisplayskip{1pt}
\begin{align}
{\mathrm{CD}}(\Pset_1, \Pset_2)=&{\mathrm{CD}}(\Pset_1, \Pset_3)={\mathrm{CD}}(\Pset_4, \Pset_5)\nonumber \\ =&{\mathrm{CD}}(\Pset_4, \Pset_6)=2{\Delta}^2.\nonumber
\end{align}}

We see that CD cannot detect the influence of the same perturbation on different point clouds. The reason is that the CD only considers individual point features.

On the other hand, PPE takes contextual information into consideration, which can also help distinguish isometrical distortion. Specifically, we use $m_{\x_i}$, $g_{\x_i}$ and $h_{\x_i}$ to describe the unique role of each point for the whole point cloud. Under the influence of the same perturbation, each point will give a different response to facilitate discrepancy quantification. We will validate the isometrical distortion distinguishability of point potential energy through multiple toy examples in Section~\ref{sec:property} and appendix.

\begin{figure}[pt]
	\centering
    \includegraphics[width=1\linewidth]{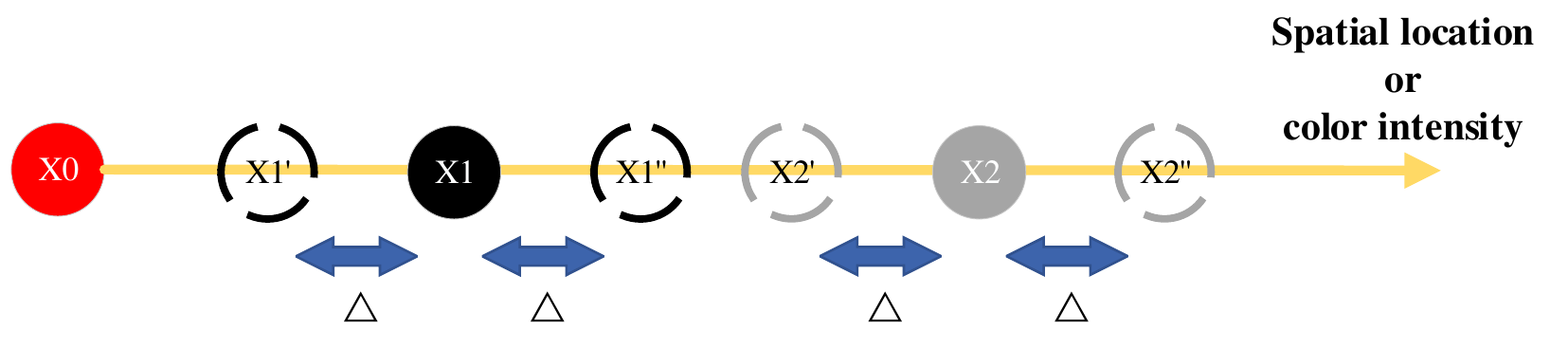}
	\caption{Toy example of isometrical distortion.}\label{fig:isometrical}
\end{figure}

\section{MPED: Measuring Point Cloud Distortion via Multiscale Potential Energy Discrepancy}\label{sec:model}

In this part, we first introduce a general point cloud distortion quantification method, i.e., MPED, based on the proposed PPE, then we give the specific form of MPED for both human and machine perception tasks. 

\subsection{General Form of MPED}\label{sec:imple_hv}

As illustrated in Fig. \ref{fig:energy}, we present the proposed MPED in several steps: i) neighborhood construction. Both source and target point clouds are divided into multiple local neighborhoods based on a set of points as the neighborhood centers; ii) for one neighborhood, the neighborhood center is set to be the zero potential energy plane. We then compute the potential energy for each point, and the total potential energy of a neighborhood is the summation of the potential energies of all the points inside the neighborhood; iii) we propose the PED to quantify the difference between the source and the target point clouds, leading to the single-scale potential energy discrepancy; and iv) we finally extend PED to a multiscale form, i.e., MPED.

{\bf Neighborhood Construction.} Let $\Xset = \{\x_i \in \R^6\}_{i=1}^N$ and $\Yset = \{\y_j \in \R^6\}_{j=1}^M$ be source and target 3D point clouds with $N$ and $M$ points, respectively. Assuming a set of points as neighborhood center $\Cset=\{\c_l \in \R^6\}_{l=1}^{L}$, its neighboring point set in the source point cloud is $\mathcal{N}_{\c_l, K}^{\Xset} \subset \Xset$ with $\c_l$ being the origin, which collects $\c_l$'s $K$ closest points in $\Xset$. Similarly, $\mathcal{N}_{\c_l,K}^{\rm \Yset} \subset \Yset$ denotes $\c_l$'s $K$ closest points in $\Yset$. Now, each center point has two local neighboring point sets that reflect its contextual roles in source and target point clouds, respectively. We then can use the total point potential energy, Eq. ~\eqref{eq:total_point_potential_energy} to quantify the spatial distribution properties in two neighboring point sets; that is,
$E_{\mathcal{N}_{\c_l, K}^{\Xset}}$ and $E_{\mathcal{N}_{\c_l, K}^{\Yset}}$, respectively.

{\bf Point Potential Energy Discrepancy.}
As we demonstrate in Section \ref{sec:intro}, the distortion of point clouds is related to both geometric and attributive information.  To capture this distortion, we can leverage the PPE proposed in Section \ref{sec:ppe}, which quantifies the spatial distribution and color of 3D points via Eqs.  \eqref{eq:classical_gravitation} and \eqref{eq:total_point_potential_energy}. Considering a point $\x_i$ in $\Xset$ corresponding to the origin $\c_l$, we give the mass, the spatial field, and the distance as follows,

\begin{align}
m_{\x_i}& = 
  \begin{cases}
  f({\c_l^I,\x_i^I}),& \text{if $\c_l^I$, $\x_i^I \neq \O$}  \\
  1, & \text{otherwise.}
  \end{cases} \label{eq:m1_g}  \\
  g_{\x_i} &=   
  \begin{cases}
  g(d_{\c_l^O, \x_i^O}), & \text{if $K>1$}  \\
  1, & \text{if $K=1$.}
  \end{cases} \label{eq:g1_g} \\
  h_{\x_i} &= d_{\c_l^O, \x_i^O}. \label{eq:h1_g}
\end{align}
 Note $d_{\mathbf{p}, \mathbf{q}}=\left\|\mathbf{p} - \mathbf{q} \right\|_p^p$. Points in $\Yset$ have similar energy formulations as above. Mass $m_{\x_i}$ is related to the point color information. Considering some point clouds do not have color information, we set: i) when $\c_l^I, \x_{i}^I \neq \phi$, $m_{\x_i}=f({\c_l^I, \x_i^I})$; ii) when $\c_l^I, \x_{i}^I= \O$, $m_{\x_i}=1$ to avoid PPE becomes 0. The distance $h_{\x_i}$ is related to point coordinate information and we set $h_{\x_i}$ as the Minkowski distance between $\x_i^O$ and $\c_l^O$. The spatial field $g_{\x_i}$ can be regarded as a weighting factor related to $h_{\x_i}$ when pooling all the point energies via Eq. (\ref{eq:total_point_potential_energy}), which is determined according to the task characteristics.  A benefit of introducing proper $g_{\x_i}$ is to distinguish isometrical distortion and we set $g_{\x_i}=1$ when $K=1$ to achieve the consistency with $\CD$.  We will give specific formulations of $m_{\x_i}$, $g_{\x_i}$ and $h_{\x_i}$ for human perception tasks in Section \ref{sec:imple_hv} and for machine perception tasks in Section \ref{sec:imple_cv}. Based on above three components, we use Eq. (\ref{eq:total_point_potential_energy}) to calculate  $E_{\mathcal{N}_{\c_l, K}^{\Xset}}$ and $E_{\mathcal{N}_{\c_l, K}^{\Yset}}$, respectively.

Referring to classic mechanics, for the object with the same mass, the farther the object is from the zero potential energy plane, the lower the gravitational accelerated speed becomes. Therefore, we make $g_i$ to satisfy the following requirement:

\noindent{\bf Requirement 1.} For two points $\x_1$ and $\x_2$ with same mass, if $h_{\x_1}\leq h_{\x_2}$, $g_i$ should satisfy $g_{\x_1}\geq g_{\x_2}$.

With $g_{\x_i}$ satisfies {\bf Requirement 1}, potential energy can better detect the isometrical distortion, which will be proved in Section \ref{sec:property} {\bf Theorem \ref{thm:theorem2}}. 

We now propose a potential energy discrepancy, i.e., PED, to measure the difference between the source and the target point cloud. Specifically, the PED is defined as
\begin{equation}\label{eq:PED}
  {\rm PED}_K =  \sum_{\c_l\in\Cset}\left|E_{\mathcal{N}_{\c_l, K}^{\Xset}}-E_{\mathcal{N}_{\c_l, K}^{\Yset}}\right|,
\end{equation}
We will select proper neighborhood center set $\Cset$ according to the corresponding application characteristics in Section \ref{sec:imple_hv} and \ref{sec:imple_cv}.

{\bf Multiscale Point Potential Energy Discrepancy.}
A hyperparameter $K$ is introduced in PED to establish a spatial neighborhood. Inspired by MSSIM \cite{wang2003multiscale}, we extend PED to a multiscale form, i.e., MPED via

{\setlength\abovedisplayskip{1pt}
\setlength\belowdisplayskip{1pt}
\begin{equation}\label{eq:MPED}
   {\rm MPED}_{\Psi} = \frac{1}{|\Psi|}\sum_{K \in \Psi}{\rm PED}_{K},
\end{equation}}
in which $\Psi$ represents a collection of $K$, $|\Psi|$ represents the number of elements in $\Psi$.

\subsection{MPED Implementation Details on Human Perception Tasks}\label{sec:imple_hv}
Here we present the detailed implementation of MPED on human perception tasks. To better predict human subjective perception, the characteristics of HVS need to be considered when applying MPED for human perception tasks, e.g., high-frequency sensitivity\cite{yang2020inferring},  Stevens's power law\cite{teghtsoonian1971exponents} and multiscale visual perception \cite{wang2003multiscale}. Therefore, we inject the above human visual characteristics into the three steps of MPED.

{\bf Neighborhood Construction.} The human visual system is more sensitive to high-frequency information \cite{wang2004image}\cite{yang2019modeling}, such as edge and contour. Therefore, human perception tasks need to pay more attention to the distortion of high-frequency features. Therefore, we use the high-pass filter to select a set of high-frequency points as neighborhood centers. The total point potential energy of these neighborhoods can reflect the characteristics of point cloud high-frequency structures.

Given the source $\Xset$ and the target $\Yset$, refer to the proposal in \cite{yang2020inferring}, we use the method proposed in \cite{chen2017fast} to filter high-frequency points as element of $\Cset$. Specifically,
{\setlength\abovedisplayskip{1pt}
\setlength\belowdisplayskip{1pt}
\begin{equation}
\Cset=\lfloor\Upsilon(\Xset, \beta)\rfloor_{L} \in \mathbb{R}^{L\times6}, L\ll N,\nonumber
\end{equation}}
where $\Upsilon(\cdot)$ is a Haar-like filter, $\beta$ is filter length, and $L$ is the number of points for $\Cset$. For more details, please check \cite{chen2017fast}\cite{yang2020inferring}.

{\bf Point Potential Energy Discrepancy.}
To better reflect the properties of human perception, we inject \textit{Stevens's power law} \cite{teghtsoonian1971exponents} into the PPE formulation. Stevens's power law characterizes an exponential relationship between the magnitude of a physical stimulus and its perceived magnitude, which can be formulated as $P\propto S^{\gamma}$, where $S$ denotes the actual intensity of one stimulus (such as sound, lighting, distance, et al.) and $P$ represents the perceived intensity of human beings; the value of $\gamma$ measures human's sensitivity to stimulus. The larger the value of $\gamma$ is, the higher human's sensitivity to stimulus is. We inject Stevens's power law into the calculation of mass $m_{\x_i}$ and spatial field  $g_{\x_i}$ respectively considering they are non-interference.

Specifically, for the formulation of the mass $m_{\x_i}$, we first introduce the color difference $d_{color}$ between $\x_i^I$ and $\c_l^I$ as follows,
{\setlength\abovedisplayskip{1pt}
\setlength\belowdisplayskip{1pt}
\begin{equation}
    d_{color} =  \sum_{j}^3  k_j |(\x_i^I)_j - (\c_l^I)_j|+1,
\end{equation}}
where $k_i$ represents the weighting factors between different color channel, e.g., if $\x^I\in \mathrm{RGB}$, $k_R: k_G: k_B = 1:2:1$; if $\x^I\in \mathrm{YUV}$, $k_Y: k_U: k_V = 6:1:1$\cite{yang2020inferring}. The color difference reflects the color variation within the neighborhood, which can be regarded as one color stimulus for our perception. 
Following hints in Stevens's power law and previous researches \cite{huang2015power} , we give the formulation of the mass by raising the above computed distances to the power of $\alpha$, i.e., $m_{\x_i}=(d_{color})^{\alpha}$.  Based on the preliminary experiments we find that $\alpha=0.5$ provides the best performance, under which we have $m_{\x_i}=\sqrt{d_{color}}$.

Similarly, for the formulation of $g_{\x_i}$, we view the geometry difference between $\x_i^O$ and $\c_l^O$, i.e., $d_{\c_l^O, \x_i^O}$, as the geometry stimulus, which reflects the geometry shifting within the neighborhood. Because both $g_{\x_i}$ and $h_{x_i}$ are related to $d_{\c_l^O, \x_i^O}$, we consider the perceived shifting based on Steven's power law as $g_{\x_i}h_{\x_i}=g(d_{\c_l^O, \x_i^O})d_{\c_l^O, \x_i^O}=(d_{\c_l^O, \x_i^O})^{\beta}$, where $g_{\x_i}$ adjusts the distance $h_{\x_i}$ to make it correlate with human perception. Based on the preliminary experiments we find that $\beta=0.5$ provides the best performance, under which $g_{\x_i}=1/\sqrt{d_{\c_l^O, \x_i^O}}$.

Finally, we write the mass, the spatial field, and the distance as follows, 
\begin{align}
  m_{\x_i} &= f(\c_l^I, \x_i^I) = \sqrt{\sum\nolimits_{j=1}^3  k_j |(\x_i^I)_j - (\c_l^I)_j|+1},  \nonumber\\
  g_{\x_i} &=   
  g(d_{\c_l^O, \x_i^O})= \frac{1}{\sqrt{\|\x_i^O -\c_l^O \|_{p}^{p}+\sigma}}, \label{g1} \\
  h_{\x_i} &= d_{\c_l^O, \x_i^O} = \| \x_i^O - \c_l^O \|_{p}^{p}, \nonumber
\end{align}
where $\sigma$ is no-zero constants to prevent numerical instability. $p\in\{1,2\}$ represents 1-norm or the square of 2-norm. We only consider the situation of $K>1$ in Eq. \ref{eq:g1_g} given the fact that human visual perception is highly adapted for extracting structural information \cite{wang2004ssim}.

For the human visual system, the final score is derived based on the whole object, and the sizes and the filtered high-frequency point number of different point cloud (in our experiment we choose $L=N/10000$) in subjective databases is usually different. Therefore, we introduce the \textit{intrinsic resolution} \cite{java2020resolution} and the number of high-frequency point (i.e., $L$) to normalize the PED value. The intrinsic resolution for $\PED_{K}$ can be formulated as follows,
\begin{equation}
    {\rm IR} = (\frac{{\sum\nolimits_{\c_l\in\Cset}\sum\nolimits_{\x_i\in N_{\c_l,K}^{\Xset}}d_{\c_l^O,\x_i^O}}}{L\cdot K})^{\beta}.
\end{equation}
where $\beta = 0.5$ is consistent with the above setting.  Then, We normalize the functions proposed in Eq. (\ref{eq:PED}) as follows,
\begin{equation}
    {\PED_K^{normalized}}=\frac{\sum_{\c_l\in\Cset}\left|E_{\mathcal{N}_{\c_l, K}^{\Xset}}-E_{\mathcal{N}_{\c_l, K}^{\Yset}}\right|}{L\cdot \rm{IR} }.
\end{equation}

{\bf Multiscale Point Potential Energy discrepancy.}
Finally, to introduce the multiscale characteristics of HVS, we replace $\PED$ in Eq. (\ref{eq:MPED}) with $\PED_K^{normalized}$ to get normalized MPED.

\subsection{MPED Implementation Details on Machine Perception Tasks}\label{sec:imple_cv}
In this section, we present the detailed implementation of MPED on machine perception tasks. 
The point clouds used for machine vision tasks are usually sparser than those for human perception. For instance,  samples in ShapeNet~\cite{chang2015shapenet} and ModelNet~\cite{wu20153d} consist of thousands of points that only reflect rough shapes of objects (e.g., Fig. \ref{fig:point_example} (a)). Meanwhile, the task formulation of machine perception is a typical optimization problem, which demands that the specific design of MPED should consider its derivative during the backward propagation.

{\bf Neighborhood Construction.} Different from human perception tasks, machine perception tasks usually aim to generate samples that are exactly the same as the reference samples. Therefore, distortion quantifications used in machine perception tasks need to equally detect the shape deformation at any location.

We make PED equally detect the shape deformations at any location via applying the following neighborhood center selection strategy: both source and target samples are used as references to calculate the geometrical distortion. Therefore, given the ground truth $\Xset$ and the reconstructed sample $\Yset$, we set $$\Cset = \Xset\bigcup\Yset.$$ 

{\bf Point Potential Energy Discrepancy.}
To derive the formulation of PPE for machine vision tasks, we need to determine the spatial field function $g(\cdot)$ in Eq. (\ref{eq:g1_g}).

Considering the neighborhood center selection, the PED can be rewritten as follows to be consistent with the form of CD (i.e., Eq. \eqref{eq:CD}),  
{\setlength\abovedisplayskip{1pt}
\setlength\belowdisplayskip{1pt}
\begin{equation}\label{eq:PED_1}
\begin{aligned}
   {\rm PED} &= \sum_{\x_i\in\Xset}\left|E_{\mathcal{N}_{\x_i, K}^{\Xset}}-E_{\mathcal{N}_{\x_i, K}^{\Yset}}\right|+\sum_{\y_j\in\Yset}\left|E_{\mathcal{N}_{\y_j, K}^{\Xset}}-E_{\mathcal{N}_{\y_j, K}^{\Yset}}\right| \\
   &={\rm PED}_{\Xset2\Yset}+{\rm PED}_{\Yset2\Xset}.
\end{aligned}
\end{equation}}

 The derivative of $\PED_{\Xset2\Yset}$ for one optimizable point $\yhat$ can be solved as Eq.  \eqref{eq:x2y_grad_g} 
{\setlength\abovedisplayskip{1pt}
\setlength\belowdisplayskip{1pt}
\begin{small}
\begin{equation}\label{eq:x2y_grad_g}
    \begin{aligned}
       \frac{\partial\PED_{\Xset2\Yset}}{\partial\yhat} &=
       2\sum_{\x_i \in \N_{\Xset2\yhat}} 
        k(d_{\x_i,\yhat})
       \left( \yhat - \x_i^{pse}\right) \\
       &=2 \left[\sum_{\x_i \in \N_{\Xset2 \yhat}} 
         k(d_{\x_i,\yhat})\right]
       \left[ \yhat - \frac{\sum\limits_{\x_i \in \N_{\Xset2 \yhat}} 
        k(d_{\x_i,\yhat})\x_i^{pse}}{\sum\limits_{\x_i \in \N_{\Xset2 \yhat}} 
        k(d_{\x_i,\yhat})} \right],
    \end{aligned}
\end{equation}
\end{small}}
where $\N_{\Xset2\yhat}=\left\{\x_i|\yhat \in \N_{\x_i,K}^{\Yset} \right\}$ represents the point collection in $\Xset$ that can find $\yhat$ in their K-nearest neighborhoods. $d_{\x_i,\yhat}=\|\x_i-\yhat\|_2^2$; the formulation of $k(\cdot)$ can be written as
{\setlength\abovedisplayskip{1pt}
\setlength\belowdisplayskip{1pt}
\begin{equation}\label{eq:kt}
    k(t) = g(t)+\frac{\partial g(t)}{\partial t}t;
\end{equation}}
$\x_i^{pse}$ represents a pseudo point of $\x_i$ which has the same or opposite location of $\x_i$ corresponding to $\yhat$, which can be formulated as
{
\begin{equation}\label{eq:pseudo_point}
    \begin{aligned}
        \x_i^{pse}&=\yhat-{\rm sgn}( E_{\N_{\mathbf{x}_i,K}^{\Yset}}-E_{\N_{\x_i,K}^{\Xset}})\left( \yhat - \x_i\right) \\
        &=
        \begin{cases}
            \x_i,& E_{\N_{\x_i,K}^{\Yset}}>E_{\N_{\x_i,K}^{\Xset}}, \\
            \yhat, & E_{\N_{\x_i,K}^{\Yset}}=E_{\N_{\x_i,K}^{\Xset}}, \\
            2\yhat-\x_i, & E_{\N_{\x_i,K}^{\Yset}}<E_{\N_{\x_i,K}^{\Xset}}. \\
        \end{cases}
    \end{aligned}
\end{equation}}

According to Eq. (\ref{eq:x2y_grad_g}), we can see that the derivative of $\PED_{\Xset2\Yset}$ descends towards the result that $\yhat$ approaches to the weighted average of $\N_{\Xset2\yhat}$. The derivative of $\PED_{\mathbf{Y}2\mathbf{X}}$ can be also rewritten as a similar formulation (Please check Section \ref{sec:analysis for machine} for more details on the derivative of MPED). Therefore, we can first consider the selection of $k(t)$ and then solve the corresponding $g(t)$. 

For the selection of $k(t)$, an intuitive thought is that for the points in $\N_{\Xset2\yhat}$ farther from the $\yhat$ should have smaller weights, which can increase the robustness of generation model since the peripheral points are the least reliable \cite{comaniciu2003kernel}. Considering  $t$ represents the distance in Eq. (\ref{eq:x2y_grad_g}), it is desired that $k(t)$ is decreasing. Therefore, we add one constraint for $k(t)$ as follows
{\setlength\abovedisplayskip{1pt}
\setlength\belowdisplayskip{1pt}
\begin{equation}
        \frac{\partial k(t)}{\partial t}=2\frac{\partial g(t)}{\partial t}+\frac{\partial^2 g(t)}{\partial t^2}t\leqslant0.
\end{equation}}
A choice of $g(t)$ satisfying the above constraint and \textbf{Requirement 1} is $g(t)=1/{\sqrt{t}}$. Furthermore, such function has two properties that are expected. One is that $k(t)\propto g(t)$, thus we can easily compare the influence of points at different positions corresponding to the center. 

The other property is related to the goal of machine perception tasks. For our generation work, it is desired that more points in the reconstructed point cloud have the same position corresponding to their ground truth in the reference point cloud. Therefore, if one point in the ground truth, $\x_k$, is significantly closer to $\yhat$ than other points, $\yhat$ is expected to move towards $\x_k$.
For $g(t)=1/{\sqrt{t}}$, we can conclude from Eq. (\ref{eq:x2y_grad_g}) and (\ref{eq:kt}) that
{\setlength\abovedisplayskip{1pt}
\setlength\belowdisplayskip{1pt}
\begin{equation}
    \frac{\sum\limits_{\x_i \in \N_{\Xset2 \hat{\mathbf{y}}}} 
        k(d_{\x_i,\yhat})\x_i^{pse}}{\sum\limits_{\x_i \in \N_{\Xset2 \yhat}}
        k(d_{\x_i,\yhat})} \approx \x_k^{pse}\approx\x_k,\quad \text{if\quad $d_{\x_k,\yhat}\rightarrow0$},
\end{equation}}
which states that $\yhat$ will quickly converge to $\x_k$ instead of other locations when they are very close.

Finally, considering $t=d_{\c_l^O, \x_i^O}$, we write the mass, the spatial field, and the distance as follows,  
\begin{align}
  m_{\x_i} &= 1,  \nonumber \\
  g_{\x_i} &= \begin{cases}
  g(d_{\c_l^O, \x_i^O})= \frac{1}{\sqrt{\|\x_i^O -\c_l^O \|_{p}^{p}+\sigma}}, & \text{if $K>1$}  \\
  1, & \text{if $K=1$.}
  \end{cases} , \label{g1}\\
  h_{\x_i} &= d_{\c_l^O, \x_i^O} = \| \x_i^O - \c_l^O \|_{p}^{p}, \nonumber
\end{align}
where $\sigma$ is a no-zero constant to prevent numerical instability and we set $p=2$.

{\bf Multiscale Point Potential Energy Discrepancy.}
To better guide the training of networks, point clouds in many machine vision tasks are normalized into the same size and point number during the preprocessing. Therefore, we directly use $\rm MPED$ in Eq. (\ref{eq:MPED}) as the loss function in our experiments.

\section{Properties of MPED}\label{sec:property}
In this section, we will present the properties of the proposed MPED. Considering MPED is only the arithmetic mean of PEDs at multiple scales, we  focus on the properties of PED at a single scale. Meanwhile, as mentioned in Section \ref{sec:problem_formulation_2}, there are some differences between human and machine perception tasks when applying point cloud distortion quantifications. Especially, distortion quantifications used in human perception tasks aim to evaluate the distortion degree of a target corresponding to its source, which is a \textit{forward prediction} process. In contrast, in machine perception tasks, distortion quantifications usually aim to assist network optimization, which includes the backward propagation process. 

Therefore, we will first give the general properties of PED in Section \ref{sec:property}. Then, we respectively compare the PED and other point cloud distortion quantifications in human and machine perception tasks, which are given in Section \ref{sec:analysis for human} and \ref{sec:analysis for machine}. Especially, our analysis reveals that the most popular distortion quantification, CD, is a special case of PED, and PED shows potential superiority over CD in both human and machine perception tasks.

\subsection{General Properties}
In this section, we discuss the general properties of the proposed PED. Specifically, we first demonstrate that CD is a special case of PED (i.e., {\bf Theorem \ref{thm:theorem1}}). Then, we demonstrate the effect of {\bf Requirement 1} when encountering isometrical perturbation (i.e., {\bf Theorem \ref{thm:theorem2}}).
Finally, we show that the proposed PED satisfies all three conditions for a desirable objective distortion quantification (i.e., differentiability, low complexity, and distortion discrimination).

\begin{theorem}\label{thm:theorem1}
For $\Xset$, $\Yset$,  if $\forall \x_i \in \Xset$, $\y_j \in \Yset$ satisfy $\x_i^I, \y_j^I = \O$ (i.e., $\Xset$, $\Yset$ only have coordinate information), $g(\cdot)=1$, $\Psi={[1]}$, $p=2$, $\Cset = \Xset\bigcup\Yset$, ${\rm PED} = {\mathrm{CD}}$.
\end{theorem}
\begin{proof}
Refer to Eq. (\ref{eq:m1_g}), if point clouds do not have additional attributes except for spatial coordinate, we have  $m_{\x_i} =1$.
Considering $c_i\in \Xset$, the nearest neighbor for $\c_l$ in $\Xset$ is exactly $\c_l$. Therefore, $\mathrm{E}_{\mathcal{N}^{\Xset}_{\c_l, 1}}=0$. The same when $c_l\in \Yset$, e.g.,  $\mathrm{E}_{\mathcal{N}^{\Yset}_{\c_l, 1}}=0$. Therefore,
\begin{equation}
    \begin{aligned}
     {\rm PED}_K &= \sum_{\c_l\in\Cset}\left|E_{\mathcal{N}_{\c_l, 1}^{\Xset}}-E_{\mathcal{N}_{\c_l, 1}^{\Yset}}\right|,\\
  &= \sum_{\x_i\in\Xset}\left|E_{\mathcal{N}_{\x_i, 1}^{\Xset}}-E_{\mathcal{N}_{\x_i, 1}^{\Yset}}\right|+ \sum_{\y_j\in\Yset}\left|E_{\mathcal{N}_{\y_j, 1}^{\Xset}}-E_{\mathcal{N}_{\y_j, 1}^{\Yset}}\right|\\
  &= \sum_{\x_i\in\Xset}\left|0-E_{\mathcal{N}_{\x_i, 1}^{\Yset}}\right|+ \sum_{\y_j\in\Yset}\left|E_{\mathcal{N}_{\y_j, 1}^{\Xset}}-0\right|\\
  &= \sum_{\x_i\in\Xset}\min_{\y_j\in \Yset}d_{\x_i,\y_j}+ \sum_{\y_j\in\Yset}\min_{\x_i\in \Xset}d_{\x_i,\y_j}\\
  &=\CD\nonumber
    \end{aligned}
\end{equation}

\end{proof}
\begin{figure}[t]
  \centering
  {\includegraphics[width=0.8\linewidth]{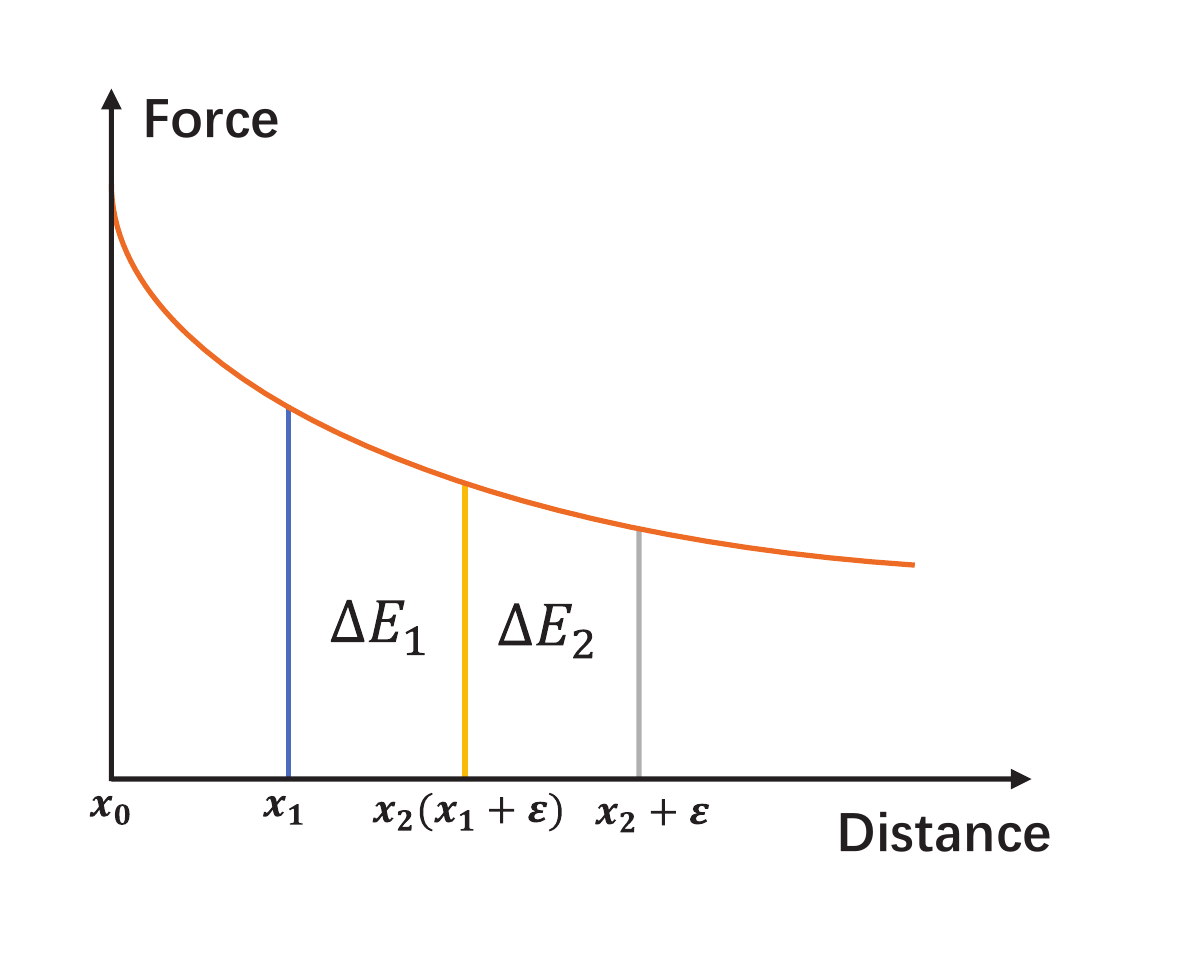}}
  \centering
  \caption{Force-distance curves. }
  \label{fig:proof}
  \end{figure}

{\bf Theorem 1} explains that CD is a special case of PED. We will make more detailed analysis about the comparison between CD and PED in both human and machine perception tasks, i.e., Section \ref{sec:analysis for human} and Section \ref{sec:analysis for machine}.  
\begin{theorem} \label{thm:theorem2}
When $m_1 = m_2$, $\left\| \x_1^O-\x_0^O \right\|_2 \leq \left\| \x_2^O-\x_0^O \right\|_2$, $g$ satisfies {\bf Requirement 1}. For an isometrical perturbation $\epsilon \in \R$, we have
\begin{equation*}
E_{\x_1 + \epsilon(\x_1-\x_0)} - E_{\x_1}  \geq  E_{\x_2 + \epsilon (\x_2-\x_0)} - E_{\x_2}.
\end{equation*}
\end{theorem}\label{thm:theorem_1}
\begin{proof}
$F=ma=mg_\x$, $F$ represents force, and $a$ represents the accelerated speed.  Assuming in Euclidean distance space, we plot the variation of $F$ corresponding to distance $h$ in Fig. \ref{fig:proof},

\noindent$\because \left\| \x_1-\x_0 \right\|_2 \leq \left\| \x_2-\x_0 \right\|_2$, $\epsilon$ is an isometrical perturbation.

\noindent$\therefore E_{\x_1 + \epsilon(\x_1-\x_0)} - E_{\x_1}-(E_{\x_2 + \epsilon(\x_2-\x_0)} - E_{\x_2}) = \Delta E_1 -\Delta E_2 \geq  0. $
\end{proof}

{\bf Theorem 2} explains that with two points share the same mass while different distances to zero potential plane, PED can detect the geometrical isometrical distortion via total energy variation.

\begin{figure}[pt]
	\centering
    \includegraphics[width=0.9\linewidth]{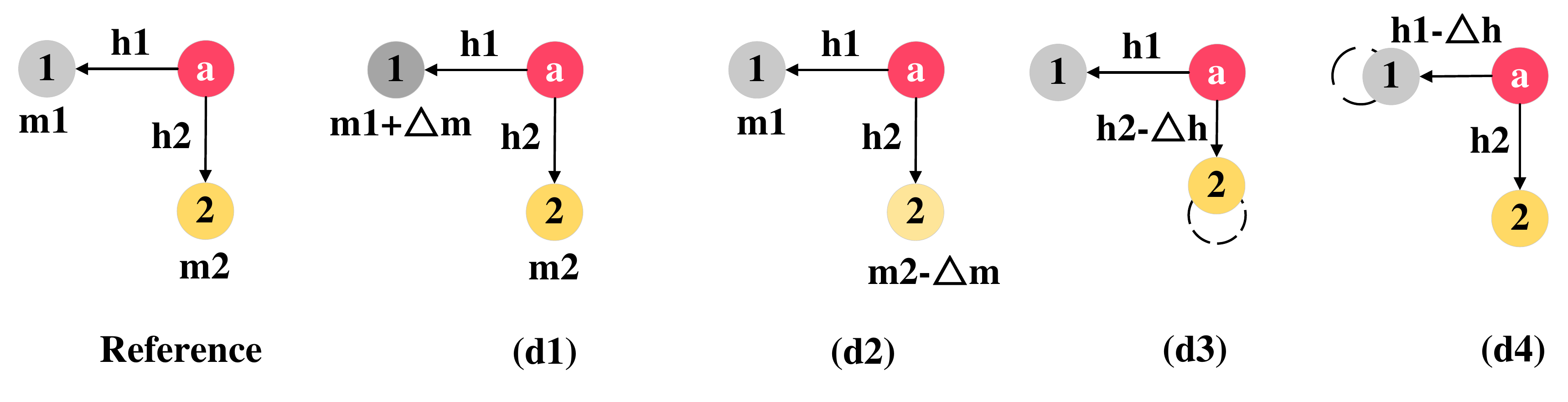}%
	\caption{Toy examples of geometrical and attributive distortion sensitivity.}
	\label{fig:case}
\end{figure}

We further check the three conditions for an ideal distortion quantification proposed in Section \ref{sec:intro}. 

$\bullet$ Differentiability. Refer to specific implementations illustrated in Section \ref{sec:imple_hv} and Section \ref{sec:imple_cv}, all the steps adopted in PED are clearly differentiable;

$\bullet$ Low complexity. The computational bottleneck of both
PED and CD is the $K$-nearest-neighbor searching. The complexity of CD is $O(N^2+N)$. In contrast, The complexity of PED is $O(N^2+KN)$, which is comparable to CD. Note the complexity of EMD is $O(N^3)$;

$\bullet$ For distortion discrimination, {\bf Theorem 2} states that PED is sensitive to isometrical distortion, which can not be detected by CD and EMD.


\begin{figure*}[pt]
	\centering
    \includegraphics[width=0.9\linewidth]{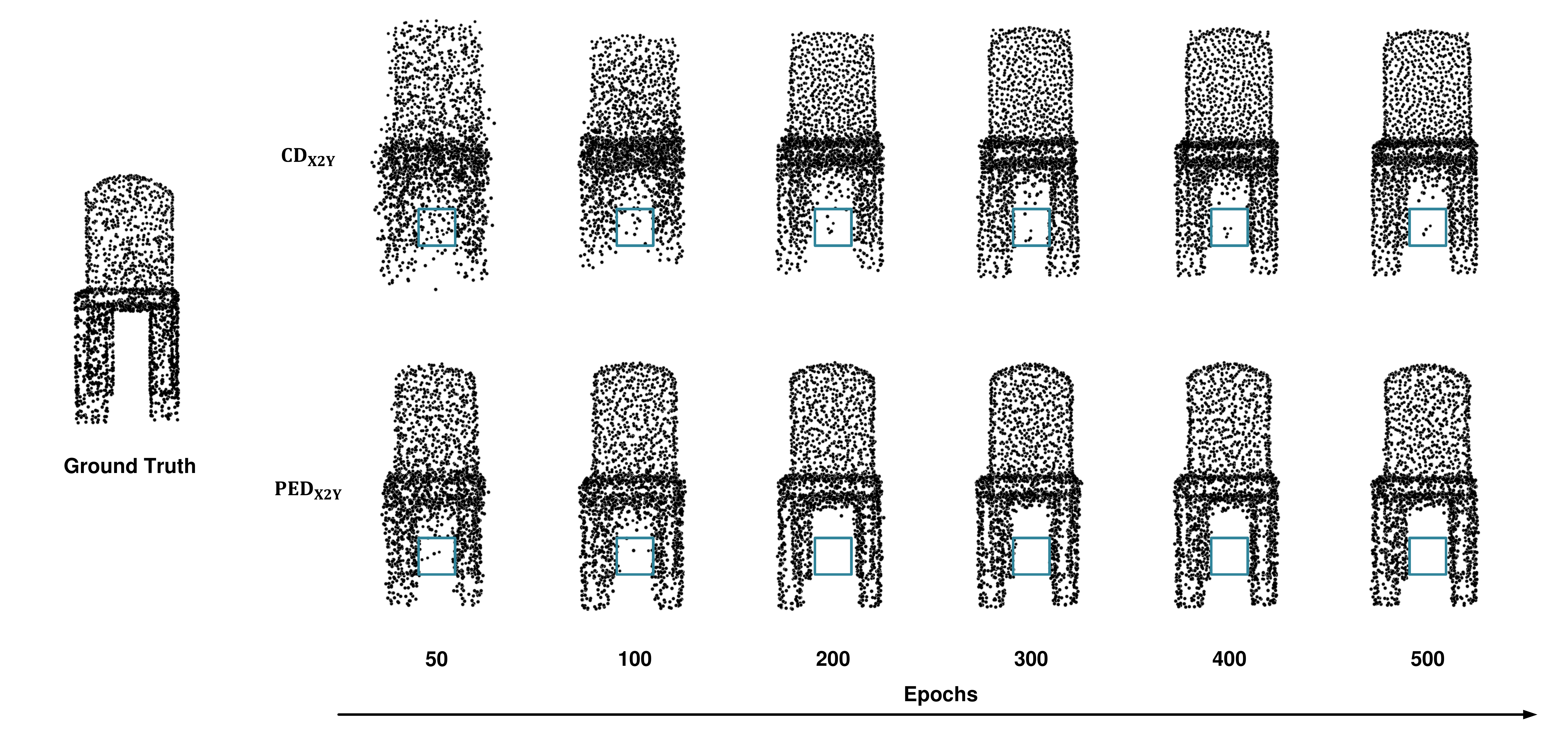}%
 	\caption{An example of point cloud reconstruction task with $\CD_{\Xset2\Yset}$ and $\PED_{\Xset2\Yset}$ as loss functions. When using $\CD_{\Xset2\Yset}$ as the loss function, several points bounded by the blue box are almost not be optimized during the training. }
	\label{fig:X2Y_case}
\end{figure*}

\begin{figure}{}
    \centering
    \subfigure[]{\includegraphics[width=0.48\linewidth]{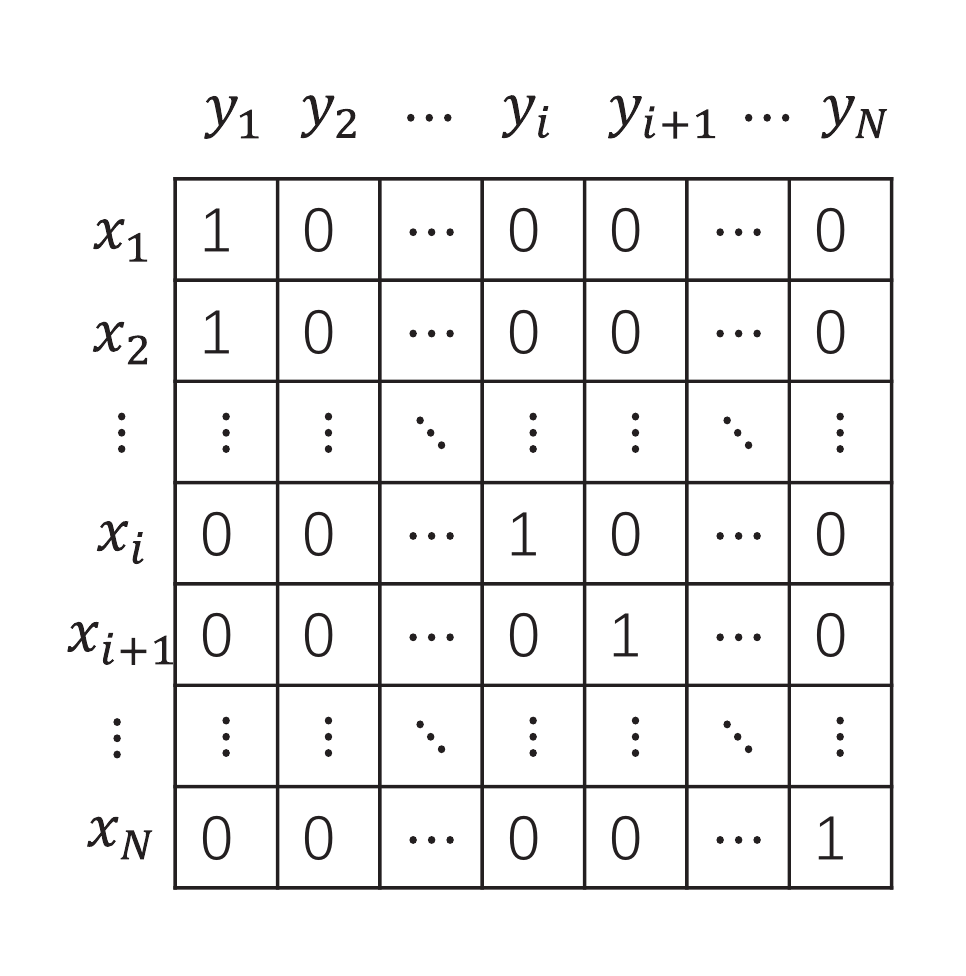}}
    \subfigure[]{\includegraphics[width=0.48\linewidth]{image/X2Y_matrix-eps-converted-to.pdf}}
    \caption{The mutual-incidence matrix for $\CD_{\Xset2\Yset}$ and $\CD_{\Yset2\Xset}$ when $N=M$.}
    \label{fig:incidence matrix}
\end{figure}
\subsection{PED for Human Perception Tasks}\label{sec:analysis for human}

For human perception tasks, we hope that PED can reflect both geometry and attribute distortions mentioned in Section \ref{sec:related_work}. Therefore, we use several toy examples to validate and demonstrate the sensitivity of PED to different distortion. Specifically, we use CD, EMD and $\mathrm{PSNR_{YUV}}$ \cite{torlig2018novel} as comparison, which are  widely used in human perception tasks. For $\mathrm{PSNR_{YUV}}$, it first matches points via Euclidean distance like CD, i.e., $\x_i\in\Xset$, $\y_j\in\Yset$, ${\rm min}(\|\x_i^O-\y_j^O\|_2)$. Then calculating the color difference between paired $\x_i$ and $\y_j$.

Compared with CD, $\mathrm{PSNR_{YUV}}$ can better detect attribute-lossy distortion existed in point clouds.

{\bf Case 1: Sensitivity to attributive isometrical  distortion.} As illustrated in Fig. \ref{fig:case} d1 and d2, $\Delta m$ means a perturbation on three color channels simultaneously.

$\diamond$ $\mathrm{CD_{d1}}=0 =\mathrm{CD_{d2}}$;

$\diamond$ $\mathrm{EMD_{d1}}=0 =\mathrm{EMD_{d2}}$;

$\diamond$ $\mathrm{PSNR_{YUV}^{d_1}}=\mathrm{PSNR_{YUV}^{d_2}}$;

$\diamond$ Because of $h_1 \neq h_2$, $\mathrm{PED}^{d_1}\neq\mathrm{PED}^{d_2}$.

{\bf Case 2: Sensitivity to geometrical isometrical distortion.} As illustrated in Fig. \ref{fig:case} d3 and d4, $\Delta h$ means a perturbation on distance.

$\diamond$ $\mathrm{CD_{d3}}=\Delta h =\mathrm{CD_{d4}}$;

$\diamond$ $\mathrm{EMD_{d3}}=\Delta h =\mathrm{EMD_{d4}}$;

$\diamond$ $\mathrm{PSNR_{YUV}^{d_3}}=0=\mathrm{PSNR_{YUV}^{d_4}}$;

$\diamond$ Because of $h_1 \neq h_2$, $m_1 \neq m_2$,  $\mathrm{PED}^{d_1}\neq\mathrm{PED}^{d_2}$ under most cases.

From the above cases, we can conclude that MPED is more sensitive to different distortions than previous distortion quantifications, which is necessary in human perception tasks. More cases can be found in appendix.

\subsection{PED for Machine Perception Tasks}
\label{sec:analysis for machine}

As mentioned in Section \ref{sec:related_work}, CD is the most favored loss function due to its low computational cost. Therefore, we will focus on the comparison between CD and PED in this section. 

Refer to Eq. (\ref{eq:PED_1}), we rewrite PED as $\PED=\PED_{\Xset2\Yset}+\PED_{\Yset2\Xset}$, and divide CD into two parts, i.e., ${\rm CD}={\rm CD}_{\Xset2\Yset}+{\rm CD}_{\Yset2\Xset}$. We first summarize the impact of $\Xset2\Yset$ and $\Yset2\Xset$ parts, and then give detailed analysis for better understanding and comparison.

In summary, our analysis show that, for $\Xset2\Yset$ part, $\PED_{\Xset2\Yset}$ can reduce the possibility that the points cannot be optimized during the backward propagation compared to $\CD_{\Xset2\Yset}$; for $\Yset2\Xset$ part, $\PED_{\Yset2\Xset}$ can better prevent generated samples from point-collapse.

{\bf Theorem 1} have shown the relation between CD and PED. For the sake of fair comparison, we set $p=2$ and $\Cset = \Xset\bigcup\Yset$ for PED, which is consistent with the conditions in {\bf Theorem 1} and implementation details in Section \ref{sec:imple_cv}.

\subsubsection{Analysis for $\Xset2\Yset$ Part}\label{X2Ypart}
To better illustrate the relationship between points of $\Xset$ and $\Yset$, we introduce a \textit{mutual-incidence matrix} $\mathbf{B}=\left\{b_{ij}\right\}\in\R^{N\times M}$, where
\begin{equation}
\begin{aligned}
   b_{ij} &= 
   \begin{cases}
   1, \quad \text{if\quad $\y_j \in {\N_{\x_i,K}^{\Yset}}$},\\
   0, \quad \text{otherwise}.
   \end{cases}
\end{aligned}\nonumber
\end{equation}
According to {\bf Theorem 1}, we know CD is a special case of PED when $K=1$. Therefore, for either ${\rm CD}_{\Xset2\Yset}$ or ${\rm PED}_{\Xset2\Yset}$, we can have
$\sum_{j}b_{ij} = K$,
that is, the row sum of the incidence matrix indicates the neighborhood size $K$.

Although the row sum of the incidence matrix is constant, the column sum is indeterminate. Actually, the $j$-th column sum indicates the number of points in $\Xset$ that can find $\y_j$ in their $K$-nearest neighborhoods. We formulate the $j$-th column sum as  $\sum_{i}b_{ij}=\left| \N_{\Xset2\y_j}\right|$, where
$\N_{\Xset2\y_j}=\left\{\x_i|\y_j \in \N_{\x_i,K}^{\Yset} \right \}$.

Fig. \ref{fig:incidence matrix} (a) shows an example of the mutual-incidence matrix of ${\rm CD}_{\Xset2\Yset}$ when $N=M$. In Fig. \ref{fig:incidence matrix} (a), the row sum of $\mathbf{B}_{\rm CD}$ is equal to $1$. However, when we focus on the column sum of $\mathbf{B}_{\rm CD}$, we can see the column sum of $\mathbf{B}_{\rm CD}$ is inconstant. For instance, for $1$-th column, which corresponds to $\y_1$, we have $\sum_{i}b_{i1}=2$; while for the $2$-th column which corresponds to $\y_2$, we have $\sum_{i}b_{i2}=0$. Similar to $\CD_{\Xset2\Yset}$, the mutual-incidence matrix of $\PED_{\Xset2\Yset}$, noted as $\mathbf{B}_{\rm PED}$, also has constant row sum $K$ and inconstant column sum.

{\bf Derivative Calculation.}
To better compare the effects of CD and PED in the optimization, we respectively explore their derivatives. For a reconstructed point $\yhat$ in $\Yset$, the partial derivative of $\CD_{\Xset2\Yset}$ w.r.t. $\yhat$ is  
\begin{equation}
    \begin{aligned}
       \frac{\partial \CD_{\Xset2\Yset}}{\partial{\yhat}}=
       \begin{cases}
       2\sum\limits_{\x_i \in \N_{\Xset2\yhat}}\left( \yhat - \x_i\right),&\sum\limits_{i}b_{ij}> 0  \\
       0, &\sum\limits_{i}b_{ij}=0. \\
       \end{cases}
    \end{aligned}\nonumber
\end{equation}

For the partial derivative of $\rm PED_{\Xset2\Yset}$ w.r.t. $\yhat$, we can set, without loss of generality, the spatial field function $g(\cdot)=1$. Then we have
\begin{equation}\label{eq:x2y_grad}
    \begin{aligned}
        \frac{\partial \PED_{\Xset2\Yset}}{\partial{\yhat}}
       =\begin{cases}
       2\sum\limits_{\x_i \in \N_{\Xset2\yhat}} \left( \yhat - \x_i^{pse}\right),&\sum\limits_{i}b_{ij}>0  \\
       0,&\sum\limits_{i}b_{ij}=0. \\
       \end{cases}
    \end{aligned}
\end{equation}
The specific formulation of $\x_i^{pse}$ has been shown in Eq. (\ref{eq:pseudo_point}).
Obviously, when $K=1$, we have $\x_i^{pse}=\mathbf{x}_i$, which leads to $ {\partial {\PED_{\mathbf{X}2 \mathbf{Y}}}/{\partial{\hat{\mathbf{y}}}}=\partial \CD_{\mathbf{X}2 \mathbf{Y}}}/{\partial{\hat{\mathbf{y}}}}$. 

{\bf Analysis and Validation.}
For all points in $\mathbf{Y}$, we can conclude the derivative formulation of CD or PED as 
\begin{equation}\label{eq:x2y_gradient}
     \frac{\partial {\rm CD}({\rm PED})_{\Xset2 \Yset}}{\partial{\Yset}}=2(\mathbf{D}\mathbf{Y}-\mathbf{B}^T\mathbf{X}^{pse}),
\end{equation}
where $\mathbf{D}={\rm diag}(\mathbf{B}^T\mathbf{1})\in\R^{M\times M}$.

According to Eq. (\ref{eq:x2y_gradient}), we can see that the derivative of $\mathbf{y}_j$ is equal to $0$ when $\sum_{i}b_{ij}=0$, which means that $\mathbf{y}_j$ can not be optimized during the backward propagation. We regard such points  as \textit{isolated points} in the optimization and other points as \textit{optimizable points}. The possibility that $\y_j$ is not an isolated point in one iteration can be formulated as $p(\sum_{i}b_{ij}>0)$. Obviously, we have
{\setlength\abovedisplayskip{1pt}
\setlength\belowdisplayskip{1pt}
\begin{equation}\label{eq:x2yproof}
    \begin{aligned}
    p_{\CD}(\sum_{i}b_{ij}>0)&=1-(\frac{M-1}{M})^N\\
    \leqslant 1-(\frac{M-K}{M})^N&=p_{\PED}(\sum_{i}b_{ij}>0).
    \end{aligned}
\end{equation}}
Eq. (\ref{eq:x2yproof}) shows that for $\Xset2\Yset$ part, PED can better avoid the existence of isolated points by selecting proper neighborhood size $K$, which endows the generated samples with better spatial distribution.

To validate our analysis, we illustrate the training process of point cloud reconstruction task in Fig. \ref{fig:X2Y_case}. Specifically, we respectively use $\CD_{\Xset2\Yset}$ and $\PED_{\Xset2\Yset}$ ($K=10$ and $g(\cdot)=1$) as the loss functions to train the LatentNet \cite{achlioptas2018learning} and observe the generated samples with the increase of training epochs. For $\CD_{\Xset2\Yset}$, we can see that some points in the generated samples always locate under the seat part (i.e., the region bounded by the blue box in Fig. \ref{fig:X2Y_case}), which states that they are almost not be optimized during the training. We regard these point as isolated points. In contrast, generated samples using $\PED_{\Xset2\Yset}$ as the loss function avoid such trap.

\subsubsection{Analysis for $\Yset2\Xset$ Part}
Similar to the analysis of $\Xset2\Yset$ part, we also introduce a mutual-incidence matrix $\mathbf{B}=\left\{b_{ij}\right\}\in\R^{N\times M}$ 
\begin{equation}
\begin{aligned}
   b_{ij} &= 
   \begin{cases}
   1, \quad \text{if\quad $\x_i \in {\N_{\y_j,K}^{\mathbf{X}}}$}\\
   0, \quad \text{otherwise}.
   \end{cases}
\end{aligned}\nonumber
\end{equation}
For either ${\CD}_{\Yset2\Xset}$ or ${\PED}_{\Yset2\Xset}$, we have
$\sum_{i}b_{ij} = K$,
that is, the column sum of $\mathbf{B}$ indicates the neighborhood size $K$. In contrast to $\Xset2\Yset$ part, the row sum of $\mathbf{B}$ can be written as $\sum_{j}b_{ij}=\left| \N_{\Yset2 \x_i}\right|$, where
$\N_{\Yset2\x_i}=\left\{\y_j|\x_i \in \N_{\y_j,K}^{\Xset} \right\}$.

The $i$-th row sum indicates the number of points in $\mathbf{Y}$ that can find $\mathbf{x}_i$ in their $K$-nearest neighborhoods. Fig. \ref{fig:incidence matrix} (b) shows an example of the incidence matrix of $\CD_{\mathbf{Y}2\mathbf{X}}$.
 In Fig. \ref{fig:incidence matrix} (b), we can see the column sum of $\mathbf{B}_{\CD}$ is equal to $1$. Meanwhile, the row sum of $\mathbf{B}_{\CD}$ is uncertain.

{\bf Derivative Calculation.}
Here we explore the derivatives of both CD and PED. For a reconstructed point $\hat{\mathbf{y}}$ in $\mathbf{Y}$, the partial derivative of $\CD_{\mathbf{Y}2\mathbf{X}}$ w.r.t. $\hat{\mathbf{y}}$ is  
\begin{equation}\label{eq:grad_cd_y2x}
        \frac{\partial \CD_{\Yset2\Xset}}{\partial{\yhat}}=
     2(\yhat-\x_i),\quad \x_i= \N_{\yhat,1}^{\Xset}.
\end{equation}

For the partial derivative of $\rm PED_{\Yset2\Xset}$ w.r.t. $\yhat$, we set the spatial field function $g(\cdot)=1$. Considering that $\yhat$ can be found as neighbor of other points in $\N_{\Yset2\yhat}$ (note as $R$), or as the neighborhood center itself (note as T), we divide the partial derivative into two parts as follows,
\begin{small}
\begin{equation}\label{eq:grad_ped_y2x}
    \begin{aligned}
        \frac{\partial \PED_{\Yset2\Xset}}{\partial{\yhat}} &= R + T \\
       R &= 2\sum_{\y_j \in \N_{\Yset2\yhat}} {\rm sgn}( E_{\N_{\y_j,K}^{\Yset}}-E_{\N_{\y_j,K}^{\Xset}})\left( \yhat - \y_j\right)  \\
       &= 2\sum_{\y_j \in \N_{\Yset2 \yhat}} \left( \yhat - \y_j^{pse}\right) \\
       T & =\frac{\partial\left|E_{\N_{\yhat,K}^{\Xset}}-E_{\N_{\yhat,K}^{\Yset}}\right|}{\partial\yhat}\\
       & = 2{\rm sgn}(E_{\N_{\yhat,K}^{\mathbf{X}}}-E_{N_{\yhat,K}^{\Yset}})(\sum_{\y_j \in \N_{\yhat,K}^{\Yset}}\y_j-\sum_{\x_i \in \N_{\yhat,K}^{\Xset}}\x_i),\\
    \end{aligned}
\end{equation}
\end{small}

\noindent when $K=1$, we have $ {\partial {\PED_{\mathbf{Y}2 \mathbf{X}}}/{\partial{\hat{\mathbf{y}}}}=\partial \CD_{\mathbf{Y}2 \mathbf{X}}}/{\partial{\hat{\mathbf{y}}}}$ because $R = 0$ and $T =2(\hat{\mathbf{y}}-\mathbf{x}_i)$.

{\bf Analysis and Validation.}
According to Eq. (\ref{eq:grad_cd_y2x}), we can know why point-collapse usually happens when using CD as the loss function. Specifically, points of $\mathbf{Y}$  corresponding to the same $\mathbf{x}_i$ (e.g., $\mathbf{x}_1$ in Fig. \ref{fig:incidence matrix} (b)) will move to the location of $\mathbf{x}_i$, which leads to denser distribution around $\mathbf{x}_i$ in the target. Meanwhile, those locations around points that can not be found (e.g., $\mathbf{x}_2$ in Fig. \ref{fig:incidence matrix} (b)), will be sparser than the ground truth. In fact, if we only use $\CD_{\mathbf{Y}2\mathbf{X}}$ as the loss function, all points in the target  may converge towards one location, which can be regarded as an extreme point-collapse.      

For the derivative of PED, two terms in Eq. (\ref{eq:grad_ped_y2x}), $R$ and $T$, play different roles in the optimization. Specificallly, $T$ pushes $\yhat$ to converge in the direction that minimizes the energy discrepancy between $\N_{\yhat,K}^{\Xset}$ and $\N_{\yhat,K}^{\Yset}$. Therefore, the information of those points ignored by CD can be introduced. Utilizing the mutual-incidence matrix $\mathbf{B}$, we formulate the possibility that $\x_i$ can not be found in $\Yset2\Xset$ neighbor searching as $p(\sum_{j}b_{ij}=0)$. Then we have
{\setlength\abovedisplayskip{1pt}
\setlength\belowdisplayskip{1pt}
\begin{equation}\label{eq:y2xproof}
    \begin{aligned}
    p_{\CD}(\sum_{j}b_{ij}=0)&=(\frac{N-1}{N})^M\\
    \geqslant (\frac{N-K}{N})^M&=p_{\PED}(\sum_{j}b_{ij}=0).
    \end{aligned}
\end{equation}}
Eq. (\ref{eq:y2xproof}) shows that PED can better avoid the existence of over-sparse or over-populated regions in the target, which reduces the possibility that point-collapse happens.

As for term $R$, it can serve as a refinement term to smooth the generated samples during the training. Considering $\Yset^{pse}=\Yset$, we have
\begin{equation}\label{eq:graphfilter}
\begin{aligned}
    \Yset-\alpha\cdot R = \left[\mathbf{I}-2\alpha(\mathbf{D}- \mathbf{A})\right]\Yset 
    =\left[\mathbf{I}-2\alpha\mathbf{L}\right]\mathbf{Y},
\end{aligned}
\end{equation}
where $\alpha$ is the learning rate; $\mathbf{A}=\left\{a_{ij}\right\}\in\R^{M\times M}$ represents an adjacent matrix, which have
\begin{equation}
\begin{aligned}
   a_{ij} &= 
   \begin{cases}
   1, \quad \text{if\quad $\y_j \in {\N_{\y_i,K}^{\mathbf{Y}}}$}\\
   0, \quad \text{otherwise}.
   \end{cases}
\end{aligned}\nonumber
\end{equation}
$\mathbf{D}=\rm{diag}(\mathbf{A}\mathbf{1})$ and $\mathbf{L}=\mathbf{D}-\mathbf{A}$ represents the graph Laplacian matrix on a directed graph. Considering $\Yset$ as the signal on the directed graph, according to graph signal processing theory \cite{7746675}, $h(\mathbf{L})=\mathbf{I}-2\alpha\mathbf{L}$ in Eq. (\ref{eq:graphfilter}) can  serves as a low-pass graph filter, which makes the point distribution of generated samples more uniform.

Fig. \ref{fig:Y2X_case} shows the generated samples of LatentNet after 500 epochs of training. We can see that, if we only use $\CD_{\mathbf{Y}2\mathbf{X}}$ as the loss function to train the network, all points in the generated sample converge to the same location in the blue box, which may be caused by an unsuccessful initialization. In contrast, $\PED_{\Yset2\Xset}$ prevents generated samples from such extreme collapse and keeps a passable shape.

From the above analysis for both human and machine perception tasks, we can see PED shows potential superiority over CD and other popular distortion quantifications. As the multiscale version of PED, MPED inherits all properties of PED and has better distortion distinguishability,  which provides a theorem basis for the application of MPED in Section \ref{sec:imple_hv} and \ref{sec:imple_cv}.

\begin{figure}[pt]
	\centering
    \includegraphics[width=0.9\linewidth]{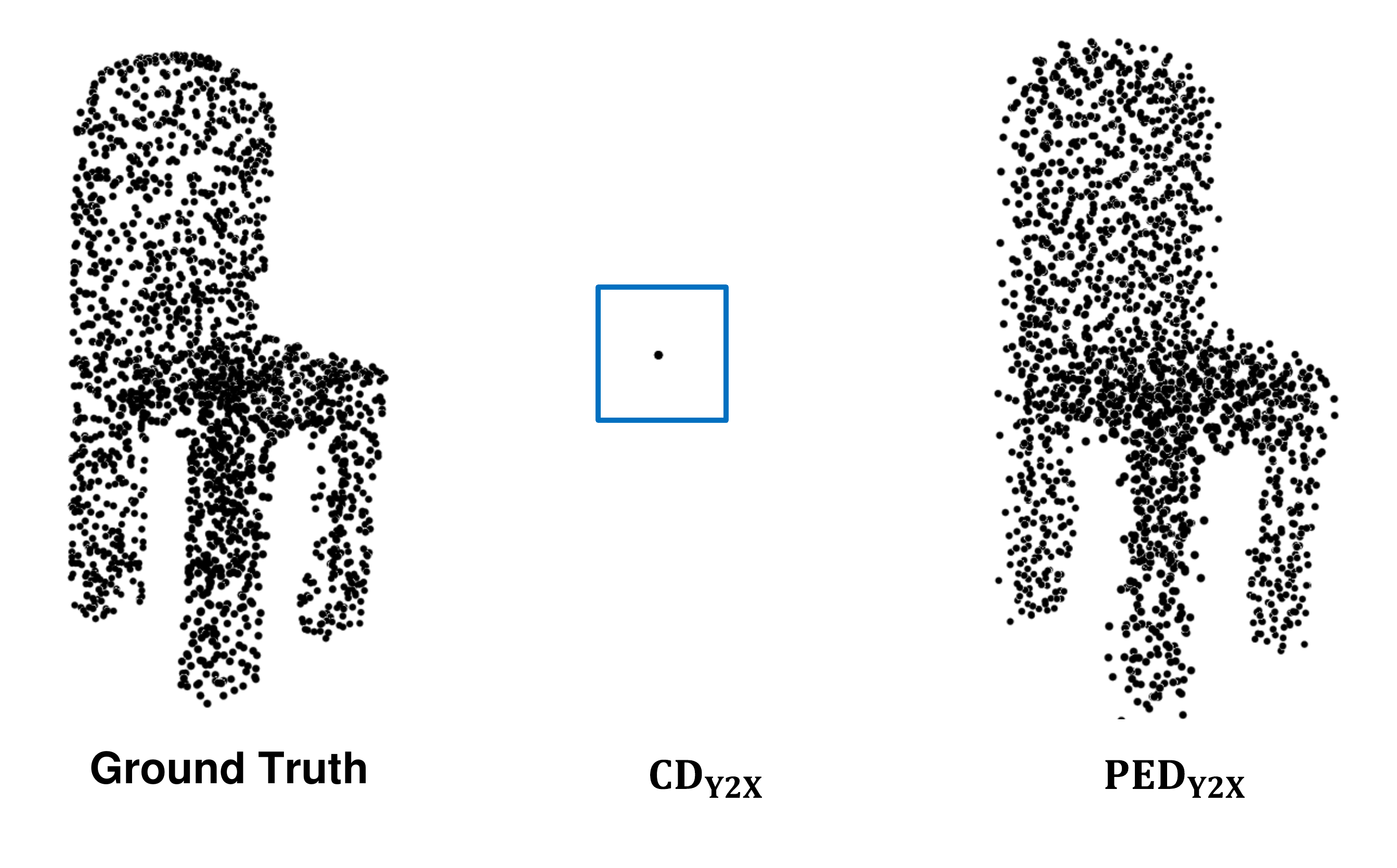}%
	\caption{An example of point cloud reconstruction task with $\CD_{\Yset2\Xset}$ and $\PED_{\Yset2\Xset}$ as loss functions. Note that all points almost converge to the same location bounded in the blue box when using $\CD_{\Yset2\Xset}$ as the loss function.}
	\label{fig:Y2X_case}
\end{figure}

\section{Experiment} \label{sec:experiment}
In this section, we present the experiment results of the proposed MPED on both human and machine perception tasks. 
\subsection{Experimental Evaluations on Human Perception Tasks}\label{sec:exp_hv}
For human perception tasks, we test MPED on three fairly large database, i.e., SJTU-PCQA \cite{yang2020predicting}, LS-PCQA\cite{liu2022point} and WPC \cite{su2019wpc}\cite{new10-wpc3.0} . 

\begin{table*}[pt]
	\caption{Model performance (PLCC, SROCC and RMSE) for point clouds samples in SJTU-PCQA database in terms of different impairments. } \label{Table:single-sjtupcqa}
	\centering
	\begin{scriptsize}
	\setlength{\tabcolsep}{0.7mm}{
		\begin{tabular}{|c|c|c|c|c|c|c|c|c|c|c|c|c|c|c|c|c|c|c|c|c|c|c|c|c|c|c|}
			\hline
			\multicolumn{3}{|c|}{} & \multicolumn{8}{|c|}{PLCC} & \multicolumn{8}{|c|}{SROCC} & \multicolumn{8}{|c|}{RMSE}  \\ \hline
			\multicolumn{3}{|c|}{quantification:}&OT&CN&GGN&DS&D+C&D+G&C+G&{\bf ALL} & OT&CN&GGN&DS&D+C&D+G&C+G&{\bf ALL} & OT&CN&GGN&DS&D+C&D+G&C+G&{\bf ALL} \\  \hline
			\multirow{11}{*}{\rotatebox{90}{SJTU-PCQA}} & \multirow{2}{*}{M}  & p2po &0.75&-&0.88&0.73&0.50&0.86&0.85&0.73&0.81&-&0.81&0.80&0.85&0.84&0.81&0.80&1.32&-&1.23&1.76&2.06&1.27&1.32&1.72 \\ \cline{3-27}
			&  & p2pl 
			&0.78&	-&	0.88&	0.45&	0.42&	0.87&	0.85&	0.65&
            0.81&	-&	0.80&	0.47&	0.53&	0.84&	0.82&	0.66&
            1.26&	-&	1.25&	1.95&	2.15&	1.26&	1.31&	1.89
  \\ \cline{2-27}
			& \multirow{2}{*}{H} & p2po &
			0.73&	-&	0.88&	0.46&	0.44&	0.80&	0.86&	0.63&
            0.78&	-&	0.81&	0.58&	0.49&	0.80&	0.82&	0.66&
            1.38&	-&	1.22&	1.93&	2.12&	1.53&	1.29&	1.92
  \\ \cline{3-27}
			&  & p2pl &
			0.75&	-&	0.88&	0.43&	0.44&	0.86&	0.86&	0.65&
            0.80&	-&	0.81&	0.54&	0.50&	0.83&	0.82&	0.65&
            1.33&	-&	1.22&	1.96&	2.13&	1.3&	1.3&	1.89
  \\ \cline{2-27}
			& \multicolumn{2}{|c|}{$\mathrm{PSNR_{YUV}}$}   &
			0.50&	0.81&	0.7&	0.53&	0.86&	0.65&	0.90&	0.65&
            0.35&	0.77&	0.68&	0.54&	0.87&	0.62&	0.85&	0.65&
            1.74&	0.89&	1.86&	1.84&	1.20&	1.91&	1.09&	1.84

 \\ \cline{2-27}
			& \multicolumn{2}{|c|}{PCQM}   &
			0.80&	0.84&	0.93&	0.84&	0.93&	0.90&	0.95&	0.86&
            0.76&	0.84&	0.91&	0.81&	0.92&	0.88&	0.92&	0.85&
            1.20&	0.83&	0.98&	1.17&	0.86&	1.07&	0.78&	1.24
 \\ \cline{2-27}
			& \multicolumn{2}{|c|}{GraphSIM}   &
			0.80&	0.79&	0.94&	0.9&	0.89&	0.93&	0.95&	0.86&
            0.69&	0.78&	0.92&	0.87&	0.89&	0.89&	0.94&	0.84&
            1.19&	0.92&	0.92&	0.94&	1.06&	0.9&	0.76&	1.25
 \\ \cline{2-27}
			& \multicolumn{2}{|c|}{EPES}   &
			0.82&	0.85&	0.93&	0.92&	0.96&	0.94&	0.96&	0.89&
            0.75&	0.82&	0.89&	0.92&	0.95&	0.91&	0.94&	0.88&
            1.15&	0.79&	0.77&	0.90&	0.69&	0.92&	0.73&	1.12
  \\ \cline{2-27}

            & \multicolumn{2}{|c|}{MPED($p=1$)}   &
            0.81&	0.85&	0.89&	0.94&	0.95&	0.95&	0.98&	\color{red}\textbf{0.90}&
            0.68&	0.82&	0.88&	0.94&	0.93&	0.93&	0.97&   
            {0.88}&
            1.17&	0.80&	0.97&	0.81&	0.82&	0.83&   0.50&	\color{red}\textbf{1.08}
 \\ \cline{2-27}
			& \multicolumn{2}{|c|}{MPED($p=2$)}   &
			0.83&	0.84&	0.92&	0.94&	0.97&	0.95&	0.98&	\color{red}\textbf{0.90}&
            0.76&	0.81&	0.90&	0.94&	0.95&	0.93&	0.97&	\color{red}\textbf{0.89}&
            1.11&	0.82&	0.85&	0.79&	0.66&	0.78&	0.53&	\color{red}\textbf{1.08}
 \\
\hline
	\end{tabular}}
	\end{scriptsize}

\end{table*}

{\bf Subjective Point Cloud Assessment Databases.}

{$\bullet$ SJTU-PCQA database.}
There are 9 high-quality point cloud samples in SJTU-PCQA.
Each native point cloud sample is augmented with 7 different types of impairments under 6 levels, including four individual distortions, Octree-based compression (OT), Color noise (CN), Geometry Gaussian noise (GGN), Downsampling (DS), and three superimposed distortions, such as Downsampling and Color noise (D+C), Downsampling and Geometry Gaussian noise (D+G), Color noise and Geometry Gaussian noise (C+G). 
In all 378 samples in SJTU-PCQA are provided with mean opinion scores (MOS).

{$\bullet$ LS-PCQA database.}
There are 104 high-quality point cloud samples in LS-PCQA, and each reference sample is processed with 34 types of impairments under 7 levels. In all there are 24,752 samples in LS-PCQA, 1,020 of them provide MOS and used in our experiments. The distortion types include Quantization noise, Contrast change, V-PCC, G-PCC, Local rotation, Luminance noise, and so on. 

{$\bullet$ WPC database.}
There are 20 voxelized point cloud samples in WPC with an average of 1.35 million points. 740 distorted point clouds with corresponding MOS are generated from the references under five types of distortions, including Downsampling, Gaussian noise contamination, G-PCC(Trisoup), G-PCC(Octree) and V-PCC.

{\bf Parameters of MPED.}
i) $\Psi$. For human perception tasks, we set $\Psi$ as [10, 5] for both $p =1$ and $p =2$;  ii)  $L$ and $\beta$. Refer to \cite{yang2020inferring}, we set $L=N/10000$ and $
\beta=4$;
iii) $\sigma$. we simply set $\sigma=1$; iv) $k_j$. We first using RGB color space to calculate $m_{\x_i}$, therefore, we set $k_R:k_G:k_B=1:2:1$.

{\bf Performance Evaluation.} 
We compare our MPED with another 7 state-of-the-art distortion quantifications, i.e.,
\begin{itemize}
    \item PSNR-MSE-P2point (M-p2po)
    \item PSNR-MSE-P2plane (M-p2pl)
    \item PSNR-Hausdorff-P2point (H-p2po)
    \item PSNR-Hausdorff-P2plane (H-p2pl)
    \item $\mathrm{PSNR_{YUV}}$
    \item PCQM
    \item GraphSIM
    \item EPES
\end{itemize}

Note EPES is a correlational research who also use energy-based features to measure dense point cloud distortion. The differences between EPES and MPED are manifold. For example, EPES is designed for human vision tasks solely, it pools global and local features refer to the form of SSIM and uses "similarity" to normalize the final scores into [0, 1]. While MPED propose to use multiscale features and the final scores are normalized into "distance" which is unified with the utilization in machine vision tasks. The design of features in EPES is obviously different from MPED, please check \cite{new1-xu2021epes} for more details.

To ensure the consistency between subjective scores (e.g., mean opinion scores) and objective predictions from various models, we map the objective predictions of different models to the same dynamic range following the recommendations suggested by the video quality experts group (VQEG) \cite{video2003final,sheikh2006statistical}, to derive popular PLCC for prediction accuracy, SROCC for prediction monotonicity, and RMSE for prediction consistency for evaluating the model performance. The larger PLCC or SROCC comes with better model performance. On the contrary, the lower RMSE is better.
More details can be found in~\cite{video2003final}. To map the dynamic range of the scores from objective quality assessment models into a common scale, the logistic regression recommended by VQEG is used. 

{$\bullet$ SJTU-PCQA database.}
Table \ref{Table:single-sjtupcqa} presents the performance of MPED and other state-of-the-art distortion quantifications on SJTU-PCQA. We see that MPED presents best overall performance on SJTU-PCQA. Specifically, the overall PLCC, SROCC and RMSE of $\mathrm{MPED}(p=1)$ is (0.90, 0.89, 1.07) and $\mathrm{MPED}(p=2)$ is (0.90, 0.88, 1.07). While M-p2po is (0.73, 0.80, 1.72), M-p2pl is (0.65, 0.66, 1.89), H-p2po is (0.63, 0.66, 1.82), H-p2pl is (0.65, 0.65, 1.89),  $\mathrm{PSNR_{YUV}}$ is (0.65, 0.65, 1.84), PCQM is (0.86, 0.85, 1.24),  GraphSIM is (0.86, 0.84, 1.25) and EPES is (0.89, 0.88, 1.12).

For different distortion quantifications in terms of distortion types. We see that: i) M-p2po, M-p2pl, H-p2po, and H-p2pl cannot handle CN because they only consider point-wise geometrical features; ii) $\mathrm{PSNR_{YUV}}$ presents reliable performance on CN, while the worst performance on OT. Referring to Section \ref{sec:analysis for human}, $\mathrm{PSNR_{YUV}}$ first matches two points via nearest neighbor searching, then uses the color difference of point pairs as distortion measurement. Essentially, OT uses a central point to replace all the points within a spatial cube. The central points usually share similar attributes with replaced points, and can be used multiple times during point matching. Therefore, $\mathrm{PSNR_{YUV}}$ are not sensitive to OT.

{$\bullet$ LS-PCQA and WPC database.}
\begin{table}[pt]
	\caption{Model performance (PLCC, SROCC and RMSE) for point clouds samples in LS-PCQA and WPC databases. } \label{Table:LS-PCQA}
	\centering
	\begin{scriptsize}
	\setlength{\tabcolsep}{0.7mm}{
		\begin{tabular}{|c|c|c|c|c|c|c|c|}
			\hline
		    \multicolumn{2}{|c|}{} & \multicolumn{3}{|c|}{LS-PCQA} &\multicolumn{3}{|c|}{WPC} \\ \hline
			\multicolumn{2}{|c|}{quantification} & {PLCC} &{SROCC} &{RMSE} & {PLCC} &{SROCC} &{RMSE}  \\ \hline
		    \multirow{2}{*}{M}  & p2po &0.46&0.26&0.73 &0.58 &0.57 &18.71\\ \cline{2-8}
			 & p2pl &0.42&0.24&0.75 &0.49&0.45&20.02 \\ \hline
			\multirow{2}{*}{H} & p2po &0.36&0.21&0.77 &0.40 &0.26 &20.98  \\ \cline{2-8}
			  & p2pl &0.36&0.21&0.77 &0.39 &0.31 &21.12 \\ \hline
			\multicolumn{2}{|c|}{$\mathrm{PSNR_{YUV}}$}   &0.50&0.48&0.72 &0.55&0.54&19.13\\ \hline
			\multicolumn{2}{|c|}{PCQM}   &0.32&0.42&0.75 &\color{red}\textbf{0.75}&\color{red}\textbf{0.74}&\color{red}\textbf{15.13} \\ \hline
			\multicolumn{2}{|c|}{GraphSIM}   &0.33&0.31&0.78 &0.69&0.68&16.50  \\ \hline
			\multicolumn{2}{|c|}{EPES}   &0.61&0.55&0.65 &0.52&0.49&19.57 \\ \hline
            \multicolumn{2}{|c|}{MPED($p=1$)}   &0.63&0.57&0.64 &0.70&0.68&16.37\\ \hline
		    \multicolumn{2}{|c|}{MPED($p=2$)}   &{\color{red}\textbf{0.66}}&{\color{red}\textbf{0.61}}&{\color{red}\textbf{0.62}} &0.64&0.63&17.58\\
\hline
	\end{tabular}}

	\end{scriptsize}
\end{table}
Table \ref{Table:LS-PCQA} presents the overall performance of MPED and other state-of-the-art distortion quantifications on LS-PCQA and WPC. We see that: i) MPED provides the best performace on LS-PCQA and comparable performance on WPC. Specifically, the PLCC and SROCC of $\mathrm{MPED} (p=2)$ on LS-PCQA and WPC are above 0.6, while other distortion quantifications cannot perform well on both two databases; ii) the SROCCs of M-p2po, M-p2pl, H-p2po and H-p2pl are lower than 0.3 on LS-PCQA. The reason is that there are several types of geometry lossless but color lossy distortion, such as color quantization dither (CQD). These four distortion quantifications cannot detect color distortion.

{\bf Ablation Study.}\label{sec:ablation_hv}

{$\bullet$ Color Space}

We have exemplified the effectiveness of MPED based on the RGB-based color channel decomposition. In this part, we test the performance of MPED in terms of other color channel space, e.g., YUV and Gaussian color model (GCM)\cite{geusebroek2000color}. YUV and GCM are two color channel spaces that consist of one luminance component and two chrominance components. Given that human visual system is more sensitive to the luminance component~\cite{torlig2018novel}, we set $k_l=6$, $k_{c1}=k_{c2}=1$~\cite{yang2020inferring}. $k_l$ represents the weighting factors of luminance component, and $k_{c1}$, $k_{c2}$ for chrominance components. We use SJTU-PCQA as the test database and the results shown in Table \ref{tab:color_space}.

\begin{table}[t]
	\caption{Model performance with various color spaces on SJTU-PCQA.} \label{tab:color_space}
	\centering
	\begin{tabular}{|c|c|c|c|}
		\hline
		Color Space & PLCC & SROCC & RMSE \\ \cline{1-4}
		RGB & 0.8954 & 0.8877 & 1.0808 \\ \cline{1-4}
		YUV & 0.8940 & 0.8853 & 1.0872 \\ \cline{1-4}
		GCM & 0.8935 & 0.8849 & 1.0899 \\ \cline{1-4}
	\end{tabular}
\end{table}

In Table \ref{tab:color_space}, we see that the performance of MPED is very close for multiple color spaces, which proves the robustness of MPED.

\subsection{Experimental Evaluation on Machine Perception Tasks} 

 For machine perception, we test MPED on three typical unsupervised learning tasks, i.e., point cloud reconstruction, shape completion, and upsampling.

{\bf Parameters of MPED.}
i) $\sigma$. We set $\sigma=10^{-10}$; ii) $\Psi$. For machine perception, we set $\Psi=[10, 5, 1]$ for $p =2$;
\begin{figure}[pt]
	\centering
    \includegraphics[width=0.9\linewidth]{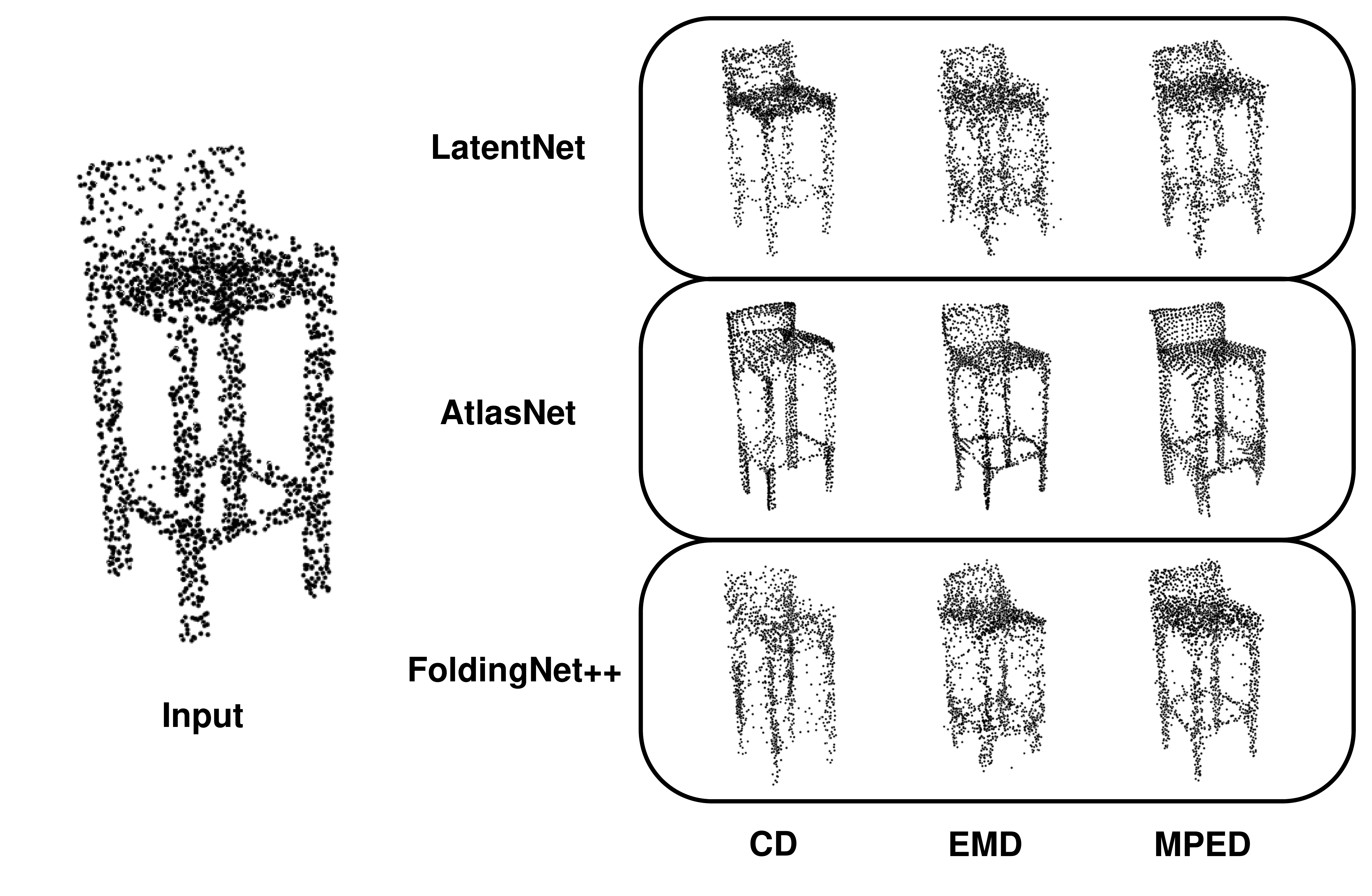}%
	\caption{Illustration of point cloud reconstruction. The proposed MPED presents the best reconstruction quality. }
	\label{fig:generation}
\end{figure}

{\bf Performance Evaluation.}
To validate the proposed MPED as a loss function for machine perception tasks, we compare it with CD\cite{ravi2020pytorch3d} and EMD\cite{PytorchEMD}. For each task, with various supervisions, we test several state-of-the-art networks under the same training settings, e.g., same network, same epochs, and same network parameters. To compare the final performance, we use CD, EMD, and JSD as objective quantitative evaluation. JSD measures the marginal distributions between the generated and reference samples \cite{achlioptas2018learning}. To distinguish CD as loss function and quantitative metrics, we use \textbf{CD} as loss function while CD as the quantitative metric. The EMD in the same way.

{$\bullet$ 3D Point Cloud Reconstruction}

\begin{table*}
\caption{Point cloud reconstruction on ShapeNet. The proposed MPED grants most of the top-level performance on different networks. }\label{TABLE:generation_table}
\begin{center}
\setlength{\tabcolsep}{0.4mm}{
\begin{tabular}{|c|c|c|c|c|c|c|c|c|c|c|c|c|c|c|c|c|}
\hline
\multirow{2}{*}{Model} & \multirow{2}{*}{Loss}&\multicolumn{5}{|c|}{CD($\times$10e-4)} &\multicolumn{5}{|c|}{EMD($\times$10e-1)}  &\multicolumn{5}{|c|}{JSD($\times$10e-2)}  \\ \cline{3-17}
& &Chair &Air &Table &Rifle & Ave &Chair &Air &Table &Rifle & Ave &Chair &Air &Table &Rifle & Ave \\
\hline\hline
\multirow{3}{*}{{LatentNet}}  & CD &\textcolor{red}{\textbf{7.18}}	&\textcolor{red}{\textbf{2.57}}	&\textcolor{red}{\textbf{6.65}}	&\textcolor{red}{\textbf{2.83}}	&\textcolor{red}{\textbf{4.81}}	&9.46	&4.38	&11.37	&9.38	&8.65	&2.91	&2.12	&3.05	&3.44	&2.88\\

 & EMD &9.40  &3.66	&9.37	&4.3	&6.68	&\textcolor{red}{\textbf{3.47}}	&\textcolor{red}{\textbf{1.56}}	&\textcolor{red}{\textbf{3.80}}	&\textcolor{red}{\textbf{3.87}}	&\textcolor{red}{\textbf{3.18}}	&1.10	&1.50	&1.14	&2.35	&1.52 \\

 & MPED  &7.49	&2.72	&7.58	&3.25	&5.26	&5.90	&3.03	&5.38	&6.60	&5.23	&\textcolor{red}{\textbf{0.31}}	&\textcolor{red}{\textbf{0.65}}	&\textcolor{red}{\textbf{0.26}}	&\textcolor{red}{\textbf{0.67}}	&\textcolor{red}{\textbf{0.47}}\\
 \hline \hline
 \multirow{3}{*}{{AtlasNet}} &CD &\textcolor{blue}{\textbf{7.79}}	&\textcolor{blue}{\textbf{3.33}}	&\textcolor{blue}{\textbf{7.20}}	&\textcolor{blue}{\textbf{3.47}}	&\textcolor{blue}{\textbf{5.45}}	&9.46	&4.42	&6.67	&9.60	&7.54	&2.91	&3.59	&1.36	&12.63	&5.12\\

 & EMD &9.40	&3.66	&9.37	&4.30	&6.68	&\textcolor{blue}{\textbf{3.47}}	&\textcolor{blue}{\textbf{1.56}}	&3.80	&\textcolor{blue}{\textbf{3.87}}	&\textcolor{blue}{\textbf{3.18}}	&1.10	&\textcolor{blue}{\textbf{1.50}}	&1.14	&2.35	&1.52\\ 

 & MPED &11.33	&3.93	&8.53	&19.13	&10.73	&3.99	&2.04	&\textcolor{blue}{\textbf{3.73}}	&4.87	&3.66	&\textcolor{blue}{\textbf{0.92}}	&1.74	&\textcolor{blue}{\textbf{0.45}}	&\textcolor{blue}{\textbf{1.49}}	&\textcolor{blue}{\textbf{1.15}}
\\
 \hline \hline
 \multirow{3}{*}{FoldingNet++}&CD  &9.08	&3.08	&7.72	&3.48	&5.84	&10.08	&5.88	&12.86	&10.8	&9.91	&7.21	&6.07	&4.74	&13.30	&7.83\\

 & EMD &10.87	&4.40	&12.73	&6.16	&8.54	&\textcolor{cyan}{\textbf{3.49}}	&\textcolor{cyan}{\textbf{2.29}}	&4.03	&5.17	&\textcolor{cyan}{\textbf{3.75}}	&2.32	&3.67	&2.00	&4.61	&3.15\\

 & MPED &\textcolor{cyan}{\textbf{6.77}}	&\textcolor{cyan}{\textbf{2.39}}	&\textcolor{cyan}{\textbf{6.55}}	&\textcolor{cyan}{\textbf{2.49}}	&\textcolor{cyan}{\textbf{4.55}}	&4.58	&2.39	&\textcolor{cyan}{\textbf{3.98}}	&\textcolor{cyan}{\textbf{5.14}}	&4.02	&\textcolor{cyan}{\textbf{0.44}}	&\textcolor{cyan}{\textbf{0.87}}	&\textcolor{cyan}{\textbf{0.33}}	&\textcolor{cyan}{\textbf{0.64}}	&\textcolor{cyan}{\textbf{0.57}}
\\
\hline
\end{tabular}}
\end{center}

\end{table*}
Here we adopt ShapeNet dataset used in \cite{achlioptas2018learning} and validate the proposed MPED on three typical  networks for point cloud reconstruction with default network parameter setting expect training epochs, e.g., LatentNet \cite{achlioptas2018learning}, AtlasNet\cite{groueix2018papier}, and FoldingNet++ \cite{chen2019deep}. For a fair comparison, we randomly set training epochs of LatentNet, AtlesNet and FoldingNet++ as 500, 300, and 100 due to the default training epochs selected by presenting best CD value on the test sequences. Four sub-categories from ShapeNet are used as test sequences, e.g., ``Chair'', ``Airplane'', ``Table'' and ``Rifle''. Specifically, for each sub-category, we set 75\% as training set and 25\% as testing set. Table \ref{TABLE:generation_table} shows the reconstruction performance of three networks with various supervisions. We use {\textcolor{red}{\textbf{red}}}, {\textcolor{blue}{\textbf{blue}}} and {\textcolor{cyan}{\textbf{cyan}}} to highlight best performance for three networks, respectively.  Besides the individual performance, we also present average performance treating four categories as a whole.

We see that: i) for LatentNet and AtlasNet, \textbf{CD} presents best CD, worst EMD and JSD; \textbf{EMD} presents best EMD, worst CD, and passable JSD; MPED presents best JSD, while passable CD and EMD; ii) for FoldingNet++, MPED presents dominated performance while only several EMDs are awarded by \textbf{EMD}; iii) considering JSD as the third party metric that was not involved in training, the proposed MPED is more capable. When using \textbf{CD} as the supervision, it is less objective to prove its superiority by providing CD is best.

Fig. \ref{fig:generation} illustrates several reconstruction results for three networks under various supervisions.  We see that: i) the results of \textbf{CD} present structure losing or blurred, especially finer parts, e.g., the leg of a chair; ii) \textbf{EMD} presents some ``noise'' points around sample surface, which means the coordinate of a single point is not accurate enough. iii) the proposed MPED can reconstruct each part of samples with high quality.

{$\bullet$ 3D Shape Completion}

 \begin{table*}
 \caption{Shape completion on ShapeNet. The proposed MPED grants most of the top-level performance on different networks.  }\label{TABLE:competion_table}
\begin{center}
\setlength{\tabcolsep}{0.4mm}{
\begin{tabular}{|c|c|c|c|c|c|c|c|c|c|c|c|c|c|c|c|c|}
\hline
\multirow{2}{*}{Model} & \multirow{2}{*}{Loss}&\multicolumn{5}{|c|}{CD($\times$10e-3)} &\multicolumn{5}{|c|}{EMD($\times$10e1)}  &\multicolumn{5}{|c|}{JSD($\times$10e-1)} \\ \cline{3-17}
& &Chair &Air &Table &Car &Ave &Chair &Air &Table &Car &Ave &Chair &Air &Table &Car &Ave \\
\hline\hline
\multirow{3}{*}{{LatentNet}}  & CD &\textcolor{red}{\textbf{7.56}}	&3.52	&5.04	&4.72	&5.21	&3.73	&3.51	&3.63	&2.65	&3.38	&5.41	&4.00	&4.19	&3.67	&4.32\\

& EMD &7.58	&4.03	&5.03	&7.58	&6.05	&\textcolor{red}{\textbf{3.71}}	&3.59	&3.60	&3.71	&3.65	&\textcolor{red}{\textbf{5.33}}	&4.07	&\textcolor{red}{\textbf{4.16}}	&5.33	&4.72\\ 

 & MPED &7.73	&\textcolor{red}{\textbf{3.37}}	&\textcolor{red}{\textbf{4.87}}	&\textcolor{red}{\textbf{4.52}}	&\textcolor{red}{\textbf{5.12}}	&3.75	&\textcolor{red}{\textbf{3.30}}	&\textcolor{red}{\textbf{3.59}}	&\textcolor{red}{\textbf{2.61}}	&\textcolor{red}{\textbf{3.31}}	&5.37	&\textcolor{red}{\textbf{3.93}}	&4.20  &\textcolor{red}{\textbf{3.60}}	&\textcolor{red}{\textbf{4.28}}
\\
 \hline \hline

 \multirow{3}{*}{{PF-Net}}   & CD &\textcolor{blue}{\textbf{3.33}}	&\textcolor{blue}{\textbf{1.77}}	&\textcolor{blue}{\textbf{3.73}}	&\textcolor{blue}{\textbf{3.94}}	&\textcolor{blue}{\textbf{3.19}}	&3.33	&1.89	&3.28	&3.72	&3.05	&4.93	&3.66	&3.91	&3.62	&4.03\\
 & EMD  &4.78	&4.83	&4.83	&5.58	&5.01	&\textcolor{blue}{\textbf{1.26}}	&1.48	&\textcolor{blue}{\textbf{1.48}}	&\textcolor{blue}{\textbf{1.39}}	&\textcolor{blue}{\textbf{1.40}}	&5.13	&4.08	&4.08	&3.49	&4.20\\ 

 & MPED &3.60	&1.91	&4.04	&4.25	&3.45	&1.48	&\textcolor{blue}{\textbf{1.13}}	&1.96	&1.44	&1.50	&\textcolor{blue}{\textbf{4.60}}	&\textcolor{blue}{\textbf{3.22}}	&\textcolor{blue}{\textbf{3.58}}	&\textcolor{blue}{\textbf{3.17}}	&\textcolor{blue}{\textbf{3.64}}\\
  \hline \hline 
  
 \multirow{3}{*}{{PF-GAN}} & CD &3.82	&2.13	&4.56	&\textcolor{cyan}{\textbf{4.07}}	&3.64	&2.12	&1.66	&2.49	&3.67	&2.48	&4.78	&3.72	&3.94	&3.63	&4.02\\
 
 & EMD &4.67	&2.81	&4.84	&5.60	&4.48	&\textcolor{cyan}{\textbf{1.27}}	&\textcolor{cyan}{\textbf{1.01}}	&\textcolor{cyan}{\textbf{1.53}}	&\textcolor{cyan}{\textbf{1.36}}	&\textcolor{cyan}{\textbf{1.29}}	&5.12	&3.82	&4.15	&3.49	&4.14
\\ 
 & MPED  &\textcolor{cyan}{\textbf{3.80}}	&\textcolor{cyan}{\textbf{1.94}}	&\textcolor{cyan}{\textbf{4.32}}	&4.30	&\textcolor{cyan}{\textbf{3.59}}	&1.53	&1.19	&1.81	&1.43	&1.49	&4.60	&\textcolor{cyan}{\textbf{3.25}}	&\textcolor{cyan}{\textbf{3.61}}	&\textcolor{cyan}{\textbf{3.20}}	&\textcolor{cyan}{\textbf{3.67}}\\

\hline
\end{tabular}}
\end{center}

\end{table*}

Same to reconstruction, we use the ShapeNet dataset proposed in \cite{achlioptas2018learning}. Four sub-categories from ShapeNet are used as the test sequences, e.g., ``Chair'', ``Airplane'', ``Table'' and ``Car''. The incomplete samples are generated via the method proposed by \cite{huang2020pf}. Specifically, for each sub-category, we set 75\% as training set and 25\% as testing set. We use LatentNet \cite{achlioptas2018learning}, PF-Net and PF-GAN \cite{huang2020pf} with default network parameter setting to test the empirical performances of various loss functions on the task of shape completion.  Besides the individual performance, we also present average performance treating four categories as a whole. All the results are shown in Table \ref{TABLE:competion_table}. We see that: i) for LatentNet and PF-GAN, MPED provide dominated performances in \textbf{CD} and JSD while  EMD in PF-GAN is preferred by \textbf{EMD}; ii) for PF-Net, MPED presents best JSD, while passable CD and EMD; \textbf{CD} and \textbf{EMD} respectively give best CD and EMD.

The qualitative examples of shape completion are illustrated in Fig. \ref{fig:completion}. For LatentNet, the output is the whole sample, while PF-GAN and PF-Net only generate the missing part. We use black points to represent the input samples, grey points represent the hollow part (e.g., airplane tail), and orange points to highlight generated points. We see that: i) the tail of \textbf{CD} is quite blurred, while \textbf{EMD} is too tight with several scattered points far away from the reconstructed surface; and ii) the proposed MPED generates the most realistic hollow part.
\begin{figure}[pt]
	\centering
    \includegraphics[width=1\linewidth]{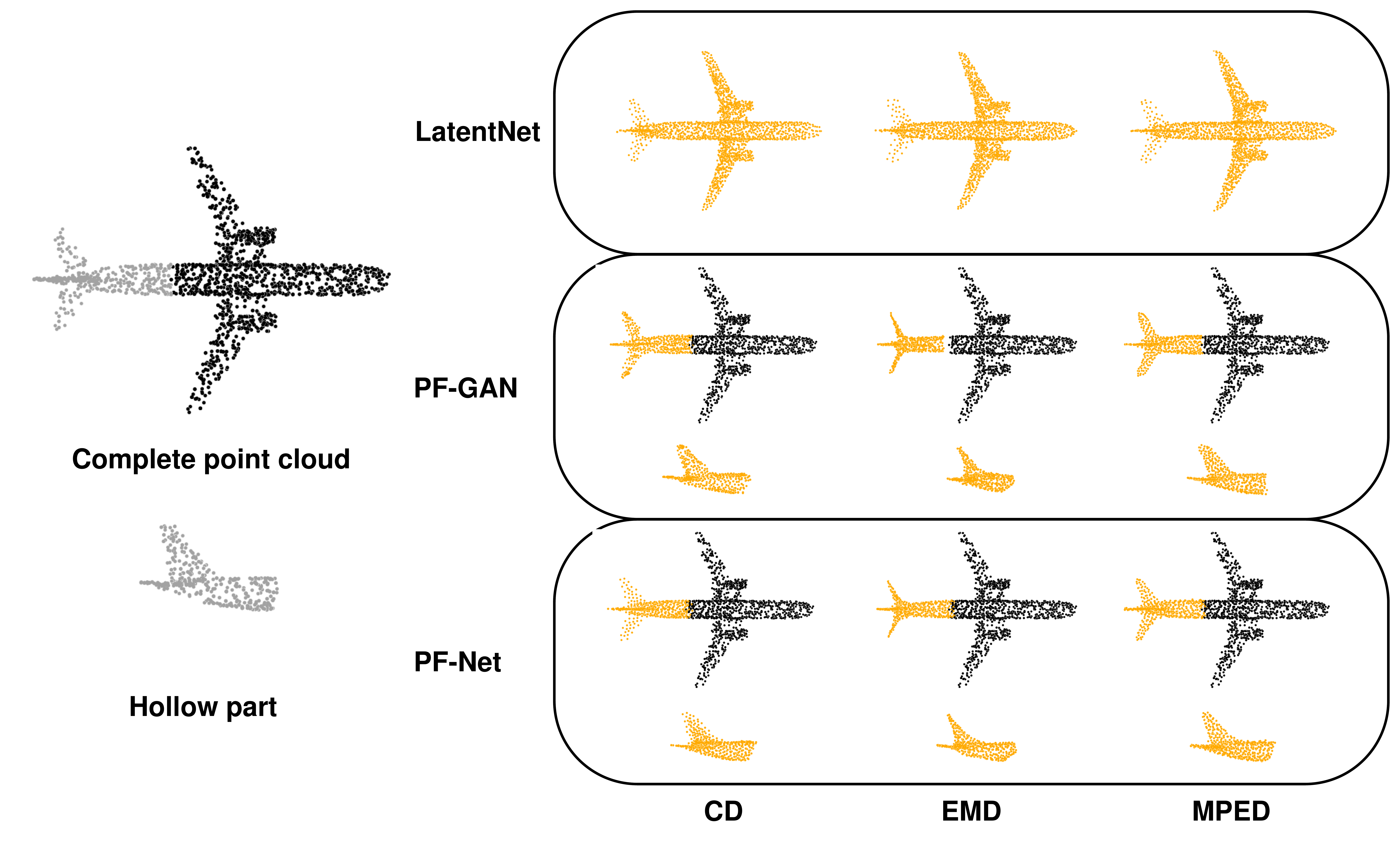}%
	\caption{Illustration of point cloud completion. The proposed MPED presents the best completion quality.}
	\label{fig:completion}
\end{figure}

\begin{figure*}[pt]
	\centering
    \includegraphics[width=0.9\linewidth]{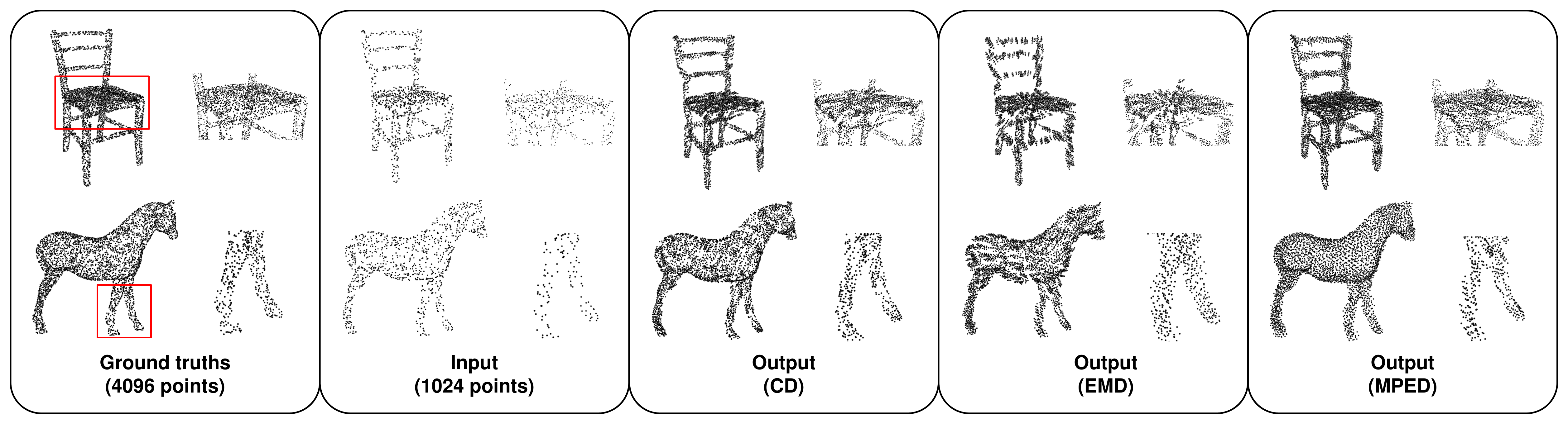}%
	\caption{Illustration of point cloud upsampling. The proposed MPED presents the best upsampling results.}
	\label{fig:upsampling}
\end{figure*}

{$\bullet$ 3D Point Cloud Upsampling}

\begin{table}
\caption{Upsampling results on PU-Net (CD ($\times$10e-4), EMD ($\times$10e1), and JSD ($\times$10e-2)). The proposed MPED presents most of the top-level performance.}\label{TABLE:sampling}
\begin{center}
\setlength{\tabcolsep}{0.4mm}{
\begin{tabular}{|c|c|c|c|c|c|}
\hline
\multirow{2}{*}{Model} &\multirow{2}{*}{Ratio}&\multirow{2}{*}{Loss} & \multicolumn{3}{|c|}{Upsampling} \\ \cline{4-6}
&  &   &CD &EMD &JSD   \\
\hline
\multirow{6}{*}{{PU-Net}} & \multirow{3}{*}{2}&CD  &6.50 &5.93 &10.12      \\
 &  &EMD &9.17 &4.92 &13.53    \\ 

 &  &MPED &\textcolor{blue}{\textbf{6.49}} &\textcolor{blue}{\textbf{4.59}} &\textcolor{blue}{\textbf{9.33}}    \\ \cline{2-6}
 & \multirow{3}{*}{4} &CD  &\textcolor{blue}{\textbf{4.96}} &9.62 &8.10      \\
 &  &EMD &7.47 &9.48 &11.20    \\ 
 &  &MPED &5.32 &\textcolor{blue}{\textbf{9.36}} &\textcolor{blue}{\textbf{7.43}}    \\ \cline{2-6}
 \hline
\end{tabular}}
\end{center}
\end{table}

We use PU-Net \cite{yu2018pu} to test the performance of three loss functions for the upsampling task. We vary the upsampling ratio as 2 and 4, respectively. And the we use the database provided by the \cite{yu2018pu}, all the experiment parameters are default parameters refer to the implementation proposed by \cite{PUGANpytorch}. The results are shown in Table \ref{TABLE:sampling}. Fig. \ref{fig:upsampling} illustrates some upsampling examples under $\mathbf{ratio=4}$. We zoom in some details for better observation.

We see that: i) for both $\mathbf{ratio=2}$ and $\mathbf{4}$, the MPED presents dominant performance; and ii) compared with MPED, \textbf{CD} presents more discontinuities, such as holes in horse's leg and chair's seat in Fig. \ref{fig:upsampling}; \textbf{EMD} generates some uneven texture, which damages the viewing experience to some extent.

\begin{table}
\begin{center}
\caption{Point cloud reconstruction on PED and MPED with different scales (CD ($\times$10e-4), EMD($\times$10e-1), JSD($\times$10e-2)).}\label{TABLE:neighbor_size_table}
\setlength{\tabcolsep}{0.5mm}{
\begin{tabular}{|c|c|c|c|c|}
\hline
\multirow{2}{*}{Model} & \multirow{2}{*}{$\Phi$} & \multicolumn{3}{|c|}{Reconstruction}   \\ \cline{3-5}
& &CD &EMD &JSD   \\
\hline\hline
\multirow{6}{*}{{LatentNet}}  & [1] &6.55 &9.64 &3.13       \\
  &[3]  &6.35 &7.20 &0.75    \\
  &[5] &6.56 &6.50 &0.54     \\
  &[10] &7.30 &5.75 &0.44      \\
  &[15] &7.68 &5.50 &0.46    \\ 
  &[10,5,1] &7.49 &5.90 &0.31 \\
 \hline
\end{tabular}}
\end{center}
\end{table}%

{\bf Ablation Study.}\label{sec:ablation_cv}

{$\bullet$ Size of Neighborhood}

We use point cloud reconstruction as a typical task.
For comparison, we test PED at different single scales ($\Psi=[1], [3], [5], [10], [15]$) and one MPED ($\Psi=[10,5,1]$). ``Chair'' is selected as test sequence of point cloud reconstruction, and LatentNet is used as the network. The results are shown in Table  \ref{TABLE:neighbor_size_table}. We see that: i) with the increase of neighbor size, the value of CD present an increasing tendency, while EMD and JSD become smaller; ii) when $K$ increased from 1 to 3, the EMD and JSD present an obvious gap, which means the point-collapse is significantly alleviated by injecting contextual information; iii) compared to PED at different single scales, MPED presents best JSD, which states that introducing multiscale characteristics contributes to the point distribution of generated samples.

{$\bullet$ Convergence rate}

To illustrate the effectiveness of the proposed MPED, Fig. \ref{fig:rate} (a) shows the tradeoff between performance and computational cost. We use LatentNet as an example. The y-axis shows the performance of JSD under different loss functions, and the x-axis shows the training time. We see that MPED achieves lower JSD while reducing the computational time by $3$ times compared to~\textbf{EMD}. Fig. \ref{fig:rate} (b)-(d) show the reconstruction loss of CD, EMD and JSD under different training epochs.

\begin{figure}[t]
  \centering
  \subfigure[]{\includegraphics[width=0.49\linewidth]{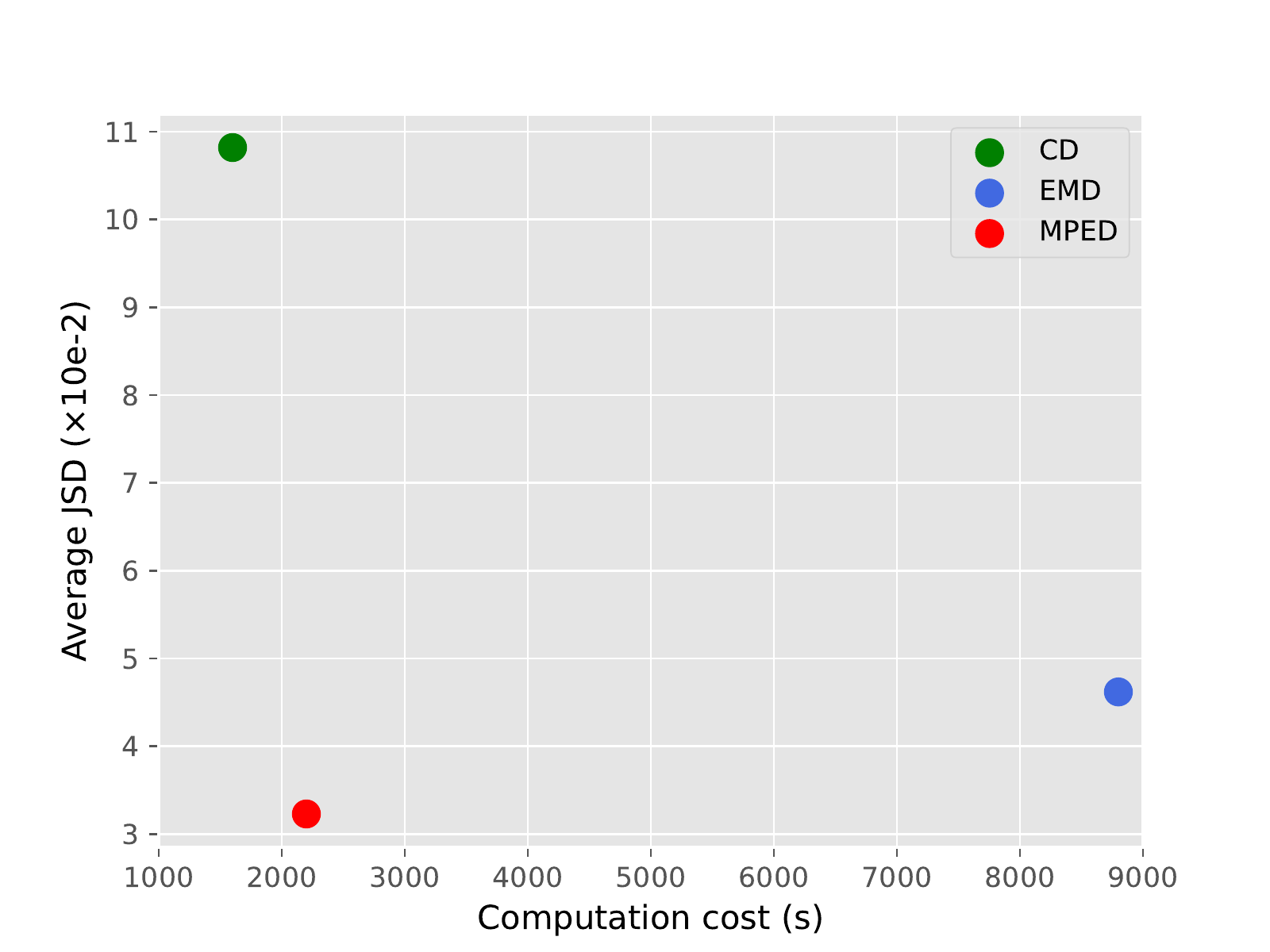}}
  \subfigure[]{\includegraphics[width=0.49\linewidth]{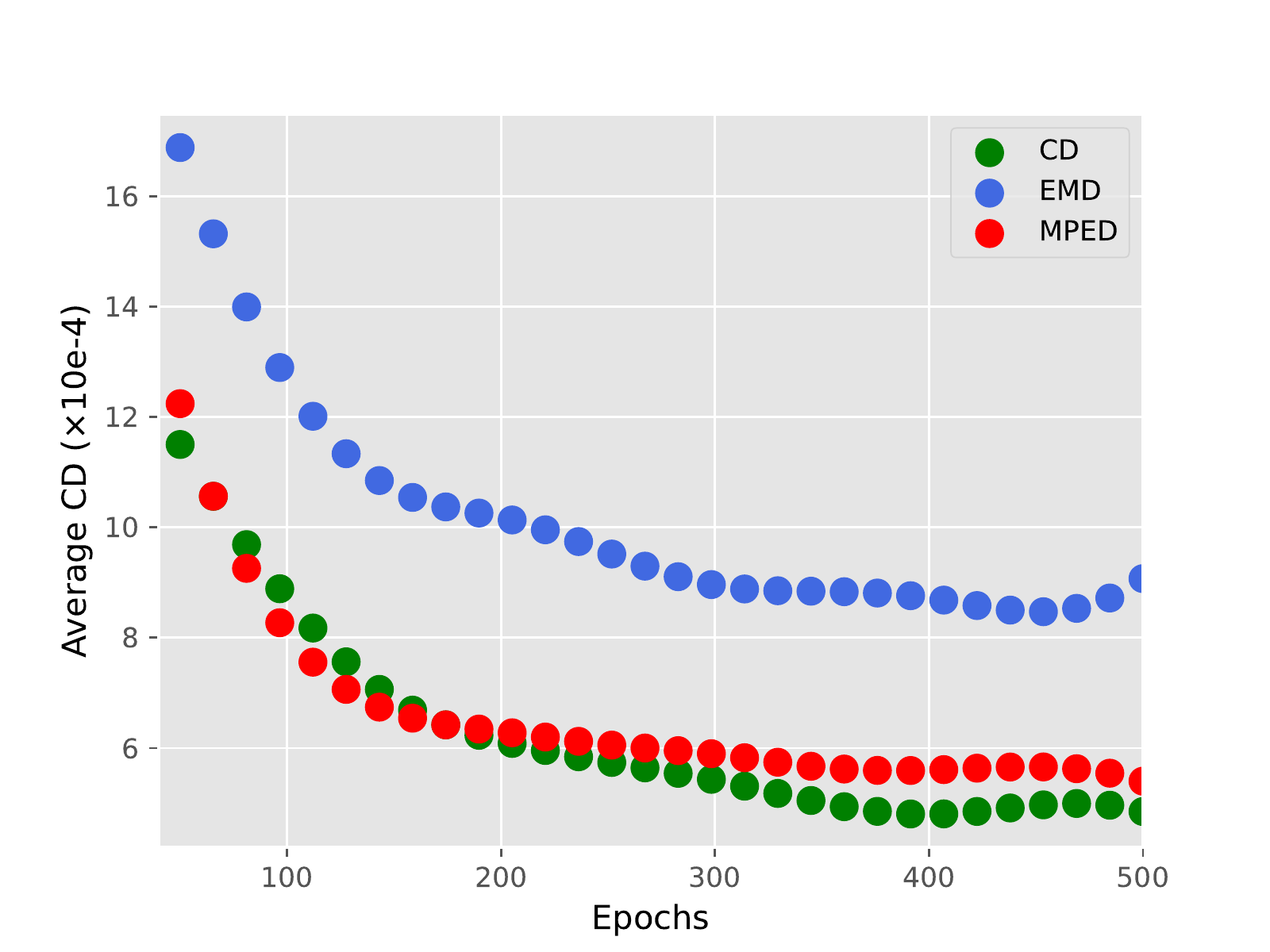}}\\
  \subfigure[]{\includegraphics[width=0.49\linewidth]{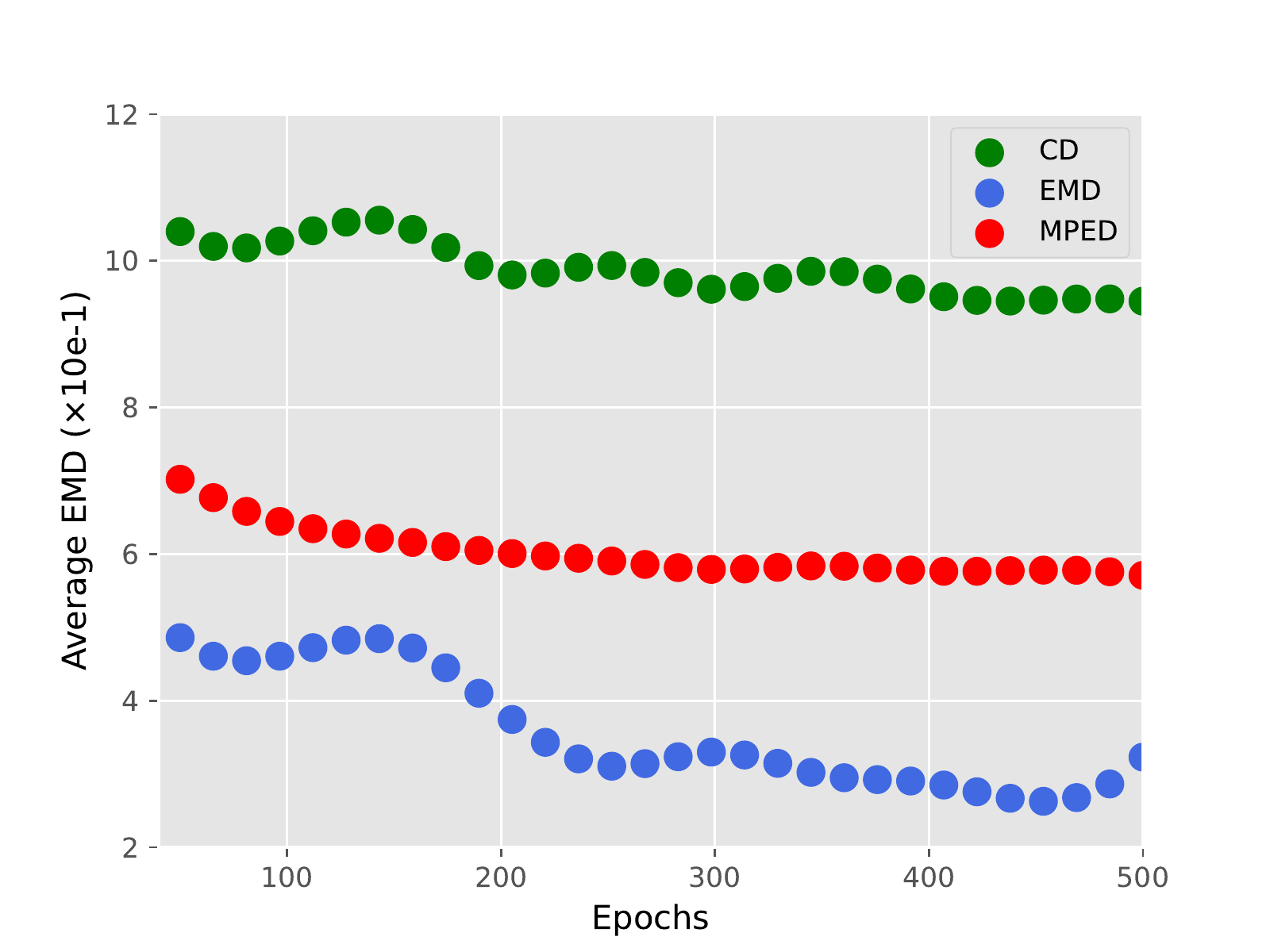}}
  \subfigure[]{\includegraphics[width=0.49\linewidth]{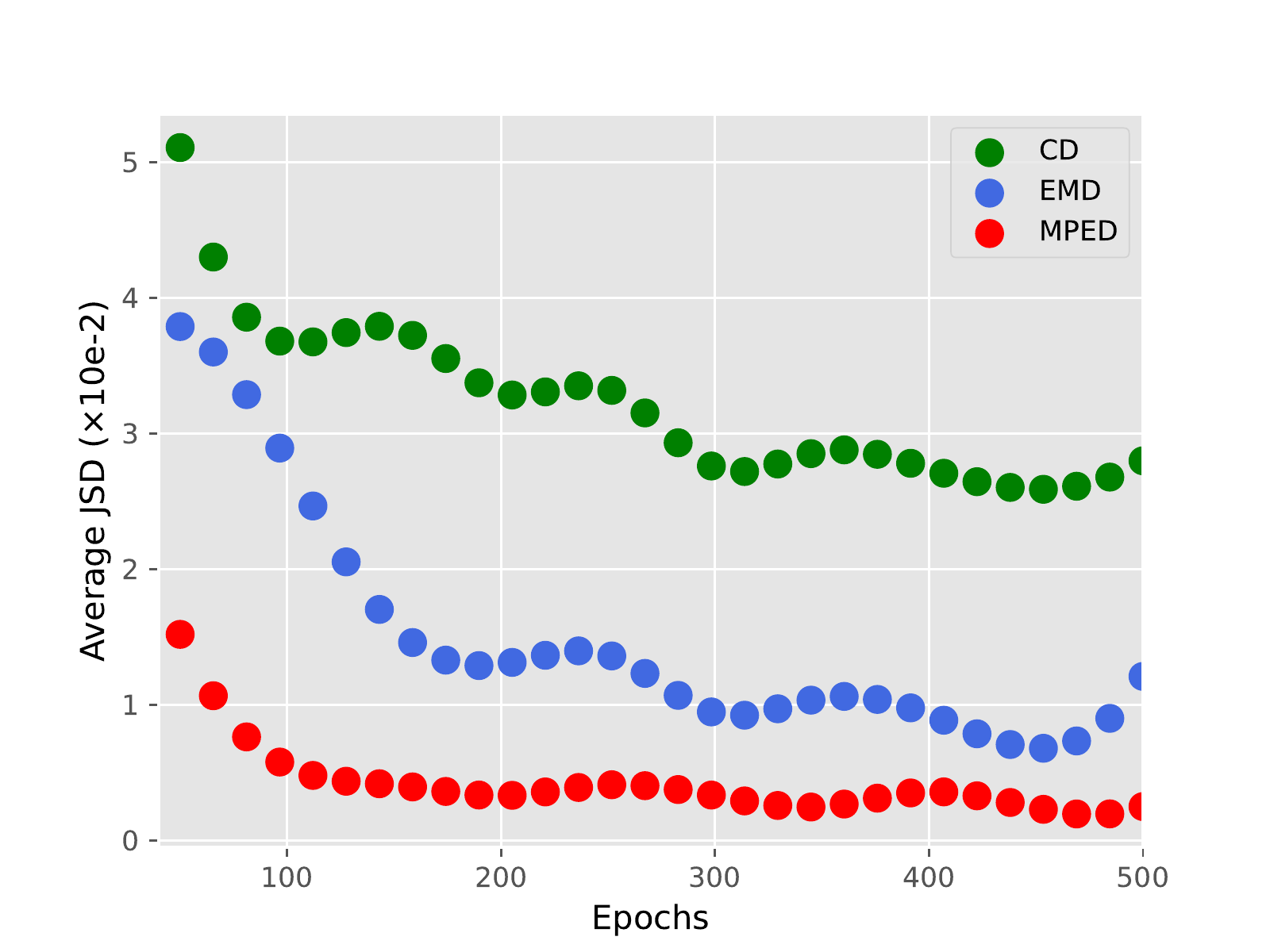}}\\
  \centering
  \caption{(a) Performance vs. time. The proposed MPED presents best performance with low time consuming; (b)-(d): reconstruction loss under different training epochs (b) CD; (c) EMD; (d) JSD. }
  \label{fig:rate}
  \end{figure}

 For CD, i.e., Fig. \ref{fig:rate} (b), MPED presents obvious faster convergence rate in the early stage, after which it realizes close performance with \textbf{CD}; for EMD, i.e., Fig. \ref{fig:rate} (c), MPED curve shows steady downward trend and costs around 300 epochs to reach stable and best results; for JSD, i.e., Fig. \ref{fig:rate} (d), 200 epochs is enough for MPED to converge to stable results, while CD and EMD need more than 400 epochs, respectively.

\section{Conclusions}
\label{sec:conclusion}
In this paper, we propose a universal point cloud distortion quantification named multiscale potential energy discrepancy (MPED). The proposed MPED is differentiable, has low computational complexity, and can discriminate distortions.  MPED presents robust performances for both human perception and machine perception tasks. Specifically, for human perception tasks, MPED reveals the best performance on two fairly large databases, i.e., SJTU-PCQA and LS-PCQA; for machine perception tasks, MPED shows better performance than Chamfer distance and Earth mover's distance on three typical unsupervised tasks, i.e., point cloud reconstruction, shape completion and upsampling. We further demonstrate the robustness of MPED in ablation study, the results show that the proposed MPED exhibits stable and reliable performance in terms of various color spaces, and presents high optimization efficiency as loss function.

\section*{Acknowledgments}
The authors would like to thank anonymous reviewers for their constructive comments to improve this manuscript.
This work was supported by the National Key R\&D Project of China (2018YFE0206700) and National Natural Science Foundation of China (61971282, U20A20185).

\ifCLASSOPTIONcaptionsoff
  \newpage
\fi

\bibliography{sample-base}

\bibliographystyle{IEEEtran}

\begin{IEEEbiography}
[{\includegraphics[width=1in,height=1.25in,clip,keepaspectratio]{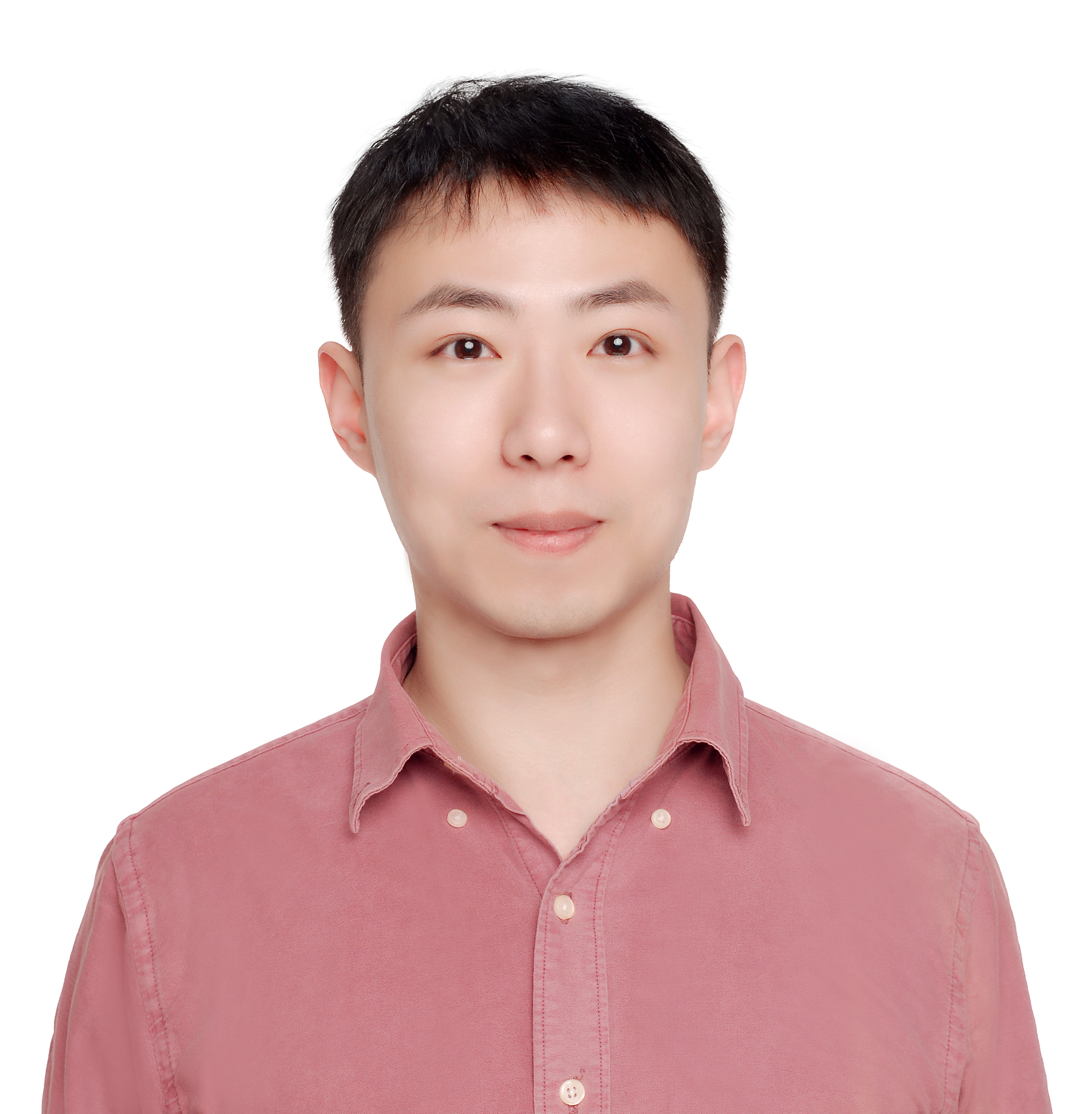}}]{Qi Yang}
received the B.S. degree in communication engineering from Xidian University, Xi’an, China, in 2017, and Ph.D degree in information and communication engineering at Shanghai Jiao Tong University, Shanghai, China, 2022. Now, he is a researcher in Tencent MediaLab. His research interests include image processing, 3D point cloud / mesh quality assessment and reconstruction. .
\end{IEEEbiography}

\begin{IEEEbiography}
[{\includegraphics[width=1in,height=1.25in,clip,keepaspectratio]{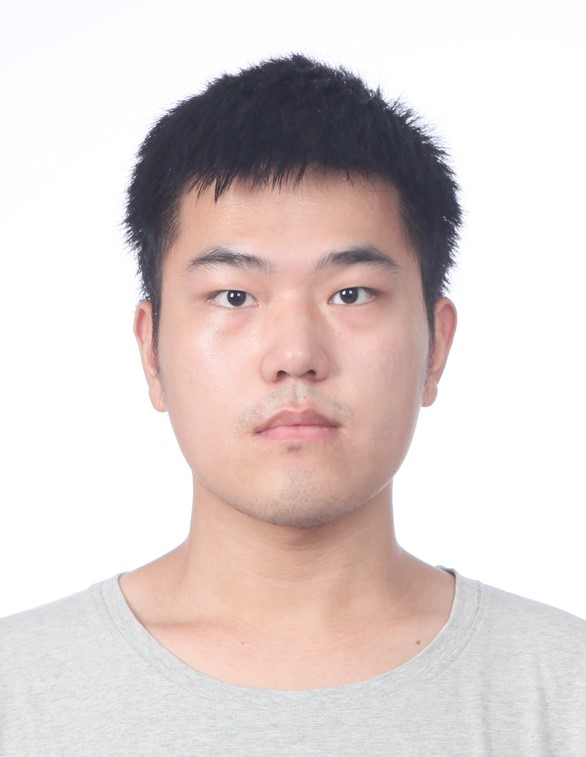}}]{Yujie Zhang}
received the B.S. degree in electronic and information engineering from Beihang University, Beijing, China, in 2020. He is currently working toward the Ph.D degree in information and communication engineering at Shanghai Jiao Tong University, Shanghai, China. His research interests include point cloud quality assessment and quality enhancement.
\end{IEEEbiography}

\begin{IEEEbiography}
[{\includegraphics[width=1in,height=1.25in,clip,keepaspectratio]{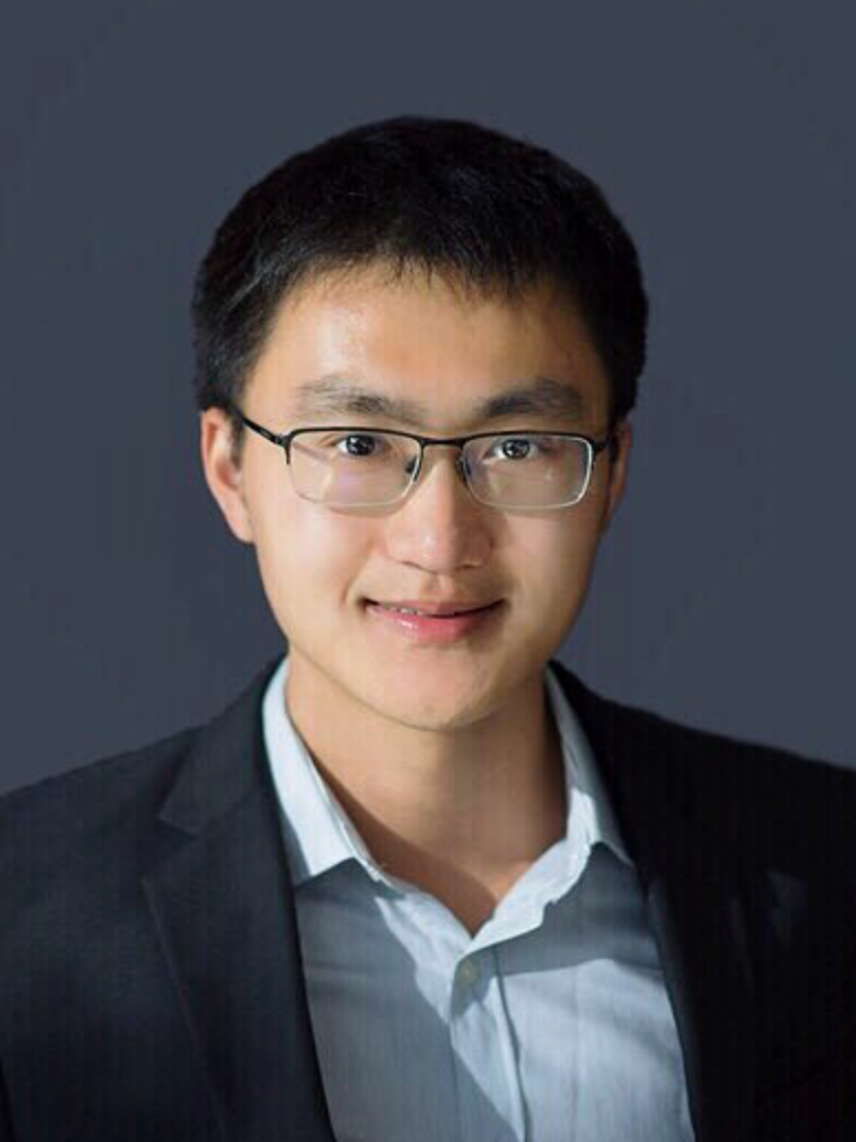}}]{Siheng Chen}
is a tenure-track associate professor of Shanghai Jiao Tong University. Before joining Shanghai Jiao Tong University, he was a research scientist at Mitsubishi Electric Research Laboratories (MERL), and an autonomy engineer at Uber Advanced Technologies Group (ATG), working on the perception and prediction systems of self-driving cars. Before joining industry, Dr. Chen was a postdoctoral research associate at Carnegie Mellon University. Dr. Chen received his doctorate in Electrical and Computer Engineering from Carnegie Mellon University, where he also received two master degrees in Electrical and Computer Engineering (College of Engineering) and Machine Learning (School of Computer Science), respectively. Dr. Chen's work on sampling theory of graph data received the 2018 IEEE Signal Processing Society Young Author Best Paper Award. His co-authored paper on structural health monitoring received ASME SHM/NDE 2020 Best Journal Paper Runner-Up Award and another paper on 3D point cloud processing received the Best Student Paper Award at 2018 IEEE Global Conference on Signal and Information Processing. Dr. Chen contributed to the project of scene-aware interaction, winning MERL President's Award. His research interests include graph machine learning, collective intelligence and autonomous driving.
\end{IEEEbiography}

\begin{IEEEbiography}
[{\includegraphics[width=1in,height=1.25in,clip,keepaspectratio]{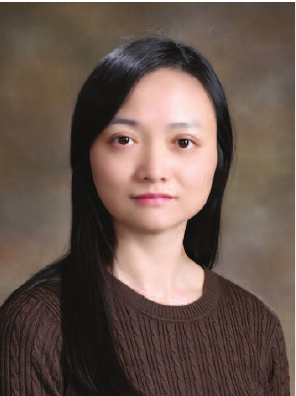}}]{Yiling Xu}
received the B.S., M.S., and Ph.D. degrees from the University of Electronic Science
and Technology of China in 1999, 2001, and 2004 respectively. From 2004 to 2013, she was
a Senior Engineer with the Multimedia Communication Research Institute, Samsung Electronics
Inc., South Korea. She joined Shanghai Jiao Tong University, where she is currently a Professor
in the areas of multimedia communication, 3-D point cloud compression and assessment, system
design, and network optimization. She is the Associate Editor of the IEEE TRANSACTIONS ON
BROADCASTING. She is also an active member in standard organizations, including MPEG, 3GPP,
and AVS.
\end{IEEEbiography}
\vspace{-10 mm}

\begin{IEEEbiography}
[{\includegraphics[width=1in,height=1.25in,clip,keepaspectratio]{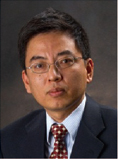}}]{Jun Sun}
is currently a professor and Ph.D. advisor of Shanghai Jiao Tong University. He received his B.S. in 1989 from University of Electronic Sciences and technology of China, Chengdu, China, and a Ph.D. degree in 1995 from Shanghai Jiao Tong University, all in electrical engineering. In 1996, he was elected as the member of HDTV Technical Executive Experts Group (TEEG) of China. Since then, he has been acting as one of the main technical experts for the Chinese government in the field of digital television and multimedia communications. In the past five years, he has been responsible for several national projects in DTV and IPTV fields. He has published over 50 technical papers in the area of digital television and multimedia communications and received 2nd Prize of National Science and  Technology Development Award in 2003, 2008. His research interests include digital television, image communication, and video encoding.
\end{IEEEbiography}

\vspace{-100mm}
\begin{IEEEbiography}
[{\includegraphics[width=1in,height=1.25in,clip,keepaspectratio]{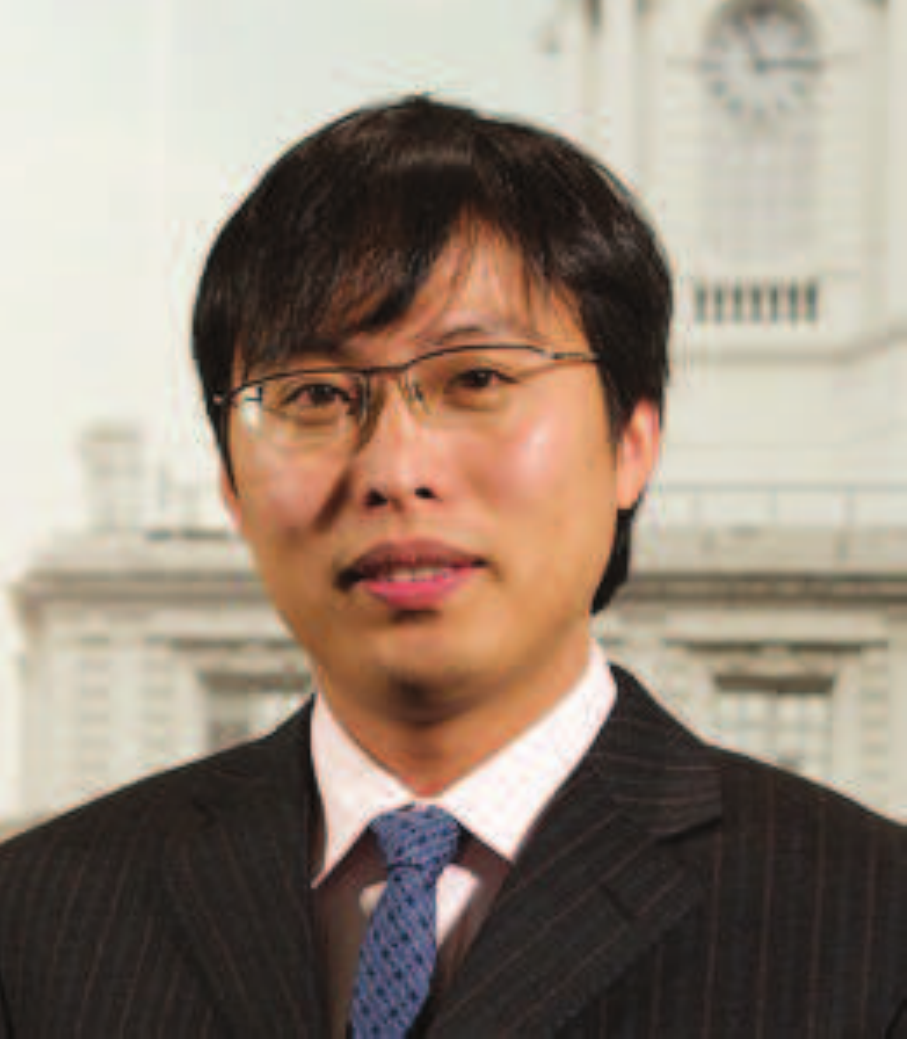}}]{Zhan Ma}
received the B.S. and M.S. degrees from Huazhong University of Science and Technology (HUST), Wuhan, China, in 2004 and 2006 respectively, and the Ph.D. degree from the New York University, New York, in 2011. He is now on the faculty of Electronic Science and Engineering School, Nanjing University, Jiangsu, 210093, China. From 2011 to 2014, he has been with Samsung Research America, Dallas TX, and  Futurewei Technologies, Inc., Santa Clara, CA, respectively. His current research focuses on the next-generation video coding, energy-efficient communication, gigapixel streaming and deep learning. He is a co-recipient of 2018 ACM SIGCOMM Student Research Competition Finalist, 2018 PCM Best Paper Finalist, and 2019 IEEE Broadcast Technology Society Best Paper Award.
\end{IEEEbiography}

\newpage
\appendices
\section{Additional Toy Example of PED in Human Perception Tasks}

In this section we provide additional toy samples to validate the geometrical and attributive distortion sensitivity of PED for human perception tasks. These samples focus more on the situations that PED has more than one centers. 

\begin{figure}[htbp]
	\centering
    \includegraphics[width=1\linewidth]{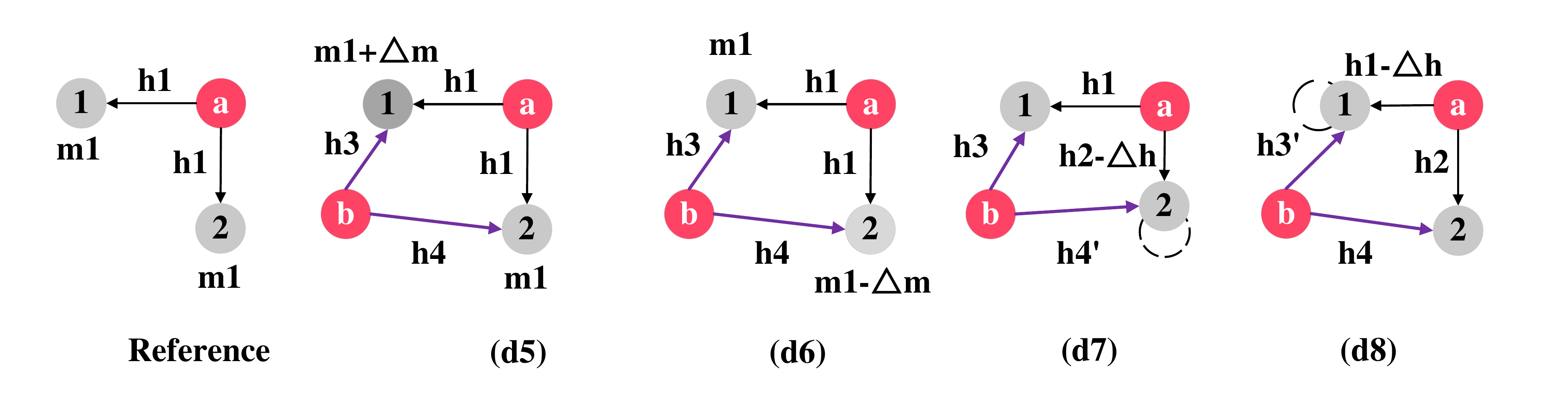}%
    \caption{Additional toy examples of geometrical and attributive distortion sensitivity.}
	\label{fig:case}
\end{figure}

{\bf Additional Case 1: Sensitivity to attributive distortions at various locations.} As illustrated in Fig. \ref{fig:case} d5 and d6,

$\diamond$ $\mathrm{CD_{d5}}=0 =\mathrm{CD_{d6}}$;

$\diamond$ $\mathrm{EMD_{d5}}=0 =\mathrm{EMD_{d6}}$;

$\diamond$ $\mathrm{PSNR_{YUV}^{d_5}}=\Delta m=\mathrm{PSNR_{YUV}^{d_6}}$;

$\diamond$ For neighborhood $a$,  $\mathrm{PED}^{a, d_5}=\mathrm{PED}^{a, d_6}$; For neighborhood $b$, because of $h_3 \neq h_4$, $\mathrm{PED}^{b, d_5}\neq\mathrm{PED}^{b, d_6}$, therefore,  $\mathrm{PED}^{ d_5}=\mathrm{PED}^{a, d_5}+ \mathrm{PED}^{b, d_5}\neq \mathrm{PED}^{a, d_6}+ \mathrm{PED}^{b, d_6}=\mathrm{PED}^{d_6}$.

{\bf Additional Case 2: Sensitivity to geometrical distortions at various locations.} As illustrated in Fig. \ref{fig:case} d7 and d8,

$\diamond$ $d_\mathrm{CD_{d5}}=\Delta h =d_\mathrm{CD_{d6}}$;

$\diamond$ $d_\mathrm{EMD_{d5}}=\Delta h =d_\mathrm{EMD_{d6}}$;

$\diamond$ $\mathrm{PSNR_{YUV}^{d_5}}=0=\mathrm{PSNR_{YUV}^{d_6}}$;

$\diamond$ For neighborhood $a$,  $\mathrm{PED}^{a, d_7}=\mathrm{PED}^{a, d_8}$; For neighborhood $b$, because of $h_3 \neq h_4$, $\mathrm{PED}^{b, d_7}\neq\mathrm{PED}^{b, d_8}$, therefore,  $\mathrm{PED}^{ d_7}=\mathrm{PED}^{a, d_7}+ \mathrm{PED}^{b, d_7}\neq \mathrm{PED}^{a, d_8}+ \mathrm{PED}^{b, d_8}=\mathrm{PED}^{d_8}$.

\begin{figure}[pt]
	\centering
     \subfigure[PLCC]{\includegraphics[width=0.8\linewidth]{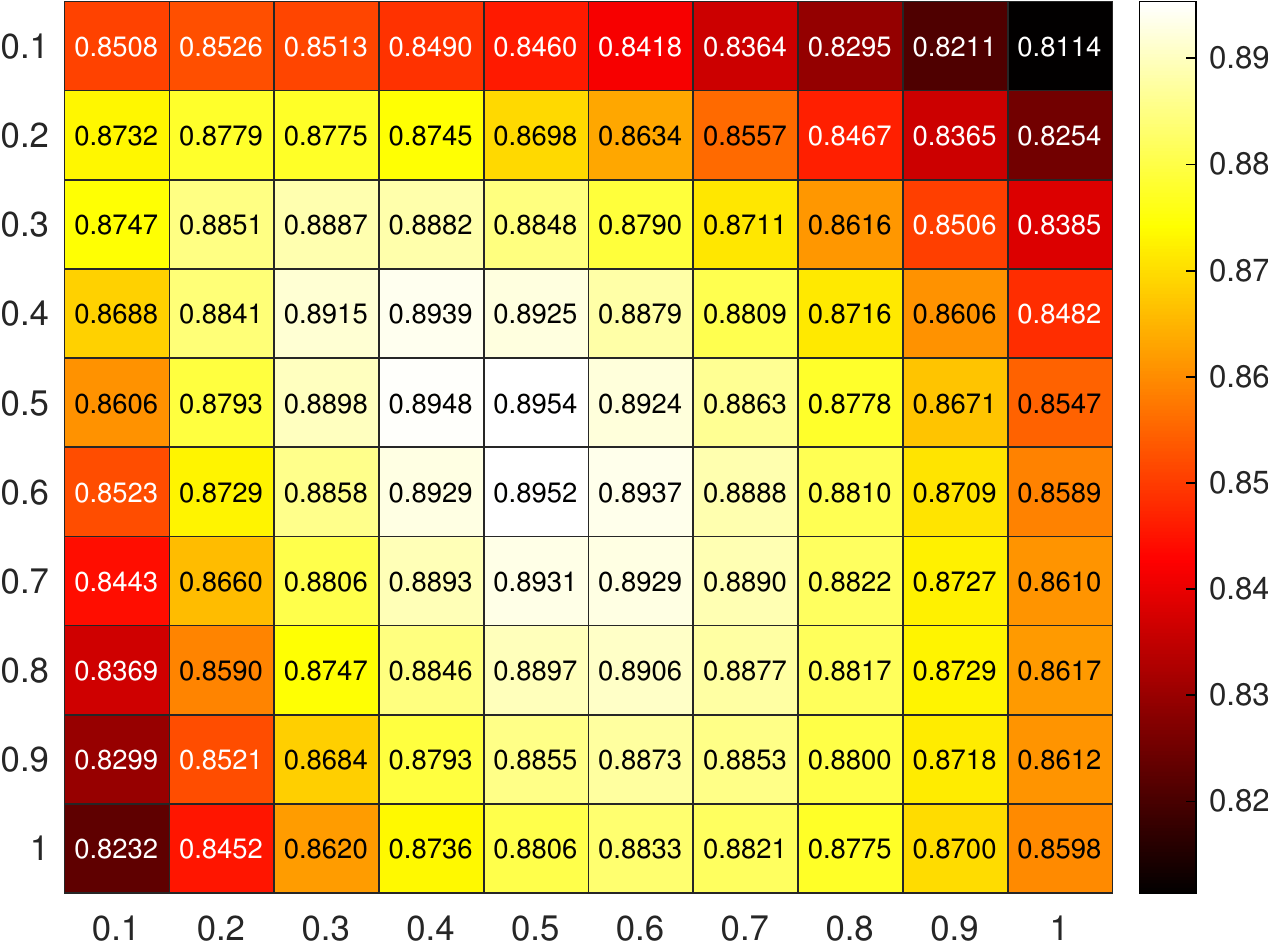}}
    \subfigure[SROCC]{\includegraphics[width=0.8\linewidth]{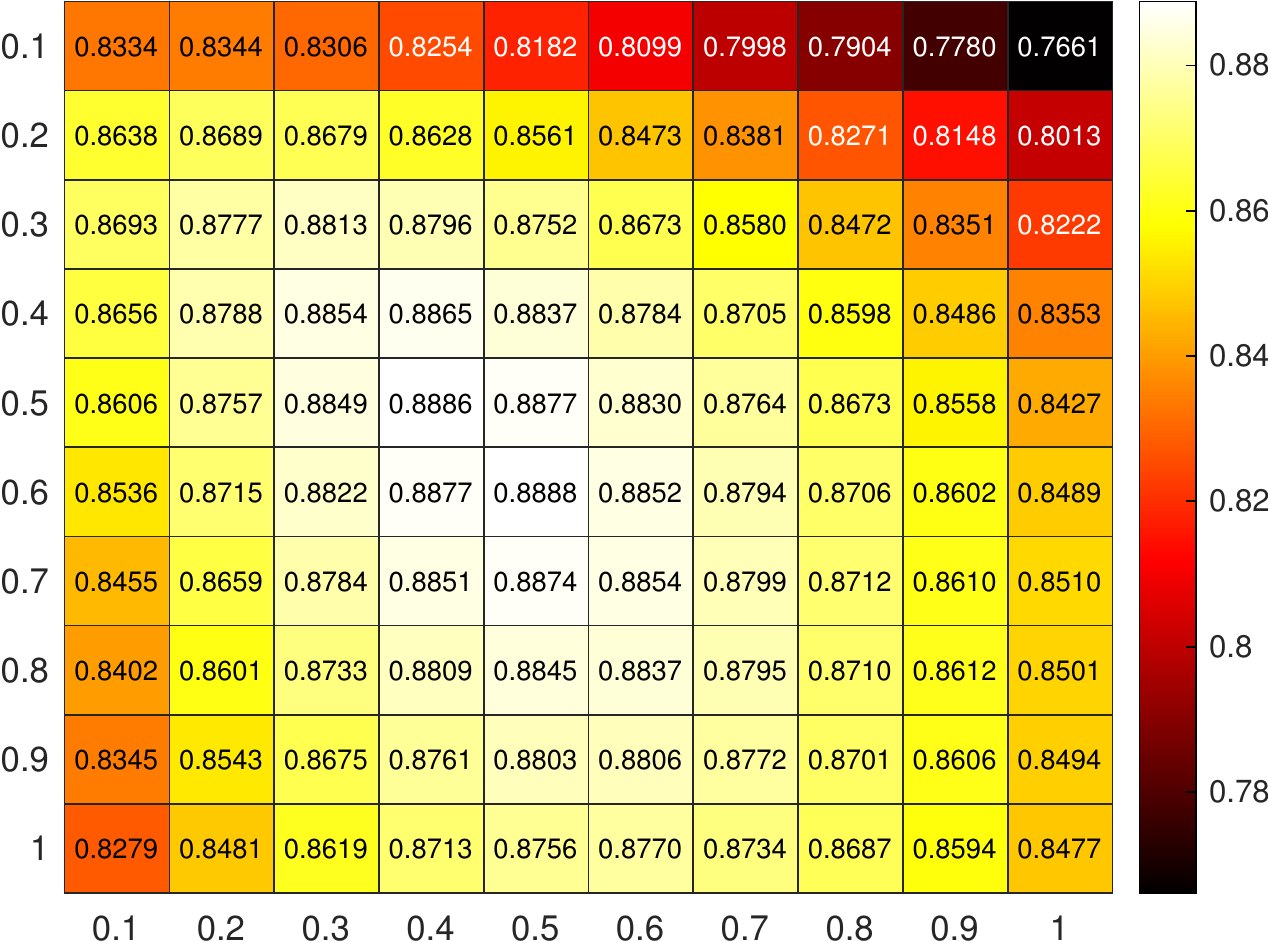}}
    \caption{Model performance with different power on SJTU-PCQA. The x and y axis are respectively $\beta$ and $\alpha$.}
	\label{fig:power comparison}
\end{figure}

\begin{table*}[t]
  \centering
  \caption{Performance comparison for FR-PCQA metrics on each individual distortion type in term of SROCC}
\begin{tabular}{|p{8em}|c|c|c|c|c|c|c|c|c|c|}
    \hline
     LS-PCQA  & {M-p2po} & {M-p2pl} & {H-p2po} & {H-p2pl} & {PSNR-YUV} & {PCQM} & {GraphSIM} & {EPES} & {MPED(p=1)} & {MPED(p=2)} \\
    \hline

    ColorNoise & {-} & {-} & {-} & {-} & 0.8340 & 0.8153 & 0.6868 & 0.7087 & \red{0.8389} & \red{0.8585} \\
    \hline
    Color QuantizationDither & {-} & {-} & {-} & {-} & \red{0.8260} & 0.7694 & 0.5168 & 0.5965 & 0.7592 & \red{0.8191} \\
    \hline
    ContrastDistortion & {-} & {-} & {-} & {-} & 0.6894 & \red{0.7440} & 0.6785 & 0.3106 & 0.6823 & \red{0.6903} \\
    \hline
    Correlated GaussianNoise & {-} & {-} & {-} & {-} & \red{0.9392} & 0.8507 & 0.5891 & 0.5373 & 0.8747 & \red{0.8747} \\
    \hline
    DownSample & \red{0.8813} & 0.6265 & \red{0.8415} & 0.8120 & 0.0149 & 0.5249 & 0.8426 & 0.7781 & 0.4410 & 0.5533 \\
    \hline
    GammaNoise & {-} & {-} & {-} & {-} & 0.7496 & 0.7077 & 0.5884 & 0.6783 & \red{0.8816} & \red{0.8613} \\
    \hline
    GaussianNoise & {-} & {-} & {-} & {-} & \red{0.645} & \red{0.7698} & 0.4826 & 0.3907 & 0.5830 & 0.5866 \\
    \hline
    GaussianShifting & 0.7414 & 0.7190 & \red{0.8298} & \red{0.8342} & 0.7552 & 0.8162 & 0.7428 & 0.5414 & 0.7862 & 0.7432 \\
    \hline
    HighFrequency Noise & {-} & {-} & {-} & {-} & 0.8363 & \red{0.9151} & 0.7629 & 0.6665 & \red{0.8902} & 0.8879 \\
    \hline
    LocalLoss & 0.5362 & 0.4979 & 0.0370 & 0.3030 & 0.6894 & 0.7708 & \red{0.8710} & 0.4964 & \red{0.6932} & 0.6890 \\
    \hline
    LocalOffset & \red{0.9373} & 0.9347 & 0.7705 & 0.7707 & 0.6671 & 0.8516 & 0.9062 & 0.8799 & 0.9320 & \red{0.9331} \\
    \hline
    LocalRotation & \red{0.8197} & 0.7126 & \red{0.8311} & 0.7345 & 0.3271 & 0.6572 & 0.7239 & 0.7466 & 0.6894 & 0.6385 \\
    \hline
    LumaNoise & {-} & {-} & {-} & {-} & 0.7727 & 0.7484 & \red{0.8178} & 0.6906 & \red{0.7943} & 0.7860 \\
    \hline
    MeanShift & {-} & {-} & {-} & {-} & 0.4223 & 0.6147 & 0.7062 & \red{0.7256} & \red{0.7634} & 0.6450 \\
    \hline
    Multiplicativ eGaussianNoise & {-} & {-} & {-} & {-} & 0.7512 & 0.7543 & 0.6484 & 0.7074 & \red{0.7795} & \red{0.7939} \\
    \hline
    PoissonNoise & {-} & {-} & {-} & {-} & \red{0.6824} & \red{0.6626} & 0.4218 & 0.4122 & 0.4751 & 0.4693 \\
    \hline
    QuantizationNoise & {-} & {-} & {-} & {-} & 0.7805 & 0.8480 & 0.6180 & 0.6048 & 0.\red{8534} & \red{0.8511} \\
    \hline
    RayleighNoise & {-} & {-} & {-} & {-} & \red{0.8938} & \red{0.8377} & 0.7069 & 0.6141 & 0.8015 & 0.8146 \\
    \hline
    SaltpepperNoise & {-} & {-} & {-} & {-} & 0.3950 & 0.6375 & 0.5601 & 0.6822 & \red{0.7187} & \red{0.7018} \\
    \hline
    Saturation Distortion & {-} & {-} & {-} & {-} & 0.7393 & \red{0.8506} & 0.7030 & 0.6322 & 0.7184 & \red{0.7150} \\
    \hline
    UniformNoise & {-} & {-} & {-} & {-} & \red{0.8982} & 0.6857 & 0.6455 & 0.5235 & \red{0.7754} & 0.7747 \\
    \hline
    UniformShifting & 0.8517 & 0.8517 & 0\red{.8575} & 0.8497 & 0.7969 & 0.6389 & \red{0.8697} & 0.8421 & 0.7480 & 0.7894 \\
    \hline
    V-PCC(lossless-geom-lossy-attrs) & 0.2347 & 0.1625 & 0.1399 & 0.0203 & 0.1084 & 0.1939 & 0.3564 & 0.4650 & \red{0.4292} & \red{0.6206} \\
    \hline
    V-PCC(lossy-geom-lossless-attrs) & 0.4141 & 0.4150 & 0.3231 & 0.4737 & 0.2844 & 0.5202 & 0.4858 & \red{0.7559} & 0.6457 & \red{0.7499} \\
    \hline
    V-PCC(lossy-geom-lossy-attrs) & 0.7534 & 0.7530 & 0.6409 & 0.6328 & 0.3721 & 0.7955 & 0.8082 & 0.8458 & \red{0.9008} & \red{0.9010} \\
    \hline
    AVS(limitlossyG-lossyA) & 0.9347 & \red{0.9458} & 0.9396 & 0.9294 & 0.8824 & - & 0.7439 & 0.9336 & 0.9169 & \red{0.9220} \\
    \hline
    AVS(losslessG-limitlossyA) & {-} & {-} & {-} & {-} & 0.8378 & - & \red{0.8772} & 0.7448 & 0.8636 & \red{0.8663} \\
    \hline
    AVS(losslessG-lossyA) & {-} & {-} & {-} & {-} & \red{0.9164} & - & 0.7958 & 0.7482 & \red{0.8928} & 0.8923 \\
    \hline
    G-PCC(lossless-geom-lossy-attrs & {-} & {-} & {-} & {-} & 0.5679 & 0.8621 & 0.6645 & 0.7319 & \red{0.8779} & \red{0.8839} \\
    \hline
    G-PCC(lossless-geom-nearlossless-attrs) & {-} & {-} & {-} & {-} & 0.8783 & \red{0.9377} & 0.8994 & 0.7908 & \red{0.9085} & 0.9059 \\
    \hline
    G-PCC(lossy-geom-lossy-attrs) & 0.9554 & 0.9255 & \red{0.9605} & \red{0.9469} & 0.7308 & 0.8450 & 0.8674 & 0.9402 & 0.9413 & 0.9436 \\
    \hline
    Octree & 0.7790 & 0.7882 & 0.8199 & 0.7521 & 0.5238 & 0.6762 & 0.7575 & 0.7557 & 0.8356 & \red{0.8603} \\
    \hline
    Reconstruction & 0.8470 & \red{0.8354} & \red{0.8125} & 0.8114 & 0.2131 & 0.7206 & 0.6475 & 0.5146 & 0.3075 & 0.5522 \\
    \hline

   \multicolumn{1}{r}{} & \multicolumn{1}{r}{} & \multicolumn{1}{r}{} & \multicolumn{1}{r}{} & \multicolumn{1}{r}{} & \multicolumn{1}{r}{} & \multicolumn{1}{r}{} & \multicolumn{1}{r}{} & \multicolumn{1}{r}{} & \multicolumn{1}{r}{} & \multicolumn{1}{r}{} \\
    
    \hline
    WPC   & {M-p2po} & {M-p2pl} & {H-p2po} & {H-p2pl} & {PSNR-yuv} & {PCQM} & {GraphSIM} & {EPES} & {MPED(p=1)} & {MPED(p=2)} \\
    \hline
    DownSample & \red{0.9005} & 0.8499 & 0.9049 & 0.8614 & 0.7069 & 0.8745 & 0.8975 & 0.7847 & 0.8970 & \red{0.9018} \\
    \hline
    GaussianNoise Contamination & 0.7286 & 0.7373 & 0.6888 & 0.6925 & 0.7767 & \red{0.8860} & 0.8402 & 0.7675 & \red{0.8797} & 0.8539 \\
    \hline
    G-PCC(Octree)  & {-} & {-} & {-} & {-} & 0.8259 & \red{0.8944} & 0.8553 & 0.7752 & \red{0.8687} & 0.8390 \\
    \hline
    G-PCC(Trisoup) & 0.4648 & 0.4626 & 0.2932 & 0.3550 & 0.6462 & \red{0.8212} &\red{0.8158} & 0.4248 & 0.5511 & 0.4933 \\
    \hline
    V-PCC  & \red{0.6977} & \red{0.7056} & 0.4450 & 0.5589 & 0.3445 & 0.6431 & 0.6124 & 0.3355 & 0.4755 & 0.4520 \\
    \hline
    
    \end{tabular}
  \label{tab:test for distortionl}%
\end{table*}%

\section{Additional Experiment for Human Perception Task} \label{FirstAppendix}
In this section we provides additional experiment results of MPED on human perception tasks, including evaluation on individual distortion type, size of neighborhood and power coefficients.
\subsection{Evaluation on Individual Distortion Type}
We compare the performance of FR-PCQA metrics towards different point cloud distortions on the LS-PCQA and WPC databases. SROCC scores are shown as the only evaluation measure and the results are shown in Table \ref{tab:test for distortionl}. For each distortion type in each database, we use \textbf{boldface} to highlight the algorithm with the top two SROCC among all competing metrics.

{$\bullet$ LS-PCQA database.}

There exist 34 distortion types in LS-PCQA and each distortion has 30 samples with MOS. In general, we can divide them into three categories: 17 types of color distortions which do not distort any geometry information, including as a wide range of color noises; 10 types of compression distortions caused by point cloud compression algorithms, including V-PCC, G-PCC (Octree-Predlift), Octree and  AVS compression algorithms; and 7 types of others distortions, including geometry shifting, downsampling and local transformation. From Table \ref{tab:test for distortionl}, we can see that: i) MPED ($p=1$ or $p=2$) is among top 2 models for 24 types of distortion, followed by 10 types for PCQM and 7 types for PSNR-YUV. It demonstrates that our method performs well on individual distortion type. ii) MPED is more capable of these color distortions and compressed distortions. Specifically, MPED becomes the top 2 models for 12 color distortion types and 6 compression distortion types.

{$\bullet$ WPC database.}
WPC database consists of 5 types of distortions, including downsample (60 samples), Gaussian Noise Contamination (180 samples), G-PCC (Octree) (80 samples), G-PCC (Trisoup) (240 samples) and V-PCC (180 samples).  From Table \ref{tab:test for distortionl}, we can see that: i) MPED shows competitive performance for downsample, Gaussian Noise Contamination and G-PCC (Octree). Specifically, MPED becomes the top 2 models for the above 3 distortion types, followed by 2 types for PCQM. ii) MPED does not work well on G-PCC (Trisoup) and V-PCC. The common characteristic of these two distortions is that they can both increase the point number greatly after compression, which usually does not lead to severe visual quality degradation. Therefore, one reasonable hypothesis is that, in the desigh of MPED, we use K-nearest-neighbor searching for the neighborhood construction, which is sensitive to the increase of point density. As a comparison, GraphSIM uses radius searching for the neighborhood construction and extracts average color gradient as quality index, which shows better performance on the above two distortions than MPED.  Therefore, one possible improvement for MPED is to adjust the method of neighborhood construction and average the potential energy discrepancy.

\subsection{Evaluation on Size of Neighborhood}
We test the performance of PED ($p=2$) on SJTU-PCQA database under different size of neighborhood to highlight the influence of scales. Considering the points of samples used in human perception are usually have more points, we set the neighborhood size $\Phi$ as $\Phi = [5], [10], [30], [50], [100]$ and show the results in Table \ref{TABLE:scale_table}. We see that, with the size of neighborhood increasing, model performance shows a downward trend. One possible reason for such phenomenon is that human perception has a marginal effect, those points farther from the center will influence the final result and lead to unreliable performance. Therefore, it is advisable to control the neighborhood size when applying MPED for human perception tasks.

\begin{table}[pt]
\begin{center}
\caption{Model Performance with different scales on
SJTU-PCQA.}\label{TABLE:scale_table}
\setlength{\tabcolsep}{0.5mm}{
\begin{tabular}{|c|c|c|c|c|}
\hline
\multirow{2}{*}{Database} &\multirow{2}{*}{$\Phi$}& \multicolumn{3}{|c|}{Quantification}   \\ \cline{3-5}  & &PLCC &SROCC &RMSE   \\
\hline\hline
\multirow{5}{*}{{SJTU-PCQA}}  & [5] &0.9017 &0.8934 &1.0491       \\
  &[10]  &0.8892 &0.8829 &1.1101    \\
  &[20]  &0.8600 &0.8560 &1.2383    \\
  &[50]  &0.8057 &0.8026 &1.4375    \\
  &[100] &0.7531 &0.7339 &1.5965    \\ 
 \hline
\end{tabular}}
\end{center}
\end{table}%

\subsection{Evaluation on Power Coefficients}
$\alpha$ and $\beta$ in Section 5.2 are used as the power coeffecients of two stimuli (i.e., color difference and geometry shifting) to derive the formulation of mass and spatial field  based on Steven's power law. In this part we test the performance of MPED ($p=2$) under different $[\alpha, \beta]$ and the show the results in Fig. \ref{fig:power comparison}. We see that: i) $[\alpha, \beta]= [0.5, 0.5]$ provides the best PLCC and the third best SROCC, which supports our parameter selection in human perception task. ii) Two extreme parameter setting, $[\alpha, \beta]=[1,0.1],[0.1,1]$, provide the worst performance. In these cases, the potential energy discrepancy is dominated by the stimulus with high power coefficient because too low power coefficient weakens the influence of the other stimulus. 

\end{document}